\documentclass{article}

\usepackage{microtype}
\usepackage{graphicx}
\usepackage{booktabs} % for professional tables
\usepackage{tikz}

\usepackage{hyperref,graphicx,amsmath,fullpage,amsfonts,subcaption,amssymb,bm,url,breakurl,epsfig,epsf,color,MnSymbol,mathbbol,fmtcount,semtrans,cite,caption,multirow,comment}

\newcommand{\tsn}[1]{{\left\vert\kern-0.25ex\left\vert\kern-0.25ex\left\vert #1 
    \right\vert\kern-0.25ex\right\vert\kern-0.25ex\right\vert}}

\definecolor{darkred}{RGB}{150,0,0}
\definecolor{darkgreen}{RGB}{0,150,0}
\definecolor{darkblue}{RGB}{0,0,200}
\hypersetup{colorlinks=true, linkcolor=darkred, citecolor=darkgreen, urlcolor=darkblue}

\newtheorem{theorem}{Theorem}[section]

\newtheorem{claim}{Claim}
\newtheorem{assumption}{Assumption}

\newtheorem{lemma}[theorem]{Lemma}

\newtheorem{definition}[theorem]{Definition}

\newcommand{\eps}{\varepsilon}

\newcommand{\tin}[1]{\|{#1}\|_{\ell_\infty}}
\newcommand{\tone}[1]{\|{#1}\|_{\ell_1}}
\newcommand{\eb}{\vct{e}}

\newcommand{\be}{\epsilon}
\newcommand{\Scp}{\mathcal{S}_+}
\newcommand{\Scn}{\mathcal{S}_-}
\newcommand{\el}{L}

\newcommand{\bp}{\beta}
\newcommand{\bn}{\alpha}

\newcommand{\beq}{\begin{equation}}

\newcommand{\eeq}{\end{equation}}

\newcommand{\nn}{\nonumber}
\newcommand{\la}{\lambda}
\newcommand{\A}{{\mtx{A}}}

\newcommand{\Ub}{{\mtx{U}}}

\newcommand{\B}{{{\mtx{B}}}}

\newcommand{\Lc}{{\cal{L}}}
\newcommand{\Jc}{{\cal{J}}}
\newcommand{\Dc}{{\cal{D}}}

\newcommand{\Pb}{{\mtx{P}}}

\newcommand{\Cb}{{\mtx{C}}}

\newcommand{\bSi}{{\boldsymbol{{\Sigma}}}}

\newcommand{\onebb}{{\mathbf{1}}}
\newcommand{\Iden}{{\mtx{I}}}

\newcommand{\distas}{\overset{\text{i.i.d.}}{\sim}}
\newcommand{\order}[1]{{\cal{O}}(#1)}

\newcommand{\smn}[1]{{\sigma_{\min}(#1)}}

\newcommand{\tn}[1]{\|{#1}\|_{\ell_2}}

\newcommand{\tf}[1]{\|{#1}\|_{F}}

\newcommand{\bteta}{\boldsymbol{\theta}}

\newcommand{\Sc}{\mathcal{S}}

\newcommand{\pa}{{\partial}}
\newcommand{\Nn}{\mathcal{N}}

\newcommand{\vb}{\vct{v}}
\newcommand{\Jb}{\mtx{J}}

\newcommand{\Ic}{{\mathcal{I}}}

\newcommand{\cb}{\mtx{c}}

\newcommand{\w}{\vct{w}}

\newcommand{\g}{{\vct{g}}}

\newcommand{\rbb}{\vct{\bar{r}}}

\newcommand{\rt}{\vct{\tilde{r}}}

\newcommand{\opnorm}[1]{\left\|#1\right\|}
\newcommand{\fronorm}[1]{\left\|#1\right\|_{F}}

\newcommand{\twonorm}[1]{\left\|#1\right\|_{\ell_2}}

\newcommand{\infnorm}[1]{\left\|#1\right\|_{\ell_\infty}}

\newcommand{\abs}[1]{\left|#1\right|}

\newcommand{\x}{\vct{x}}
\newcommand{\rb}{\vct{r}}

\newcommand{\y}{\vct{y}}

\newcommand{\W}{\mtx{W}}
\newcommand{\bgl}{{~\big |~}}

\definecolor{emmanuel}{RGB}{255,127,0}

\newcommand{\pb}{{\vct{p}}}

\newcommand{\R}{\mathbb{R}}
\newcommand{\Pro}{\mathbb{P}}

\newcommand{\E}{\operatorname{\mathbb{E}}}

\newcommand{\grad}[1]{{\nabla\Lc(#1)}}

\newcommand{\vct}[1]{\bm{#1}}
\newcommand{\mtx}[1]{\bm{#1}}

\newcommand{\Pc}{{\cal{P}}}
\newcommand{\X}{{\mtx{X}}}

\numberwithin{equation}{section} 

\def \endprf{\hfill {\vrule height6pt width6pt depth0pt}\medskip}
\newenvironment{proof}{\noindent {\bf Proof} }{\endprf\par}

\begin{document}

\title{Gradient Descent with Early Stopping is Provably Robust\\to Label Noise for Overparameterized Neural Networks}
\author{Mingchen Li\thanks{{Department of Computer Science and Engineering, University of California, Riverside, CA}}\quad \quad Mahdi Soltanolkotabi\thanks{Ming Hsieh Department of Electrical Engineering, University of Southern California, Los Angeles, CA}\quad  \quad Samet Oymak\thanks{{Department of Electrical and Computer Engineering, University of California, Riverside, CA}}}
\maketitle

\begin{abstract}
Modern neural networks are typically trained in an over-parameterized regime where the parameters of the model far exceed the size of the training data. Such neural networks in principle have the capacity to (over)fit any set of labels including pure noise. Despite this, somewhat paradoxically, neural network models trained via first-order methods continue to predict well on yet unseen test data. This paper takes a step towards demystifying this phenomena. Under a rich dataset model, we show that gradient descent is provably robust to noise/corruption on a constant fraction of the labels despite overparameterization. In particular, we prove that: (i) In the first few iterations where the updates are still in the vicinity of the initialization gradient descent only fits to the correct labels essentially ignoring the noisy labels. (ii) to start to overfit to the noisy labels network must stray rather far from from the initialization which can only occur after many more iterations. Together, these results show that gradient descent with early stopping is provably robust to label noise and shed light on the empirical robustness of deep networks as well as commonly adopted heuristics to prevent overfitting.
\end{abstract}
%We believe that these unique characteristics of gradient descent may also help demystify commonly adopted heuristics used to avoid overfitting such as early stopping.
%%%%%%%%%%%%%%%%%%%%%%%%%%%%%%%%%%%%%%%%%%%%%%%%%%%%%%%%%%%%%%%%%%%%%%%%%%%%%%%
%%%%%%%%%%%%%%%%%%%%%%%%%%%%%%%%%%%%%%%%%%%%%%%%%%%%%%%%%%%%%%%%%%%%%%%%%%%%%%%

\section{Introduction}
%\subsection{Motivation}
%\MS{edited up to here}
%The overfitting is a serious concern as DNNs are deployed in safety critical applications such as robotics and military. This brings an urgent need to understand foundational properties of DNNs such as generalization ability and noise robustness.

This paper focuses on an intriguing phenomena: overparameterized neural networks are surprisingly robust to label noise when first order methods with early stopping is used to train them. To observe this phenomena consider Figure \ref{mnistacc} where we perform experiments on the MNIST data set. Here, we corrupt a fraction of the labels of the training data by assigning their label uniformly at random. We then fit a four layer model via stochastic gradient descent and plot various performance metrics in Figures \ref{mnistacca} and \ref{mnistaccb}. Figure \ref{mnistacca} (blue curve) shows that indeed with a sufficiently large number of iterations the neural network does in fact perfectly fit the corrupted training data. However, Figure \ref{mnistacca} also shows that such a model does not generalize to the test data (yellow curve) and the accuracy with respect to the ground truth labels degrades (orange curve). These plots clearly demonstrate that the model overfits with many iterations. In Figure \ref{mnistaccb} we repeat the same experiment but this time stop the updates after a few iterations (i.e.~use early stopping). In this case the train accuracy degrades linearly (blue curve). However, perhaps unexpected, the test accuracy (yellow curve) remains high even with a significant amount of corruption. This suggests that with early stopping the model does not overfit and generalizes to new test data. Even more surprising, the train accuracy (orange curve) with respect to the ground truth labels continues to stay around $100\%$ even when $50\%$ of the labels are corrupted (see also \cite{guan2018said} and \cite{rolnick2017deep} for related empirical experiments). That is, with early stopping overparameterized neural networks even correct the corrupted labels! These plots collectively demonstrate that overparameterized neural networks when combined with early stopping have unique generalization and robustness capabilities. As we detail further in Section \ref{sec numer} this phenomena holds (albeit less pronounced) for richer data models and architectures. %In this paper we aim to demystify these unique generalization and robustness capabilities mathematically.

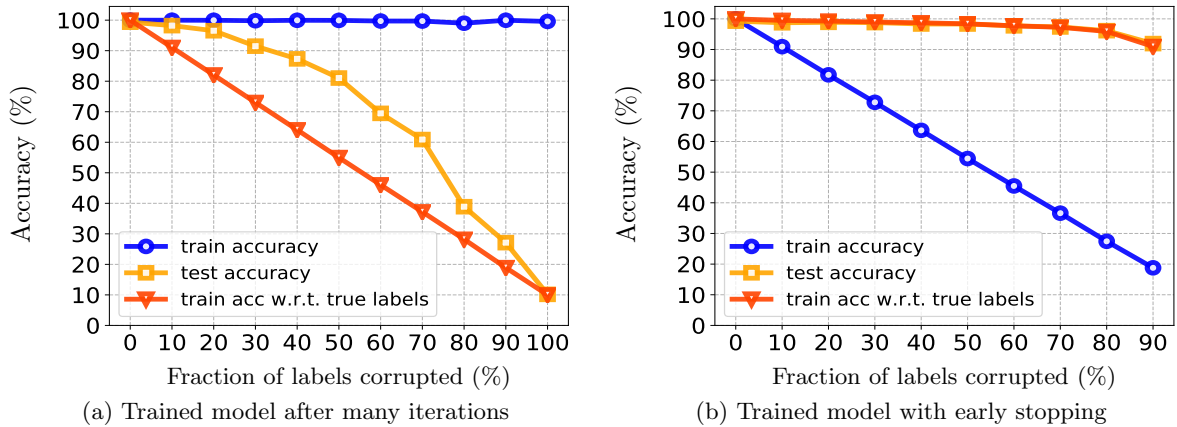
\begin{figure}[t!]
\begin{centering}
\begin{subfigure}[t]{3.1in}
\begin{tikzpicture}
\node at (0,0) {\includegraphics[height=0.7\linewidth,width=1\linewidth]{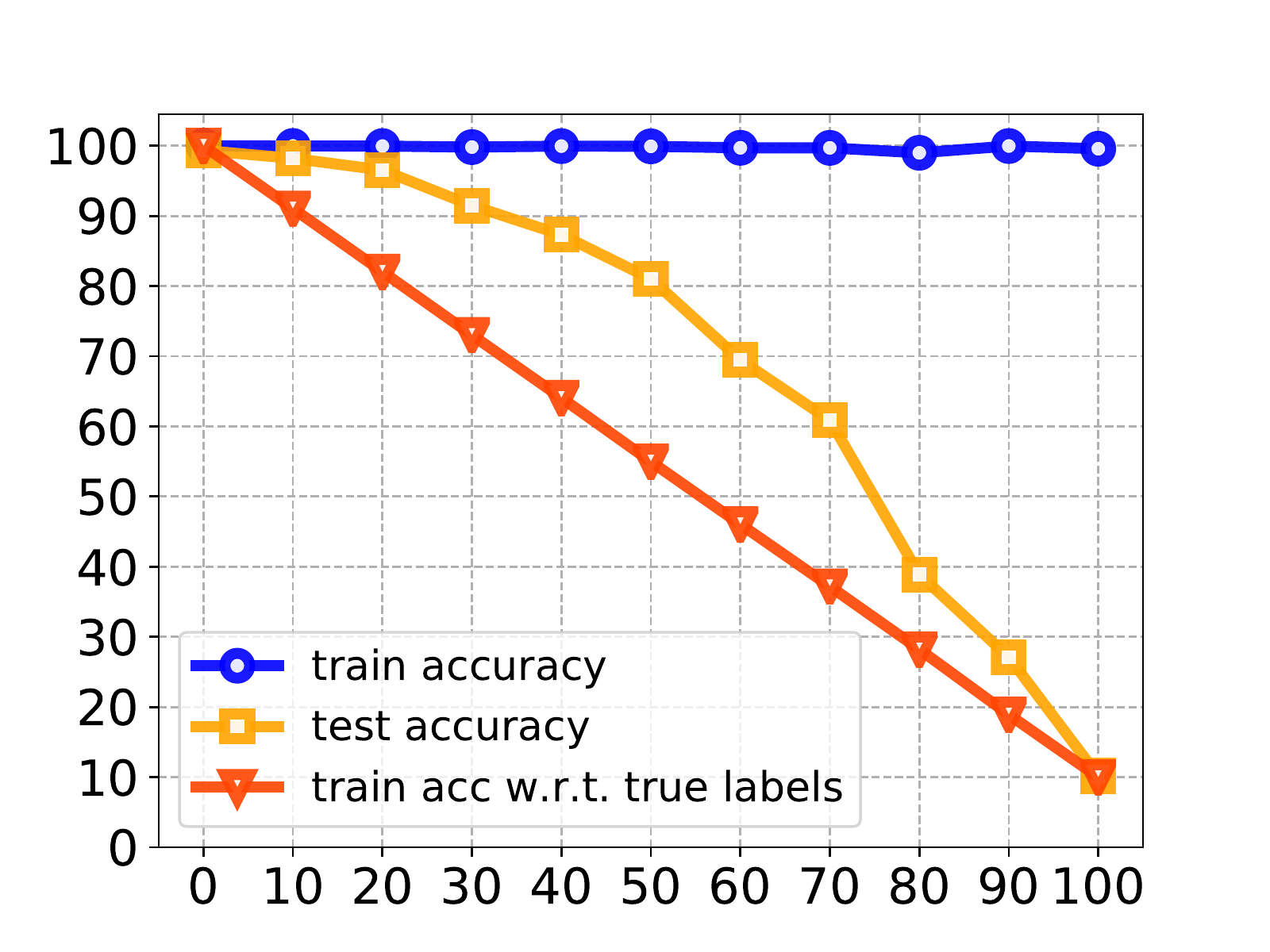}};
\node at (0.1,-2.85) {\small{Fraction of labels corrupted (\%)}};
\node[rotate=90] at (-4.1,0) {Accuracy (\%)};
\end{tikzpicture}\vspace{-6pt}
%\node at (1,1) {source};
%\includegraphics[height=0.7\linewidth,width=1\linewidth]{figs/neural_net2_500.pdf}\vspace{-5pt}\
\caption{Trained model after many iterations}
\label{mnistacca}
\end{subfigure}
\end{centering}\hspace{-5pt}
\begin{centering}
\begin{subfigure}[t]{3.1in}
\begin{tikzpicture}
\node at (0,0) {\includegraphics[height=0.7\linewidth,width=1\linewidth]{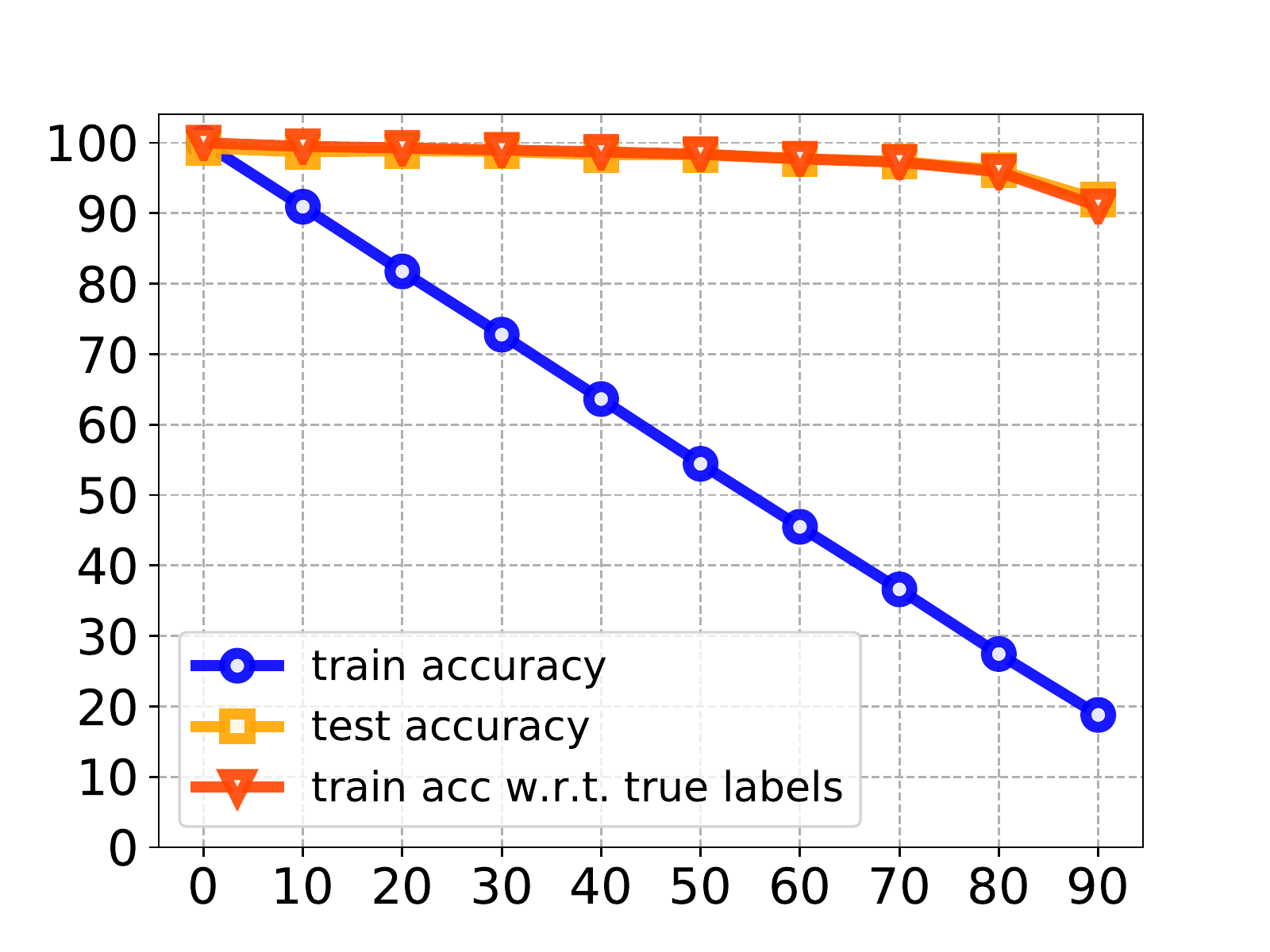}};
\node at (0.1,-2.85) {\small{Fraction of labels corrupted (\%)}};
\node[rotate=90] at (-4.1,0) {Accuracy (\%)};
%\node at (1.5,-1.8) {Normalized distance $\left(\frac{\tf{\mtx{W}^{\ell}-\mtx{W}^{\ell}_0}}{\tf{\mtx{W}_0^{\ell}}}\right)$};
\end{tikzpicture}\vspace{-6pt}
\caption{Trained model with early stopping}\label{mnistaccb}
\end{subfigure}%{\color{red} DELETE \{a data set of 50,000 samples from\}}
\end{centering}\vspace{-0.15cm}\caption{\small{In these experiments we use a $4$ layer neural network consisting of two convolution layers followed by two fully-connected layers to train MNIST with various amounts of random corruption on the labels. In this architecture the convolution layers have width $64$ and $128$ kernels, and the fully-connected layers have $256$ and $10$ outputs, respectively. Overall, there are $4.8$ million trainable parameters. {We use 50k samples for training, 10k samples for validation, and we test the performance on a 10k test dataset.} We depict the training accuracy both w.r.t. the corrupted and uncorrupted labels as well as the test accuracy. (a) Shows the performance after 200 epochs of Adadelta where near perfect fitting to the corrupted data is achieved. (b) Shows the performance with early stopping. We observe that with early stopping the trained neural network is robust to label corruption.}}\label{mnistacc}% to achieve overfitting
\vspace{-0.3cm}
\end{figure}
%Conv1(64 kernel, 3x3, with relu) -> Conv2(128 kernel, 3x3, with relu) -> Max pooling(2x2) -> Dropout1(0.25) -> Fully Connect1(256 output) -> Dropout2(0.5) -> Fully Connect2(num_output)
%4.8 M params

This paper aims to demystify the surprising robustness of overparameterized neural networks when early stopping is used. We show that gradient descent is indeed provably robust to noise/corruption on a {\em{constant fraction of the labels}} in such over-parameterized learning scenarios. In particular, under a fairly expressive dataset model and focusing on one-hidden layer networks, we show that after a few iterations (a.k.a.~\emph{early stopping}), gradient descent finds a model (i) that is within a small neighborhood of the point of initialization and (ii) only fits to the correct labels essentially ignoring the noisy labels. We complement these findings by proving that if the network is trained to overfit to the noisy labels, then the solution found by gradient descent must stray rather far from the initial model. Together, these results highlight the key features of a solution that {\em{generalizes well}} vs. a solution that {\em{fits well}}.\vspace{-1pt}% within our theoretical framework. 
%Our results also show that overparameterization has no impact on the robustness of the network to outliers.\MS{??? I thought the whole point was to show overparam makes it more robust}

Our theoretical results further highlight the role of {\em{the distance between final and initial network weights}} as a key feature that determines noise robustness vs. overfitting. This is inherently connected to the commonly used early stopping heuristic for DNN training as this heuristic helps avoid models that are too far from the point of initialization. In the presence of label noise, we show that gradient descent {\em{implicitly}} ignores the noisy labels as long as the model parameters remain close to the initialization. Hence, our results help explain why early stopping improves robustness and helps prevent overfitting. Under proper normalization, the required distance between the final and initial network and the predictive accuracy of the final network is independent of the size of the network such as number of hidden nodes. Our extensive numerical experiments corroborate our theory and verify the surprising robustness of DNNs to label noise. Finally, we would like to note that while our results show that solutions found by gradient descent are inherently robust to label noise, specialized techniques such as $\ell_1$ penalization or sample reweighting are known to further improve robustness. Our theoretical framework may enable more rigorous understandings of the benefits of such heuristics when training overparameterized models.

%We believe approach outlined here has the potential to assess the benefit of these heuristics when training overparameterized networks.%, once the network is large 

% which can only occur after many more iterative updates. 

%Recent line of works advocates that, overparameterization doesn't hurt DNN part of the reason DNNs have  is that the optimization process tends to find models with 

% we show that first order methods such as gradient descent are provably robust to noise/corruption on a fraction of the labels in this over-parametrized regime. In particular, we show that in the first few iterations where the updates are still in the vicinity of the initialization these algorithms only fit to the correct labels essentially ignoring the noisy labels. We believe that these unique characteristics of gradient descent may also help demystify commonly adopted heuristics used to avoid overfitting such as early stopping.
%Recent line of works  It has been observed that, first order methods implicitly find good solutions with better better generalization abilities.
\subsection{Prior Art}
Our work is connected to recent advances on theory for deep learning as well as heuristics and theory surrounding outlier robust optimization.

\noindent {\bf{Robustness to label corruption:}} DNNs have the ability to fit to pure noise \cite{zhang2016understanding}, however they are also empirically observed to be highly resilient to label noise and generalize well despite large corruption \cite{rolnick2017deep}. In addition to early stopping, several heuristics have been proposed to specifically deal with label noise \cite{reed2014training,malach2017decoupling,zhang2018generalized,scott2013classification,khetan2017learning,han2018co}. See also \cite{frenay2014comprehensive,menon2018learning,ren2018learning,shen2018iteratively} for additional work on dealing with label noise in classification tasks. Label noise is also connected to outlier robustness in regression which is a traditionally well-studied topic. In the context of robust regression and high-dimensional statistics, much of the focus is on regularization techniques to automatically detect and discard outliers by using tools such as $\ell_1$ penalization \cite{chen2013robust,li2013compressed,balakrishnan2017computationally,liu2018high,bhatia2015robust,candes2011robust,foygel2014corrupted}. We would also like to note that there is an interesting line of work that focuses on developing robust algorithms for corruption not only in the labels but also input data \cite{diakonikolas2018sever,prasad2018robust,klivans2018efficient}. Finally, noise robustness is particularly important in safety critical domains. Noise robustness of neural nets has been empirically investigated by Hinton and coauthors in the context of automated medical diagnosis \cite{guan2018said}.
%When learning from pairwise relations, noisy labels can be connected to graph clustering and community detection problems \cite{cai2015robust,vinayak2014graph,abbe2016exact}. 

% study the benefits of overparameterization for training neural networks and related optimization problems  However these works are specialized towards neural nets and similar to us the bounds on the network size to achieve global optimality appear to be suboptimal.\footnote{We note that while both our results and these papers are suboptimal for one-hidden layer neural networks, they are not directly comparable with each other. We assume $n\le d$ where as these papers assume $poly(n)\lesssim k$. Also the conclusions and assumptions are different from each other.} In contrast, we focus on general nonlinearities and also on the gradient descent trajectory showing that among all the global optima, gradient descent converges to one with near minimal distance to the initialization.\SO{copied from shortest!}I \footnote{Jacobian matrix is obtained by merging the partial derivatives of network output with respect to individual samples.}
\noindent {\bf{Overparameterized neural networks:}} Intriguing properties and benefits of overparameterized neural networks has been the focus of a growing list of publications \cite{zhang2016understanding,soltanolkotabi2018theoretical,brutzkus2017sgd,chizat2018global,arora2018optimization, Ji:2018aa, venturi2018spurious, Zhu:2018aa, Soudry:2016aa, Brutzkus:2018aa, azizan2018stochastic, neyshabur2018towards}. A recent line of work \cite{li2018learning,allen2018learning,allen2018convergence,du2018gradient,zou2018stochastic, du2018gradient2, anon2019overparam} shows that overparameterized neural networks can fit the data with random initialization if the number of hidden nodes are polynomially large in the size of the dataset. This line of work however is not informative about the robustness of the trained network against corrupted labels. Indeed, such theory predicts that (stochastic) gradient descent will eventually fit the corrupted labels. In contrast, our focus here is not in finding a global minima, rather a solution that is robust to label corruption. In particular, we show that with early stopping we fit to the correct labels without overfitting to the corrupted training data. Our result also differs from this line of research in another way. The key property utilized in this research area is that the Jacobian of the neural network is well-conditioned at a random initialization if the dataset is sufficiently diverse (e.g.~if the points are well-separated). In contrast, in our model the Jacobian is approximately low-rank with the rank of the Jacobian corresponding to different clusters/classes within the dataset. We harness this low-rank nature to prove that gradient descent is robust to label corruptions. We further utilize this low-rank structure to explain why neural networks can work with much more modest amounts of overparameterization where the number of parameters grow with rank rather than the sample size. Furthermore, our numerical experiments verify that the Jacobian matrix of real datasets (such as CIFAR10) indeed exhibit low-rank structure. This is closely related to the observations on the Hessian of deep networks which is empirically observed to be low-rank \cite{sagun2017empirical,chaudhari2016entropy}.  An equally important question for understanding the convergence behavior of optimization algorithms for overparameterized models is understanding their generalization capabilities. This is the subject of a few interesting recent papers \cite{Arora:2018aa, Bartlett:2017aa, Golowich:2017aa, song2018mean, Brutzkus:2017aa, oymak2019generalization, Belkin:2018aa, Liang:2018aa, Belkin:2018ab}. While in this paper we do not tackle generalization in the traditional sense, we do show that solutions found by gradient descent are robust to label noise/corruption which demonstrates their predictive capabilities and in turn suggests better generalization.% accuracy despite label noise.
%our results do not directly address generalization, by characterizing the properties of the global optima that (stochastic) gradient descent converges to it may help demystify the generalization capabilities of overparametrized models.
%Recently in \cite{} we showed that this conclusion continues to hold with more modest amounts of overparameterization and as soon as the number of parameters of the model exceed the square of the size of the training data set. 
%in the model exceeds the number of clusters raised to the fourth power and is independent of the number of data points.
%We would also like to note that the importance of the Jacobian for overparameterized neural network analysis has also been noted by other papers including \cite{shortest,soltanolkotabi2018theoretical,du2018gradient} and also \cite{keskar2016large,chaudhari2016entropy} which investigate the optimization landscape and properties of SGD for training neural networks.

\subsection{Models}
 \begin{figure*} 
\centering
\includegraphics[scale=1]{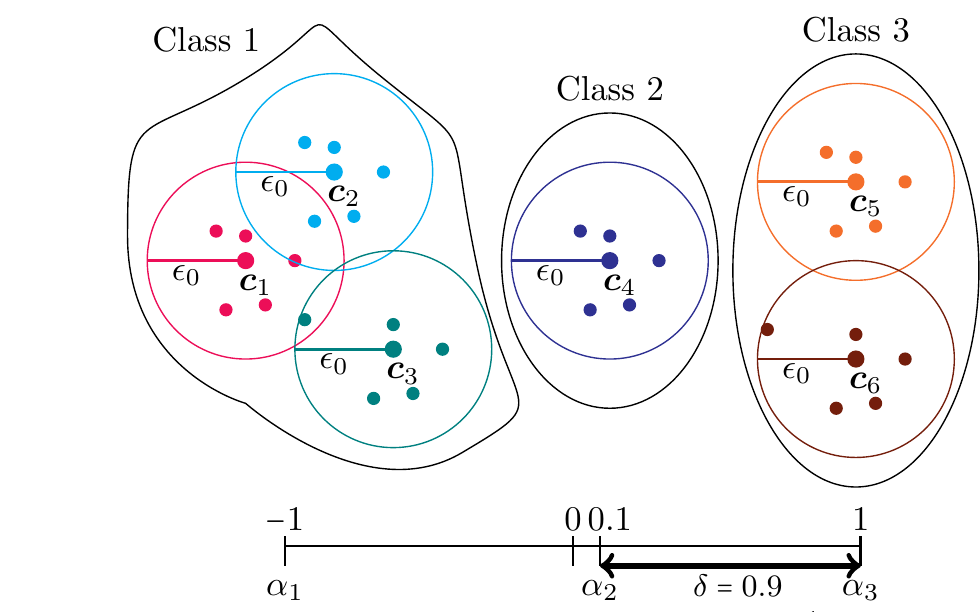}
\caption{Visualization of the input/label samples and classes according to the clusterable model in Definition \ref{cdata}. In the depicted example there are $K=6$ clusters, $\bar{K}=3$ classes.  In this example the number of data points is $n=30$ with each cluster containing $5$ data points. The labels associated to classes $1$, $2$, and $3$ are $\alpha_1=-1$, $\alpha_2=0.1$, and $\alpha_3=1$, respectively so that $\delta=0.9$. We note that the placement of points are exaggerated for clarity. In particular, per definition the cluster center and data points all have unit Euclidean norm.}
 \label{mod}
\end{figure*}
%this is a cartoon depiction and 
%Let us fix the notation. Projection on a subspace $\Sc$ is denoted by $\Pi_{\Sc}(\cdot)$. The spectral norm and minimum singular value of a matrix is denoted by $\|\cdot\|$ and $\smn{}$ respectively.%We will explore a specific dist

We first describe the dataset model used in our theoretical results. In this model we assume that the input samples $\vct{x}_1,\vct{x}_2,\ldots,\vct{x}_n\in\R^d$ come from $K$ clusters which are located on the unit Euclidean ball in $\R^d$. We also assume our dataset consists of $\bar{K}\le K$ classes where each class can be composed of multiple clusters. We consider a deterministic dataset with $n$ samples with roughly balanced clusters each consisting on the order of ${n/K}$ samples.\footnote{This is for ease of exposition rather than a particular challenge arising in the analysis.} Finally, while we allow for multiple classes, in our model we assume the labels are scalars and take values in $[-1,1]$ interval. Each unit Euclidean norm $\x$ is assigned to one of these class labels as described next. We formally define our dataset model below and provide an illustration in Figure \ref{mod}.% as described be %Let $\tilde{y}(\x)$ be a labeling function mapping a vector $\x$ to one of these class labels as follows: Pick $\tilde{y}(\cb_1)$ to $\tilde{y}(\cb_K)$ arbitrarily. Any other input $\x$ is assigned to the label of nearest cluster.

\begin{definition} [$(\eps_0,\delta)$ Clusterable dataset] \label{cdata} A clusterable dataset of size $n$ consisting of input/label pairs $\{(\vct{x}_i,y_i)\}_{i=1}^n\in\R^d\times \R$ is described as follows. The input data have unit Euclidean norm and originate from $K$ clusters with the $\ell$th cluster containing $n_\ell$ data points. The number of points originating from each cluster is well-balanced in the sense that $c_{low}\frac{n}{K}\le n_\ell\le c_{up}\frac{n}{K}$ with $c_{low}$ and $c_{up}$ two numerical constants obeying $0<c_{low}<c_{up}$. We use $\{\cb_\ell\}_{\ell=1}^K\subset\R^d$ to denote the cluster centers which are unit Euclidean norm vectors. The input data points $\vct{x}$ that belong to the $\ell$-th cluster obey $\twonorm{\vct{x}-\vct{c}_\ell}\le \eps_0$, with $\eps_0>0$ denoting the input noise level.

	The labels $y_i$ belong to one of $\bar{K}\le K$ classes. Specifically, $y_i\in\{\alpha_1,\alpha_2,\ldots,\alpha_{\bar{K}}\}$ with $\{\alpha_\ell\}_{\ell=1}^{\bar{K}}\in[-1,1]$ denoting the labels associated with each class. All the elements of the same cluster belong to the same class and hence have the same label. However, a class can contain multiple clusters. Finally, the labels/classes are separated in the sense that\vspace{-4pt}
	\begin{align}
	|\alpha_r-\alpha_s|\geq \delta\quad\text{for}\quad r\neq s,\label{alpha eq}
	%\\&\tn{\cb_i-\cb_j}\geq 2\eps_0.
	\end{align}
	with $\delta>0$ denoting the label separation and for any two clusters $\ell,\ell'$ belonging to two different classes we have $\twonorm{\vct{c}_\ell-\vct{c}_{\ell'}}\ge 2\eps_0$.
\end{definition}\vspace{-6pt}

In the data model above $\{\cb_\ell\}_{\ell=1}^K$ are the $K$ cluster centers that govern the input distribution. We note that in this model different clusters can be assigned to the same label. Hence, this setup is rich enough to model data which is not linearly separable: e.g.~over $\R^2$, we can assign cluster centers $(0,1)$ and $(0,-1)$ to label $1$ and cluster centers $(1,0)$ and $(-1,0)$ to label $-1$. Note that the maximum number of classes are dictated by the separation $\delta$, in particular, ${\bar{K}}\leq \frac{2}{\delta}+1$. Our dataset model is inspired from mixture models and is also related to the setup of \cite{li2018learning} which provides polynomial guarantees for learning shallow networks. Next, we introduce our noisy/corrupted dataset model.\vspace{-2pt}% For instance, linear classifier cannot separate this
\begin{definition}[$(\rho,\eps_0,\delta)$ corrupted dataset] \label{noisy model}Let $\{(\x_i,\widetilde{y}_i)\}_{i=1}^n$ be an $(\eps_0,\delta)$ clusterable dataset with $\alpha_1$, $\alpha_2$, $\ldots, \alpha_{\bar{K}}$ denoting the $\bar{K}$ possible class labels. A $(\rho,\eps_0,\delta)$ noisy/corrupted dataset $\{(\x_i,{y}_i)\}_{i=1}^n$ is generated from $\{(\x_i,\widetilde{y}_i)\}_{i=1}^n$ as follows. For each cluster $1\leq \ell\leq K$, at most $\rho n_\ell$ of the labels associated with that cluster (which contains $n_\ell$ points) is assigned to another label value chosen from $\{\alpha_\ell\}_{\ell=1}^{\bar{K}}$. We shall refer to the initial labels $\{\widetilde{y}_i\}_{i=1}^n$ as the ground truth labels.
\end{definition}\vspace{-2pt}
We note that this definition allows for a fraction $\rho$ of corruptions in each cluster. Next we define the ground truth label function.
\vspace{-0.1cm}
\begin{definition}[Ground truth label function]\label{labfunc} Consider the setting of Definition \ref{cdata} with cluster centers $\{\cb_\ell\}_{\ell=1}^K\subset\R^d$ and class labels $\{\alpha_\ell\}_{\ell=1}^{\bar{K}}$. We define the ground truth label function $\vct{x}\mapsto \widetilde{y}(\vct{x})$ as the function that maps a point $\vct{x}\in\R^d$ to a class label $\{\alpha_1,\alpha_2,\ldots,\alpha_{\bar{K}}\}$ by assigning to it the label corresponding to the closest cluster center. In mathematical terms\vspace{-2pt}
	\begin{align*}
	\widetilde{y}(\vct{x})=\text{label of }\vct{c}_{\hat{\ell}}\quad\text{where}\quad \hat{\ell}=\underset{1\le \ell\le K}{\arg\min} \twonorm{\vct{x}-\vct{c}_\ell}.
	\end{align*}
	In particular, when applied to the training data it yields the ground truth labels i.e.~$\tilde{y}(\vct{x}_i)=\tilde{y}_i$.
\end{definition}

\vspace{-0.2cm}

%This definition sets a constant fraction of maximum allowable noise per cluster. This is in contrast to setting a total noise level for the overall dataset. The intuition behind this model is that, if label noise is localized to a specific cluster, one can simply flip the labels of that cluster using only $\order{n/K}$ noise. Instead, we will prove that, our dataset model can tolerate a total label noise amount of $\order{n/\bar{K}}$ which is independent of the number of clusters and only depends on number of classes.
%Hence, each class is allowed to be composed of multiple clusters, which makes our model more general. 
%is allowed to have a nontrivial distribution.

%This makes our model more less separable compared to a classical $k$-means  is more general than an $K$ local

%This is a  %Note that each $\alpha_i$ can have 

% have ${\bar{K}}$ distinct values $(\bar{\alpha}_i)_{i=1}^K$. %Observe that $\W$ is the main source of Furthermore, we fix 
\noindent {\bf{Network model:}} We will study the ability of neural networks to learn this corrupted dataset model. To proceed, let us introduce our neural network model. We consider a network with one hidden layer that maps $\R^d$ to $\R$. Denoting the number of hidden nodes by $k$, this network is characterized by an activation function $\phi$, input weight matrix $\W\in\R^{k\times d}$ and output weight vector $\vb\in\R^k$. In this work, we will fix output $\vb$ to be a unit vector where half the entries are $1/\sqrt{k}$ and other half are $-1/\sqrt{k}$ to simplify exposition.\footnote{If the number of hidden units is odd we set one entry of $\vct{v}$ to zero.} We will only optimize over the weight matrix $\W$ which contains most of the network parameters and will be shown to be sufficient for robust learning. We will also assume $\phi$ has bounded first and second order derivatives, i.e.~$\abs{\phi'(z)},\abs{\phi''(z)}\le \Gamma$ for all $z$. The network's prediction at an input sample $\x$ is given by\vspace{-2pt}
\begin{align}
\vct{x}\mapsto f(\W,\x)=\vb^T\phi(\W\x),\label{nn model}
\end{align}
where the activation function $\phi$ applies entrywise. Given a dataset $\{(\x_i,y_i)\}_{i=1}^n$, we shall train the network via minimizing the empirical risk over the training data via a quadratic loss\vspace{-2pt}
\begin{align}
\Lc(\W)=\frac{1}{2}\sum_{i=1}^n (y_i-f(\W,\x_i))^2.\label{q loss}
\end{align}%To optimize the network, w
In particular, we will run gradient descent with a constant learning rate $\eta$, starting from a random initialization $\W_0$ via the following gradient descent updates
\begin{align}
\W_{\tau+1}=\W_{\tau}-\eta \grad{\W_{\tau}}.\label{grad it}
\end{align}
%\vspace{-5pt}
\section{Main results}

%The notation $\order{\cdot}$ denotes that a certain identity holds up to a fixed numerical constant. Also, 
%\subsection{Robustness of neural network to label noise with early stopping}
Our main result shows that overparameterized neural networks, when trained via gradient descent using early stopping are fairly robust to label noise. Throughout, $\|\cdot\|$ denotes the largest singular value of a given matrix. $c$, $c_0$, $C$, $C_0$ etc. represent numerical constants. The ability of neural networks to learn from the training data, even without label corruption, naturally depends on the diversity of the input training data. Indeed, if two input data are nearly the same but have different uncorrupted labels reliable learning is difficult. We will quantify this notion of diversity via a notion of condition number related to a covariance matrix involving the activation $\phi$ and the cluster centers $\{\cb_\ell\}_{\ell=1}^K$.

\begin{definition} [Neural Net Cluster Covariance] \label{clust cov}Define the matrix of cluster centers
	\[
	\Cb=[\cb_1~\dots~\cb_K]^T\in\R^{K\times d}.
	\]
	Let $\g\sim\Nn(0,\Iden_d)$. Define the neural net covariance matrix $\bSi(\Cb)$ as
	\[
	\bSi(\Cb)=(\Cb\Cb^T)\bigodot \E_{\g}[\phi'(\Cb\g)\phi'(\Cb\g)^T].
	\]
	Here $\bigodot$ denotes the elementwise product. Also denote the minimum eigenvalue of $\bSi(\Cb)$ by $\la(\Cb)$.
	% and define the following condition number associated with the cluster centers $\mtx{C}$
	%\vspace{-5pt}\begin{align*}
	%\kappa(\mtx{C})=\sqrt{\frac{d}{K}}\frac{\opnorm{\mtx{C}}}{\lambda(\mtx{C})}.
	%\end{align*}
\end{definition} 
\vspace{-0.2cm}
%Therefore, one can think of $\kappa(\mtx{C})$ as a condition number associated with the neural network which characterizes the distinctness/diversity of the cluster centers. 
%and smaller the condition number $\kappa(\mtx{C})$ is% (and hence $\kappa(\Cb)<\infty$)
One can view $\bSi(\Cb)$ as an empirical kernel matrix associated with the network where the kernel is given by $\mathcal{K}(\vct{c}_i,\vct{c}_j)=\bSi_{ij}(\Cb)$. Note that $\bSi(\Cb)$ is trivially rank deficient if there are two cluster centers that are identical. In this sense, the minimum eigenvalue of $\bSi(\Cb)$ will quantify the ability of the neural network to distinguish between distinct cluster centers. The more distinct the cluster centers, the larger $\lambda(\mtx{C})$ is. Throughout we shall assume that $\la(\Cb)$ is strictly positive. Related assumptions are empirically and theoretically studied in earlier works by \cite{allen2018convergence,xie2016diverse,du2018gradient,du2018gradient2}. For instance, when the cluster centers are maximally diverse e.g.~uniformly at random from the unit sphere $\la(\mtx{C})$ scales like a constant (\cite{anon2019overparam}). Additionally, for ReLU activation, if the cluster centers are separated by a distance $\nu>0$, then $\la(\Cb)\geq  \frac{\nu}{100K^2}$ (\cite{zou2018stochastic,anon2019overparam}).

% \cite{anon2019overparam} 
% This property is empirically verified to hold in earlier works \cite{xie2016diverse} when $\phi$ is a standard activation (e.g.~ReLU, softplus). As a concrete example, for ReLU activation, using results from \cite{anon2019overparam} one can show if the cluster centers are separated by a distance $\nu>0$, then $\la(\Cb)\geq  \frac{\nu}{100K^2}$. We note that variations of the $\lambda(\mtx{C})>0$ assumption based on the data points (i.e.~$\lambda(\mtx{X})>0$ not cluster centers) \cite{anon2019overparam,du2018gradient,du2018gradient2} are utilized to provide convergence guarantees for DNNs. Also see \cite{allen2018convergence,zou2018stochastic} for other publications using related definitions. 

Now that we have a quantitative characterization of distinctiveness/diversity in place we are now ready to state our main result. Throughout we use $c_{\Gamma}, C_{\Gamma}$, etc.~to denote constants only depending on $\Gamma$. We note that this theorem is slightly simplified by ignoring logarithmic terms and precise dependencies on $\Gamma$.  We refer the reader to Theorem \ref{main thm robust22} for the precise statement.\vspace{-1pt}
% including logarithmic terms
%We also use $a\gtrsim_{\Gamma}b$ to denote that $a\ge C_{\Gamma} b$ holds. 
%\footnote{If $k$ is odd we set one entry to zero $\lfloor \frac{k-1}{2}\rfloor$ to $1/\sqrt{k}$ and $\lfloor \frac{k-1}{2}\rfloor$ entries to $-1/\sqrt{k}$.}
%Furthermore, let $\{\widetilde{y}_i\}_{i=1}^n$ be the corresponding uncorrupted ground truth labels. 
\begin{theorem} [Main result] \label{main thm robust2}Consider a clusterable corrupted dataset of $\{(\x_i,y_i)\}_{i=1}^n\in\R^d\times \R$ per Definition \ref{noisy model} with cluster centers $\{\vct{c}_\ell\}_{\ell=1}^K$ aggregated as rows of a matrix $\mtx{C}\in\R^{K\times d}$. Also consider a one-hidden layer neural network with $k$ hidden units and one output of the form $\vct{x}\mapsto \vct{v}^T\phi\left(\mtx{W}\vct{x}\right)$ with $\mtx{W}\in\R^{k\times d}$ and $\vct{v}\in\R^k$. Also suppose the activation $\phi$ obeys $|\phi(0)|,|\phi'(z)|,|\phi''(z)|\leq \Gamma$ for all $z$ and some $\Gamma\geq 1$. Furthermore, fix half of the entries of $\vct{v}$ to $1/\sqrt{k}$ and the other half to $-1/\sqrt{k}$ and train only over $\mtx{W}$. Starting from an initial weight matrix $\mtx{W}_0$ with i.i.d.~$\mathcal{N}(0,1)$ entries, run gradient descent updates $\mtx{W}_{\tau+1}=\mtx{W}_\tau-\eta\nabla \mathcal{L}(\mtx{W}_\tau)$ on the least-squares loss \eqref{q loss} with step size $\eta= \bar{c}_{\Gamma}\frac{K}{n}\frac{1}{\opnorm{\mtx{C}}^2}$. Furthermore, assume the number of hidden units obey
	\vspace{-4pt}\begin{align*}
	k\ge C_{\Gamma}\frac{K^2\|\Cb\|^4}{\la(\Cb)^4},
	\end{align*}
	with $\lambda(\mtx{C})$ be the minimum cluster eigenvalue per Definition \ref{clust cov}. Then as long as $\eps_0\le \widetilde{c}_{\Gamma}\delta\la(\Cb)^2/K^2$ and $\rho\le\frac{\delta}{8}$
	with probability at least $1-3/K^{100}-Ke^{-100d}$, after $T= c_{\Gamma}\frac{\|\Cb\|^2}{\lambda(\Cb)}$ iterations, the neural network $f(\mtx{W}_{T},\cdot)$ found by gradient descent predicts the true label function $\tilde{y}(\x)$ per Definition \ref{labfunc} for all input $\vct{x}\in\R^d$ that lie within $\eps_0$ neighborhood of a cluster center $\cb_1,\dots,\cb_K$
	\vspace{-2pt}\begin{align}
	\arg\min_{\alpha_\ell:1\leq \ell\leq \bar{K}}|f(\W_{T},\x)-\alpha_\ell|=\widetilde{y}(\x).\label{pls satisfy this eq2}
	\end{align}
	%holds for all $1\leq i\leq n$.
	%$T= \order{\frac{K}{\eta n}\log(\frac{1}{\rho})}$
	%for all $1\leq i\leq n$, with respect to {\em{true labels}} $\{\tilde{y}_i\}_{i=1}^n$ we have
	%\[
	%|f(\W_{\tau},\x_i)-\tilde{y}_i|\leq 4\rho+\order{{\eps_0K\sqrt{n}}{}}\quad\text{and}\quad \frac{\tn{f(\W_{\tau},\X)-\tilde{\y}}}{\sqrt{n}}\leq 4\rho+\order{{\eps_0K}{}}.
	%\]
	Eq. \eqref{pls satisfy this eq2} applies to all training samples. Finally, for all $0\leq \tau\leq T$, the distance to initialization obeys\vspace{-3pt}
	\[
	\tf{\W_\tau-{\W}_0} \le \bar{C}_{\Gamma} \left(\sqrt{\frac{K}{\lambda(C)}}+\frac{K^2}{\lambda(C) \|\Cb\|^2}\tau\eps_0\right).% \frac{}{\la(\Cb)}\log(\frac{\Gamma\sqrt{n\log K}}{\rho})^2}.
	\]
\end{theorem}

%\begin{theorem} [Training neural nets with corrupted labels] \label{main thm robust2}Let $\{(\x_i,y_i)\}_{i=1}^n$ be an $(s,\eps_0,\delta)$ clusterable noisy dataset as described in Definition \ref{noisy model}. Let $\{\tilde{y}_i\}_{i=1}^n$ be the corresponding noiseless labels. Suppose $|\phi(0)|,|\phi'|,|\phi''|$ bounded by a constant, input noise and the number of hidden nodes satisfy 
%\[
%\eps_0\leq \order{\frac{\la(\Cb)}{K}}\quad\text{and}\quad k\geq \order{{\frac{K^4}{\la(\Cb)^4}}}.
%\]
%where $\Cb\in\R^{K\times d}$ is the matrix of cluster centers. Set learning rate $\eta\leq\order{ \frac{K}{n\opnorm{\Cb}^2}}$ and randomly initialize $\W_0\distas\Nn(0,1)$. With probability $1-2/K^{100}$, after $\tau\geq \order{\frac{K}{\eta n\la(\Cb)}}$ iterations, for all $1\leq i\leq n$, with respect to {\em{true labels}} $\{\tilde{y}_i\}_{i=1}^n$ we have
%\[
%|f(\W_{\tau},\x_i)-\tilde{y}_i|\leq 4\rho+\order{\frac{\eps_0K\sqrt{n}}{\la(\Cb)}}\quad\text{and}\quad \frac{\tn{f(\W_{\tau},\X)-\tilde{\y}}}{\sqrt{n}}\leq 4\rho+\order{\frac{\eps_0K}{\la(\Cb)}}.
%\]
%\end{theorem}

% While this perturbation argument is fairly intuitive (when $\eps_0$ is small), its formalization introduces further technicalities and is left to future work.% of this perturbation argument
% label noise has no contribution to classification error and the error only depends on the input imperfections ($\eps_0>0$).
\noindent Theorem \ref{main thm robust2} shows that gradient descent with early stopping is robust and predicts correct labels despite constant corruption. $\eps_0$ neighborhood of the cluster centers can be viewed as the test data since it corresponds to the support of the input distribution where data is sampled from. Further properties are discussed below.

\vspace{-0pt}
\noindent\textbf{Robustness.} The solution found by gradient descent with early stopping degrades gracefully as the label corruption level $\rho$ grows. In particular, as long as $\rho\leq \delta/8$, the final model is able to correctly classify any input data. In particular, when applied to the training data \eqref{pls satisfy this eq2} yields
$\arg\min_{\alpha_\ell:1\leq \ell\leq \bar{K}}|f(\W_{\tau},\x)-\alpha_\ell|=\widetilde{y}_i$ so that the network labels are identical to the ground truth labels completely removing the corruption on the training data. In our setup, intuitively the label gap obeys $\delta\sim \frac{1}{\bar{K}}$, hence, we prove robustness to\vspace{-0pt}
\[
\text{Total number of corrupted labels}\lesssim\frac{n}{\bar{K}}.
\]
This result is independent of number of clusters and only depends on number of classes. An interesting future direction is to improve this result to allow on the order of $n$ corrupted labels. %Such a result maybe possible by using a multi-output classification neural network. 

%For binary classification, this implies robustness to a constant fraction of corrupted labels. 
\vspace{-0pt}
\noindent\textbf{Early stopping time.} We only need few iterations to find a good model: Using proposed step size, the iteration number is at most order $K$ and typically scales as $\max(1,K/d)$ up to condition numbers. 

%Interestingly, under our data model t
\vspace{-0pt}
\noindent\textbf{Modest overparameterization.} Our result applies as soon as the number of hidden units in the network exceeds $K^2\|\Cb\|^4$ which lies between $K^2$ and $K^4$ which is independent of the sample size $n$. This can be interpreted as network having enough capacity to fit the cluster centers $\{\cb_\ell\}_{\ell=1}^K$ and their true labels. If cluster centers are incoherent (e.g.~random) and $K\geq d$, the required number of parameters in the network ($k\times d$) scales as $dK^2\|\Cb\|^4\lesssim K^4$.

%Interestingly, the required overparameterization is essentially independent of the size of the training data $n$ (ignoring logarithmic terms) and only depends on $K$.
%ing on the number of clusters.%conditioning of the data points,  and conditioning of the cluster centers. 
\vspace{-0cm}
\noindent\textbf{Distance from initialization.} Another feature of Theorem \ref{main thm robust2} is that the network weights do not stray far from the initialization as the distance between the initial model and the final model (at most) grows with the square root of the number of clusters ($\sqrt{K}$). Intuitively, more clusters correspond to a richer data distribution, hence we need to travel further away to find a viable model. While our focus in this work is early stopping, the importance of distance to initialization motivates the use of $\ell_2$-regularization with respect to the initial point i.e.~solving the regularized empirical risk minimization\vspace{-0pt}
\begin{align}
\W_{\text{ridge}}=\arg\min_{\W}\frac{1}{2}\sum_{i=1}^n (y_i-f(\W,\x_i))^2+\la \tf{\W-\W_0}^2\label{ridge}
\end{align}
where $\W_0$ is the point of initialization for the gradient based algorithm that will be used to solve \eqref{ridge}.
%However, the final model with early stopping is guaranteed to stay within $\sqrt{K}$ radius.
%This dependence is intuitive as the {\em{Rademacher complexity}} of the function space is dictated by the distance to initialization and should grow with the square-root of the number of input clusters to ensure the model is expressive enough to learn the dataset. 
%\\\SO{FILL/FIX BELOW}

\vspace{-5pt}
\subsection{To (over)fit to corrupted labels requires straying far from initialization}\vspace{-5pt}
In this section we wish to provide further insight into why early stopping enables robustness and generalizable solutions. Our main insight is that while a neural network maybe expressive enough to fit a corrupted dataset, the model has to travel a longer distance from the point of initialization as a function of the distance from the cluster centers $\eps_0$ and the amount of corruption. We formalize this idea as follows. Suppose (1) two input points are close to each other (e.g.~they are from the same cluster), (2) but their labels are different, hence the network has to map them to distant outputs. Then, the network has to be large enough so that it can amplify the small input difference to create a large output difference. Our first result formalizes this for a randomly initialized network. Our random initialization picks $\W$ with i.i.d.~standard normal entries which ensures that the network is isometric i.e.~given input $\x$, $\E[f(\W,\x)^2]=\order{\tn{\x}^2}$.
%$\W_0\sim\Nn(0,1)^{k\times d}$, 
%\vspace{-0.2cm}
\begin{theorem}\label{double pert}
	Let $\vct{x}_1,\vct{x}_2\in\R^d$ be two vectors with unit Euclidean norm obeying $\twonorm{\vct{x}_2-\vct{x}_1}\le \epsilon_0$. Let $f(\mtx{W},\vct{x})=\vct{v}^T\phi\left(\mtx{W}\vct{x}\right)$ where $\vb$ is fixed, $\mtx{W}\in\R^{k\times d}$, and $k\ge c d$ with $c>0$ a fixed constant. Assume $\abs{\phi'},\abs{\phi''}\le \Gamma$. Let $y_1$ and $y_2$ be two scalars satisfying $\abs{y_2-y_1}\ge \delta$. Suppose $\W_0\distas\Nn(0,1)$. Then, with probability $1-2e^{-(k+d)}-2e^{-\frac{t^2}{2}}$, for any $\mtx{W}\in\R^{k\times d}$ such that $\fronorm{\mtx{W}-\mtx{W}_0}\le c\sqrt{k}$ and
	\vspace{-0.1cm}
	\begin{align*}
	f(\mtx{W},\vct{x}_1)=y_1\quad\text{and}\quad f(\mtx{W},\vct{x}_2)=y_2,
	\end{align*}
	holds, we have $\opnorm{\W-\W_0}\ge \frac{\delta}{C\Gamma\eps_0}-\frac{t}{1000}$.
\end{theorem}% The smaller input noise implies a more difficult optimization process.
In words, this result shows that in order to fit to a dataset with a {\em{single corrupted label}}, a randomly initialized network has to traverse a distance of at least $\delta/\eps_0$.  Lemma \ref{simple pert} in the supplementary clarifies the role of the corruption amount $\rho$ and shows that more label corruption within a fixed class requires a model with a larger norm in order to fit the labels. 
\noindent\textbf{Can we really overfit to corruption?}~A natural question is whether early stopping is necessary i.e.~can we perfectly interpolate to the corrupted dataset model of Definition \ref{noisy model}. The recent works \cite{du2018gradient2,allen2018convergence,anon2019overparam} on neural net optimization answers this affirmatively. In particular, as long as no two input samples are not identical, sufficiently wide neural networks trained with gradient descent can provably and perfectly interpolate a corrupted dataset. %However, the (provable) width of such networks will grow inversely with the minimum distance between the input samples i.e.~smaller $\eps_0$ implies larger capacity networks which is consistent with increasing difficulty of overfitting we expect from Theorem \ref{double pert}.
%Hence, such networks requ
%\vspace{-3pt}
\subsection{Key Technical Ideas}\vspace{-3pt}\label{keyidea}
%\vspace{-0.2cm}
Our key proof idea is that semantically meaningful datasets (such as the clusterable dataset model) should have a low-dimensional representation. We use Jacobian mapping of the neural network to capture such structure in data. Specifically, we leverage the approximate low-rankness of the Jacobian matrix\vspace{-5pt}
\[
\Jc(\W)=\begin{bmatrix}\frac{\pa f(\x_1,\W)}{\pa \W}~\dots ~\frac{\pa f(\x_n,\W)}{\pa \W}\end{bmatrix}^T.
\]
We show that the optimization is implicitly decomposed into two stages which corresponds to the column subspaces induced by the large and small singular values of the Jacobian. First, denoting the overall network prediction by $f(\W)=[f(\W,\x_1)~\dots~(\W,\x_n)]^T$, we represent the residual as\vspace{-2pt}%decompose the prediction residual in terms of clean residual and label noise as 
\[
\underbrace{\y-f(\W_\tau)}_{\text{corrupted residual}}=\underbrace{\tilde{\y}-f(\W_\tau)}_{\text{clean residual}}+\underbrace{\y-\tilde{\y}}_{\text{label corruption}}
\]
Under our dataset model, we prove that clean residual is aligned with the large singular subspace whereas label noise is aligned with the small subspace. As a result, gradient descent learns the useful information (clean residual) in few iterations whereas it takes much longer to overfit to noise justifying the use of early stopping. % due to small singular values.
%%
%whose entries are . We study the interaction between the Jacobian mapping and the prediction residual. Specifically, 

%\input{generic}
\section{Numerical experiments}\label{sec numer}

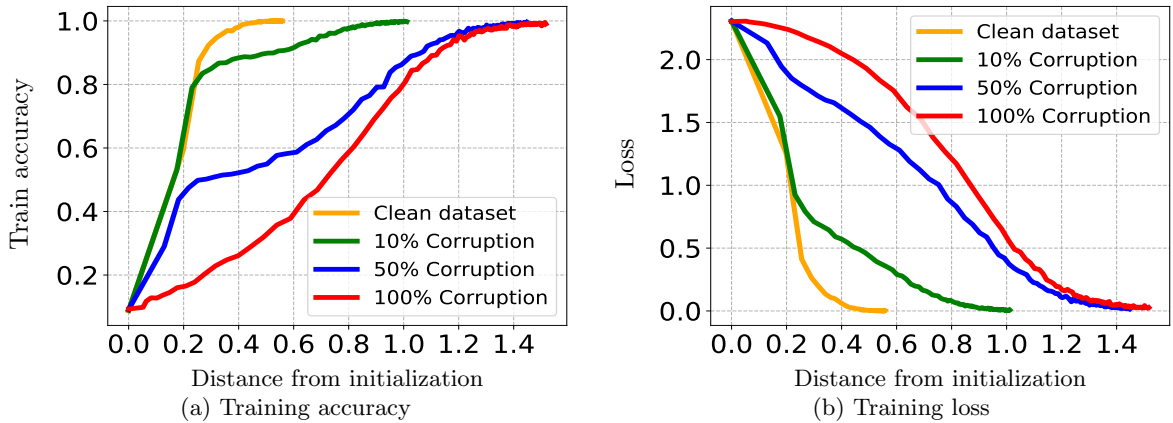
\begin{figure}[t!]

\begin{centering}
\begin{subfigure}[t]{3.1in}
\begin{tikzpicture}
\node at (0,0) {\includegraphics[height=0.7\linewidth,width=1\linewidth]{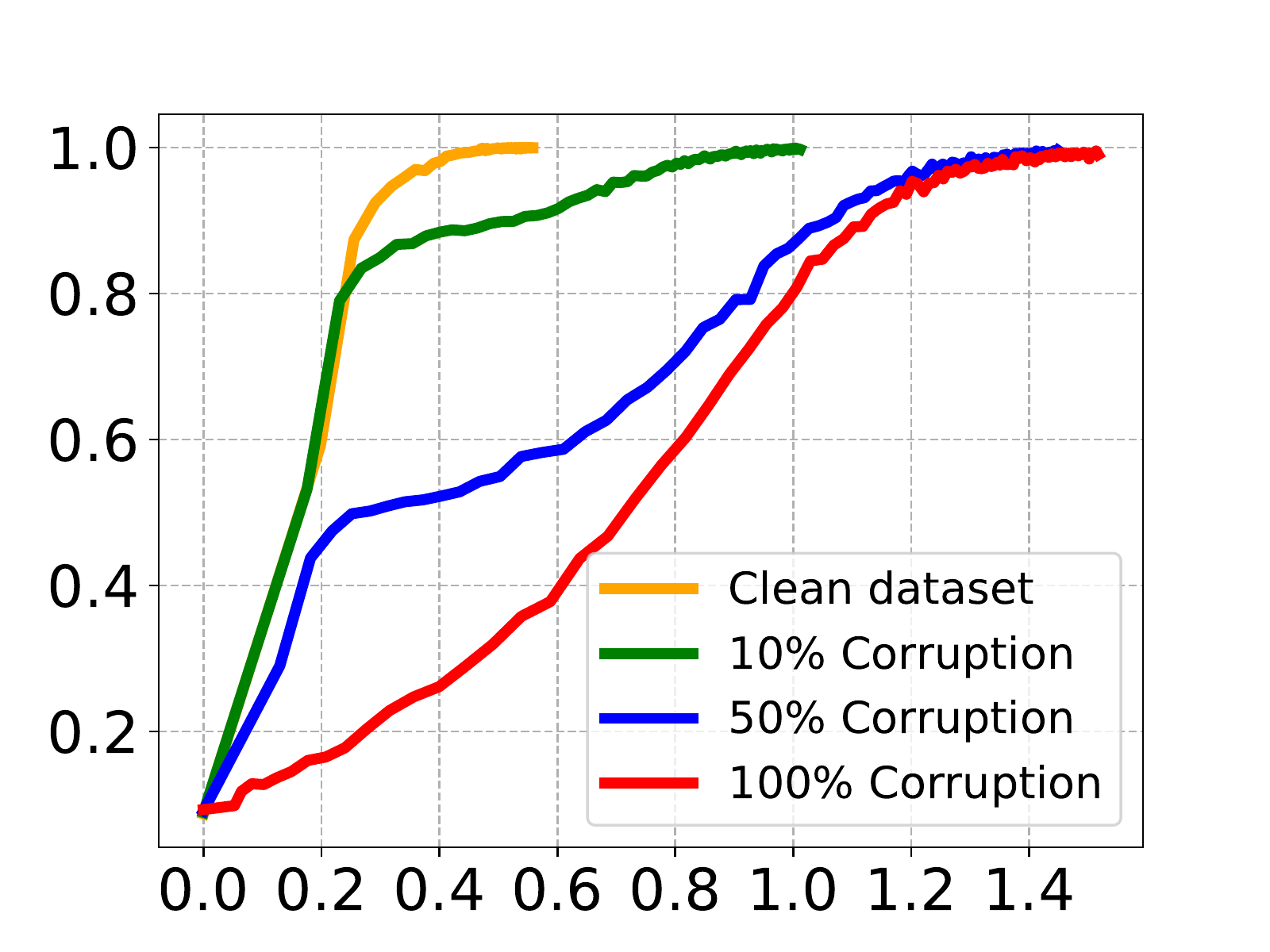}};
\node at (0.1,-2.85) {\small{Distance from initialization}};
\node[rotate=90] at (-4.1,0) {Train accuracy};
\end{tikzpicture}\vspace{-6pt}
%\node at (1,1) {source};
%\includegraphics[height=0.7\linewidth,width=1\linewidth]{figs/neural_net2_500.pdf}\vspace{-5pt}\
\caption{Training accuracy}
\label{fig5a}
\end{subfigure}
\end{centering}\hspace{-5pt}
\begin{centering}
\begin{subfigure}[t]{3.1in}
\begin{tikzpicture}
\node at (0,0) {\includegraphics[height=0.7\linewidth,width=1\linewidth]{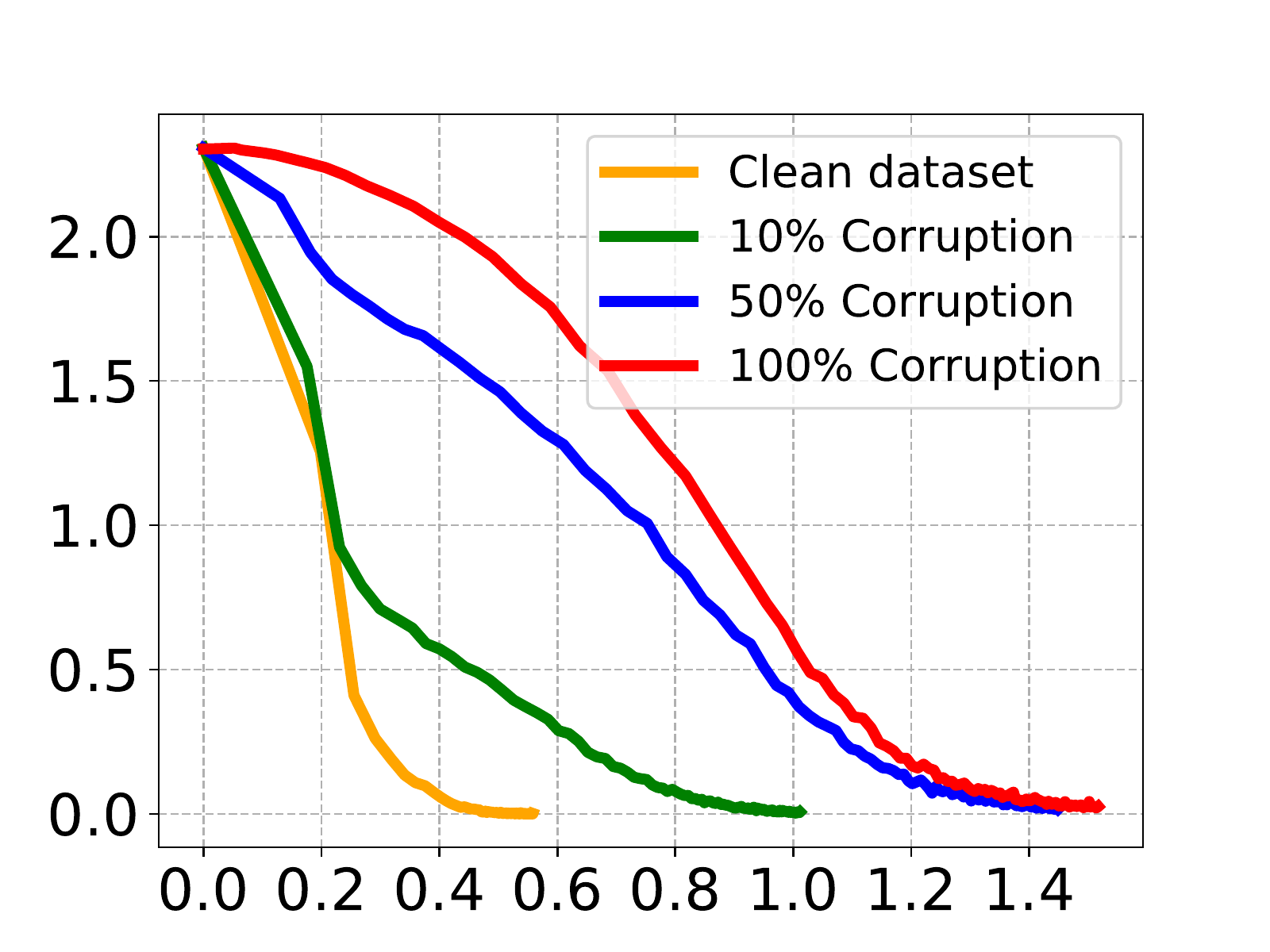}};
\node at (0.1,-2.85) {\small{Distance from initialization}};
\node[rotate=90] at (-4.1,0) {Loss};
%\node at (1.5,-1.8) {Normalized distance $\left(\frac{\tf{\mtx{W}^{\ell}-\mtx{W}^{\ell}_0}}{\tf{\mtx{W}_0^{\ell}}}\right)$};
\end{tikzpicture}\vspace{-6pt}
\caption{Training loss}\label{fig5b}
\end{subfigure}
\end{centering}\vspace{-0.2cm}\caption{We depict the training accuracy of a LENET model trainined on 3000 samples from MNIST as a function of relative distance from initialization. Here, the x-axis keeps track of the distance between the current and initial weights of all layers combined.}\label{fig5}
\vspace{-0.2cm}

\end{figure}

We conduct several experiments to investigate the robustness capabilities of deep networks to label corruption. In our first set of experiments, we explore the relationship between loss, accuracy, and amount of label corruption on the MNIST dataset to corroborate our theory. Our next experiments study the distribution of the loss and the Jacobian on the CIFAR-10 dataset. Finally, we simulate our theoretical model by generating data according to the corrupted data model of Definition \ref{noisy model} and verify the robustness capability of gradient descent with early stopping in this model.

In Figure \ref{fig5}, we train the same model used in Figure \ref{mnistacc} with $n=3,000$ MNIST samples for different amounts of corruption. Our theory predicts that more label corruption leads to a larger distance to initialization. To probe this hypothesis, Figure \ref{fig5a} and \ref{fig5b} visualizes training accuracy and training loss as a function of the distance from the initialization. These results demonstrate that the distance from initialization gracefully increase with more corruption. 
%We do remark that this phenomena appear to be more visible in simpler datasets such as MNIST. Figure \ref{fig5} shows that, MNIST on clean data finds a model much closer to the initialization compared to noisy training. %travels very little compared to noisy has much smaller distance to%in distance 

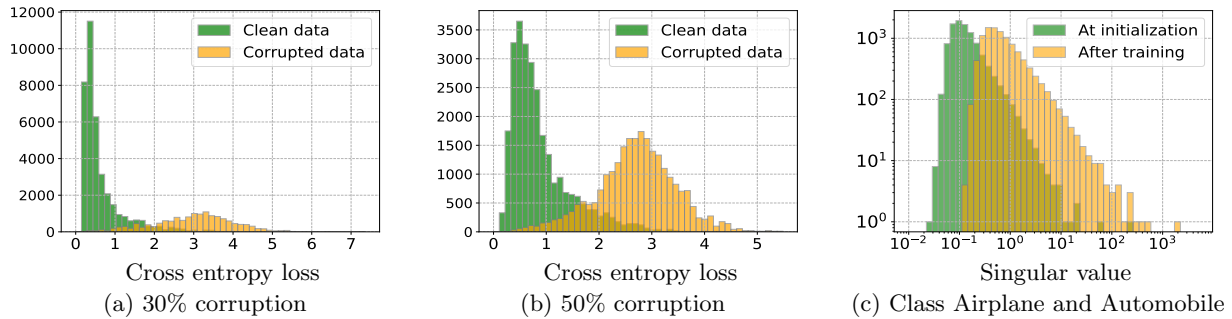
\begin{figure}[t!]
	
	\begin{centering}
		\begin{subfigure}[t]{2.15in}
			\begin{tikzpicture}
			\node at (0,0) {\includegraphics[height=0.7\linewidth,width=1\linewidth]{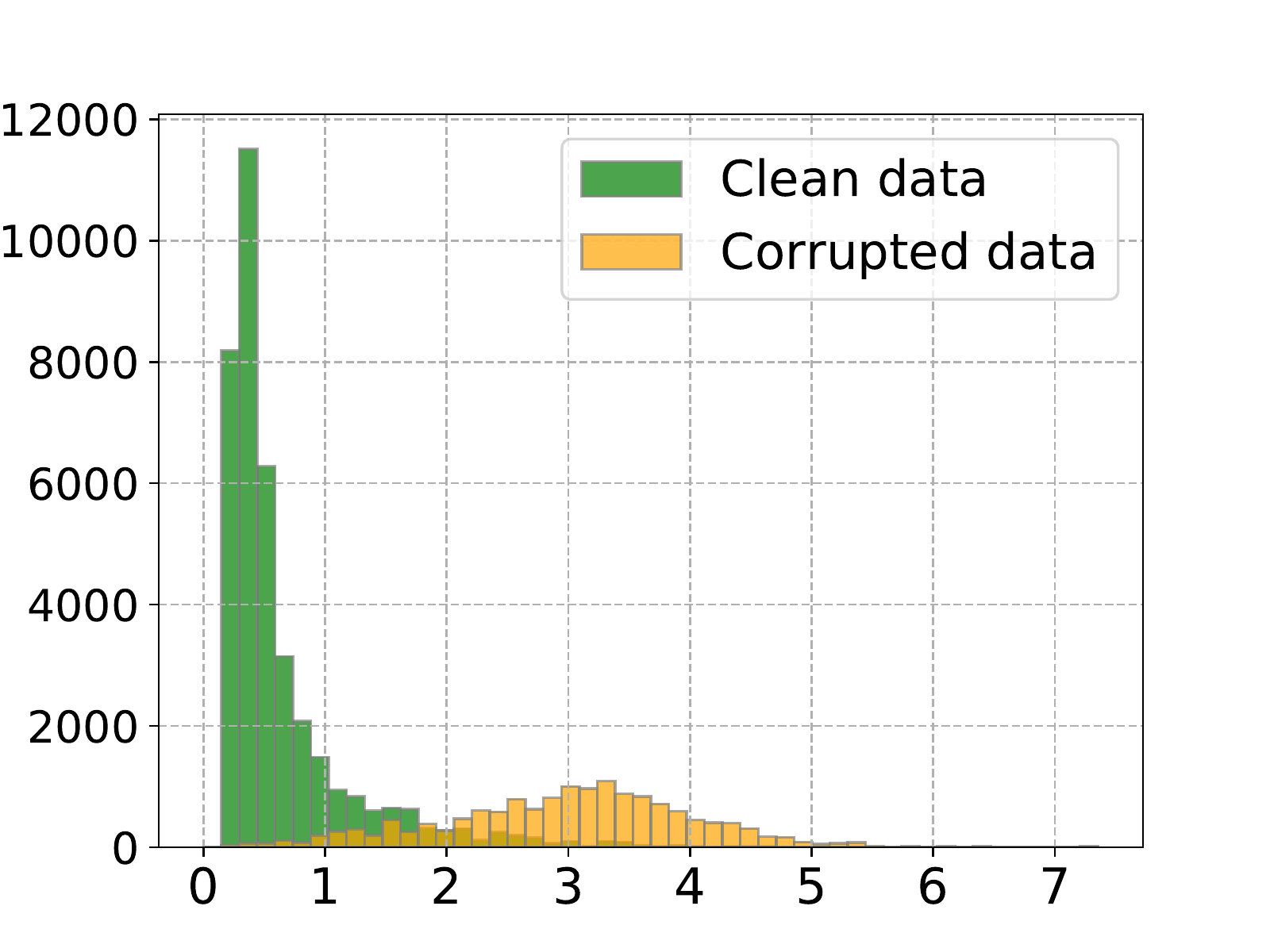}};
			\node at (0.1,-2.05) {\small{Cross entropy loss}};
			%\node[rotate=90] at (-1.9,0) {\small{Histogram}};
			\end{tikzpicture}\vspace{-7pt}
			\caption{30\% corruption}\label{fig1a}
		\end{subfigure}
	\end{centering}\hspace{-4pt}
	\begin{centering}
		\begin{subfigure}[t]{2.15in}
			\begin{tikzpicture}
			\node at (0,0) {\includegraphics[height=0.7\linewidth,width=1\linewidth]{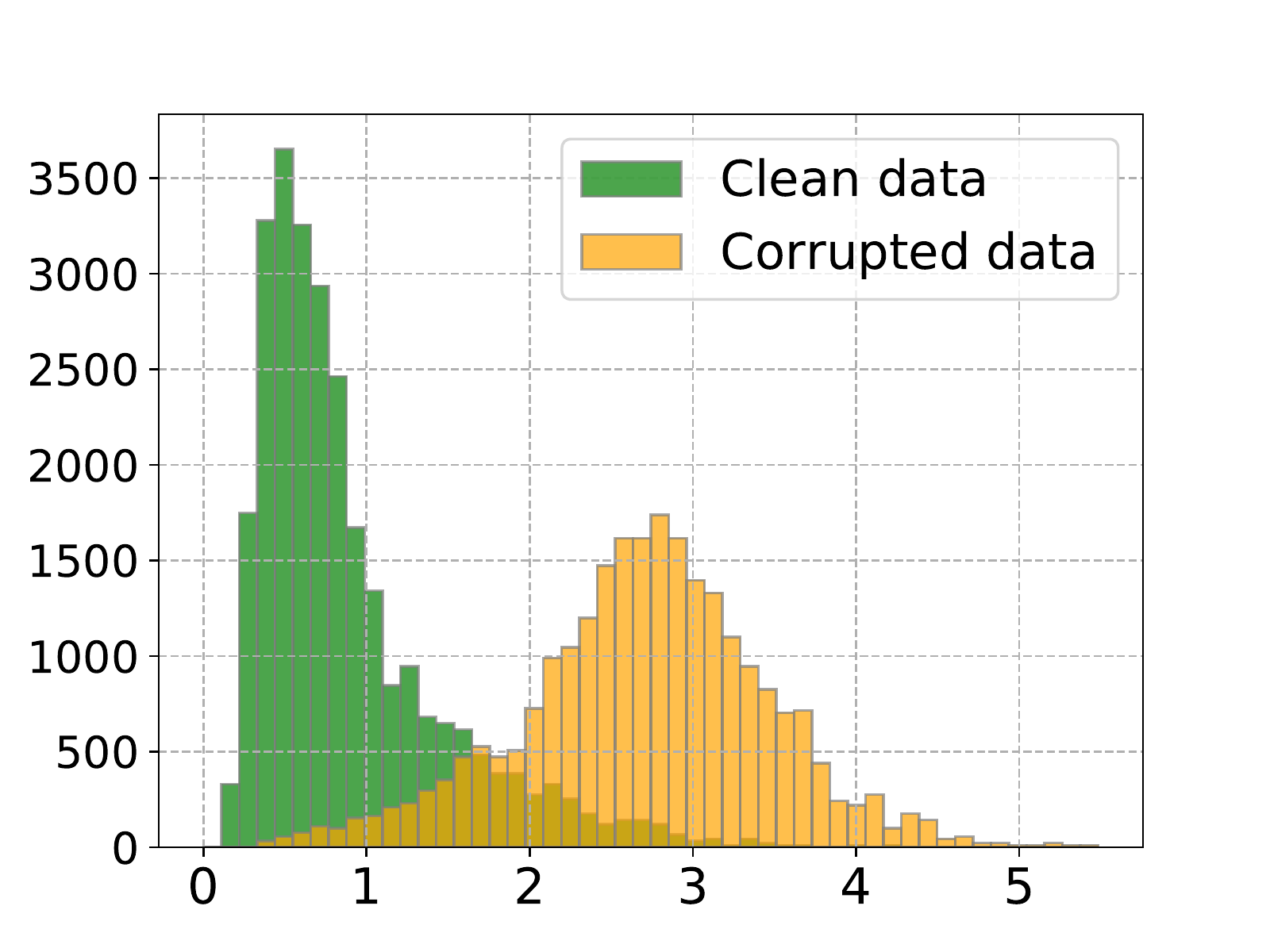}};
			\node at (0.1,-2.05) {\small{Cross entropy loss}};
			\end{tikzpicture}\vspace{-7pt}
			\caption{50\% corruption}\label{fig1b}
		\end{subfigure}
	\end{centering}\hspace{-5pt}
	\begin{centering}
		\begin{subfigure}[t]{2.15in}
			\begin{tikzpicture}
			\node at (0,0) {\includegraphics[height=0.7\linewidth,width=1\linewidth]{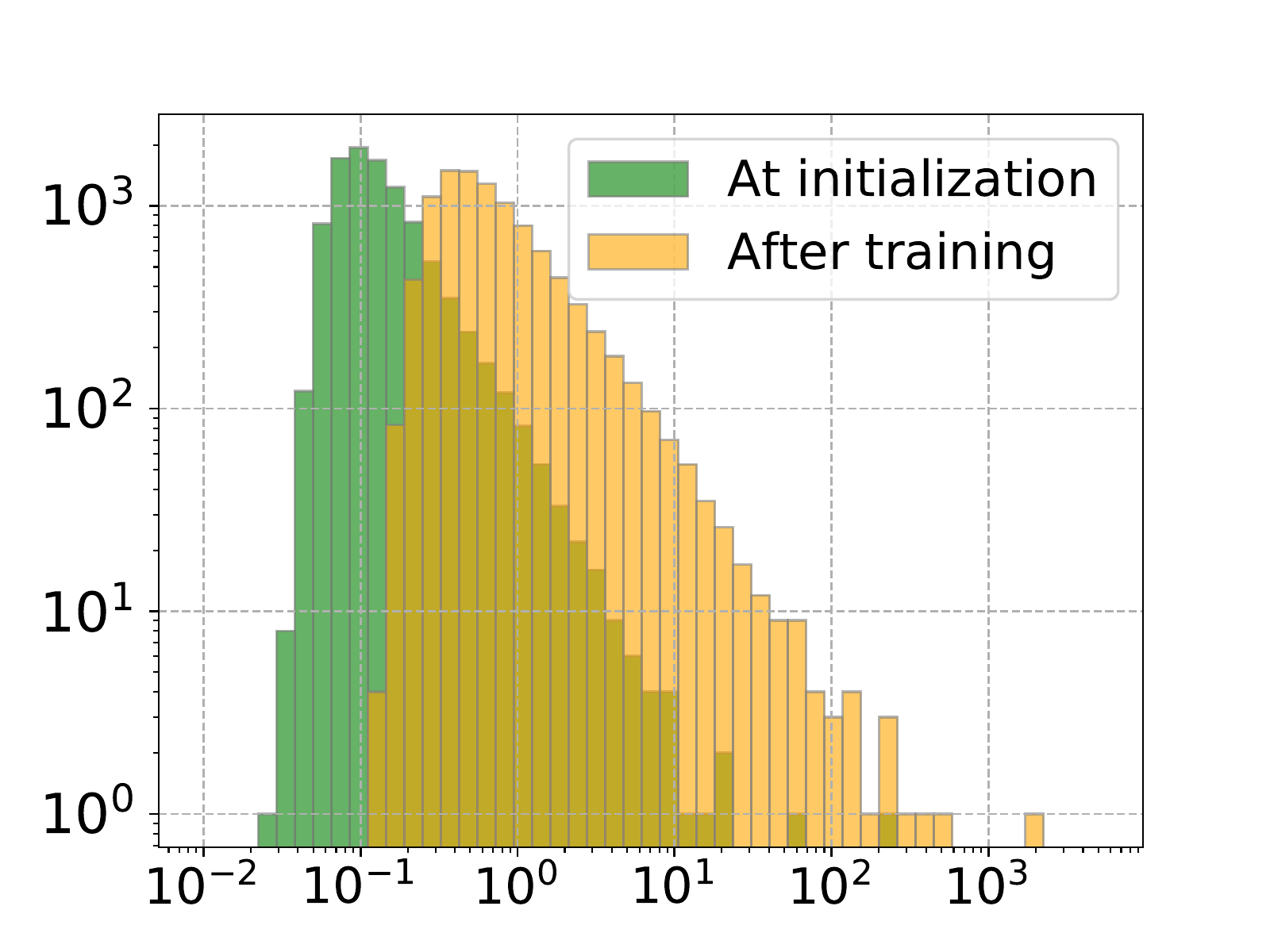}};
			\node at  (0.1,-2.05) {\small{Singular value}};
			\end{tikzpicture}\vspace{-7pt}
			\caption{Class Airplane and Automobile}\label{fig1c}
		\end{subfigure}
	\end{centering}

	\vspace{-0.2cm}\caption{(a)(b) Are histograms of the cross entropy loss of individual data points based on a model trained on 50,000 samples from CIFAR-10 with early stopping. The loss distribution of clean and corrupted data are separated but gracefully overlap as corruption increases. (c) is histogram of singular values obtained by forming the Jacobian by taking partial derivatives of class Airplane and Automobile on $10000$ samples.}\label{fig1}
	\vspace{-0.2cm}
\end{figure}

%In our first set of experiments, we explore the noise robustness of gradient descent on CIFAR-10 dataset to corroborate our theory. In second set of experiments, we simulate our theoretical model by generating data according to the noisy dataset model of Definition \ref{noisy model} and verify its robustness.

%Our CIFAR model does not perfectly fit to the clean samples because of the interference of noisy samples i.e.~there are clean samples which do not fit to the provided labels. As a result, the distribution of the clean and noisy data has a visible overlap. However, once we restrict our attention to clean samples that are fit by the model, w

%For Figures \ref{fig1} and \ref{fig2} we train with all 50,000 training samples. 
%While this is true for simple datasets such as MNIST (where training overfits with 100\% accuracy), o
%We opted for cross entropy as it is the standard classification loss however we remark that least-squares loss achieves similar accuracy. %opted for
Next, we study the distribution of the individual sample losses on the CIFAR-10 dataset. We conducted two experiments using Resnet-20 with cross entropy loss\footnote{We used cross entropy as it is the standard classification loss however least-squares achieves similar accuracy.}. In Figure \ref{fig1a} and \ref{fig1b} we assess the noise robustness of gradient descent where we used all 50,000 samples with either 30\% random corruption or 50\% random corruption. Theorem \ref{main thm robust} predicts that when the corruption level is small, the loss distribution of corrupted vs. clean samples should be separable. Figure \ref{fig1a} shows that when 30\% of the data is corrupted the distributions are approximately separable. When we increase the shuffling amount to 50\% in Figure \ref{fig1b}, the training loss on the clean data increases as predicted by our theory and the distributions start to gracefully overlap. %We remark that we only show the noisy data assigned to the incorrect label as there is 10\% chance the label will remain correct.
%Our MNIST experiments reveal the same pattern without the need for isolating the non-fitting clean labels. This is due to the fact that for MNIST the model can fit to the clean part of the corrupted training data (figure omitted).

As we briefly discussed in Section \ref{keyidea} (see proofs in the supplementary for more extensive discussion), our technical framework utilizes a bimodal prior on the Jacobian matrix \eqref{jacob eq} of the model. We now further investigate this hypothesis. For a binary class task, size of the Jacobian matrix is sample size ($n$) $\times$ total number of parameters in the model ($p$). The neural network model we used for CIFAR 10 has around $p=270,000$ parameters in total. In Figure \ref{fig1c} we illustrate the singular value histogram of binary Jacobian model where the training classes are Airplane and Automobile. We trained the model with all samples and focus on the histogram of all training data ($n=10,000$) before and after the training. In particular, only 10 to 20 singular values are larger than $0.1\times$ the top one. This is consistent with earlier works that studied the Hessian spectrum. Another intriguing finding is that the distribution of before and after training are fairly close to each other highlighting that even at random initialization, the Jacobian spectrum exhibits bimodal structure.
%\input{numerics_extra}

%In Figure \ref{fig4}, we turn our attention to verifying our theorem via a synthetic experiment. In this experiment, we generated $K=5$ classes with associated clusters centers are generated uniformly at random in $\R^{d=20}$. The class labels are $0$ to $4$ with a spacing of $1$. We added a noise amount of $\eps_0=0.1$ and the clusters are guaranteed to be within $0.2$ distance of each other. We picked a network with $k=1000$ hidden units. Figure \ref{fig4a} highlights the loss distribution with $50\%$ shuffle and implies that the gap between noisy and clean loss distribution is surprisingly resilient to massive noise. In Figure \ref{fig4b}, we boost the label noise to $70\%$ shuffle rate which leads to overlap between distributions. Overall, this experiment indicates that there might be room for improving the guarantees provided by Theorem \ref{main thm robust}. In particular, we believe it is possible to allow constant fraction of noise within each class and Figure \ref{fig4a} appears to support this.

%\begin{color}{red}
In Figure \ref{fig4}, we turn our attention to verifying our findings for the corrupted dataset model of Definition \ref{noisy model}. We generated $K=2$ classes where the associated clusters centers are generated uniformly at random on the unit sphere of $\R^{d=20}$. We also generate the input samples at random around these two clusters uniformly at random on a sphere of radius $\eps_0=0.5$ around the corresponding cluster center. Hence, the clusters are guaranteed to be at least $1$ distance from each other to prevent overlap. Overall we generate $n=400$ samples ($200$ per class/cluster). Here, $\bar{K}=K=2$ and the class labels are $0$ and $1$. We picked a network with $k=1000$ hidden units and trained on a data set with  $400$ samples where 30\% of the labels were corrupted. Figure \ref{fig4a} plots the trajectory of training error and highlights the model achieves good classification in the first few iterations and ends up overfitting later on. In Figures \ref{fig8a} and \ref{fig8b}, we focus on the loss distribution of \ref{fig4a} at iterations $80$ and $4500$. In this figure, we visualize the loss distribution of clean and corrupted data. Figure \ref{fig8a} highlights the loss distribution with early stopping and implies that the gap between corrupted and clean loss distributions is surprisingly resilient despite a large amount of corruption and the high-capacity of the model. In Figure \ref{fig8b}, we repeat plot after many more iterations at which point the model overfits. This plot shows that the distribution of the two classes overlap demonstrating that the model has overfit the corruption and lacks generalization/robustness.% overlap between distributions. 
%In Figure \ref{fig4b}, we boost the training sample to $4000$ which makes model resilient to label noise. 
%\end{color}

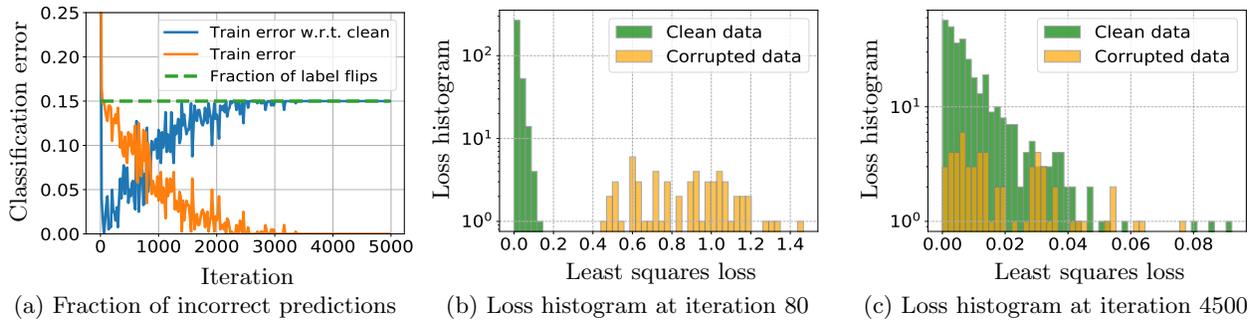
\begin{figure}[t!]
	\begin{centering}
		\begin{subfigure}[t]{2.15in}
			\begin{tikzpicture}
			\node at (0,0) {\includegraphics[height=0.7\linewidth,width=1\linewidth]{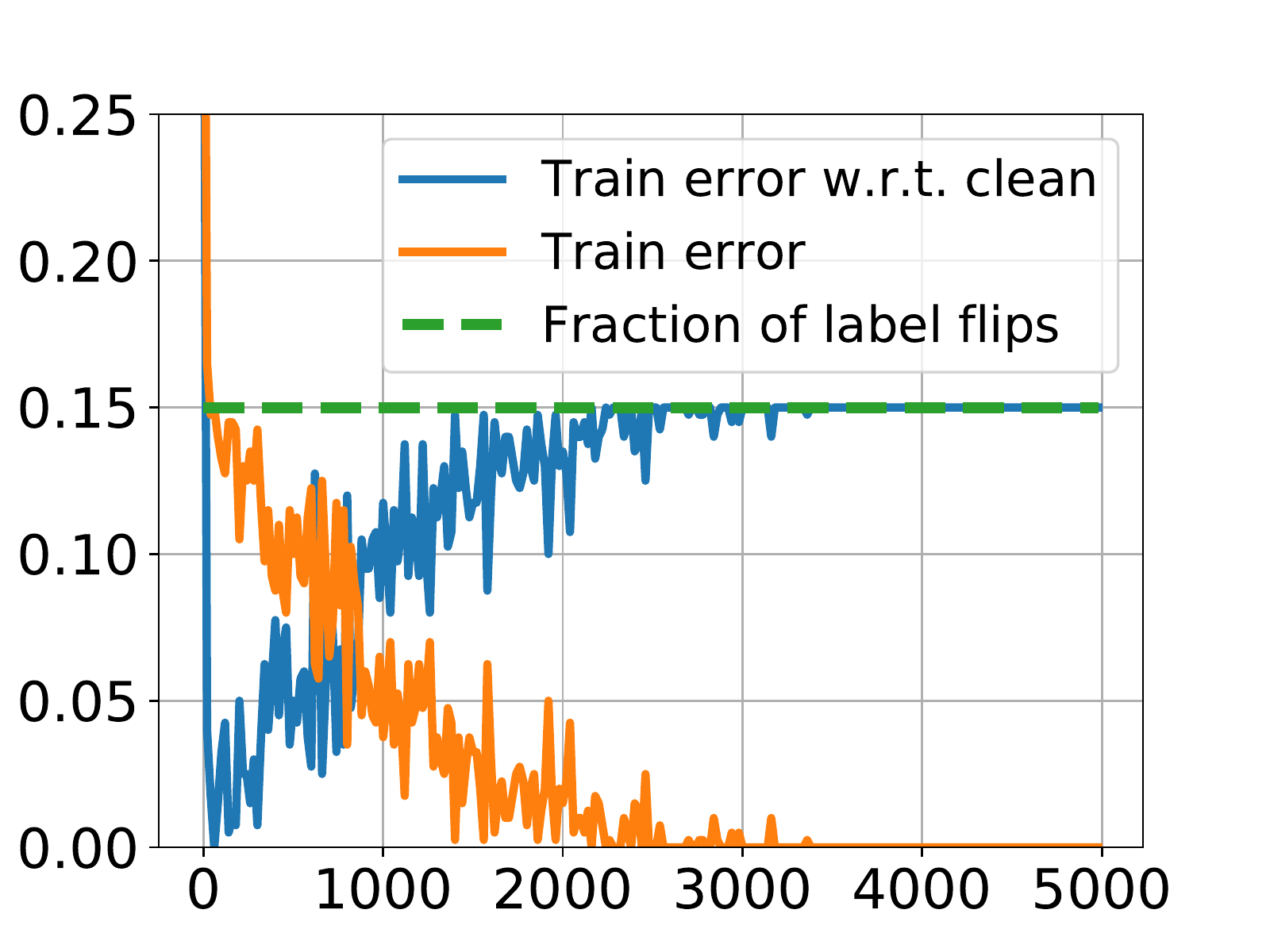}};
			\node at (0.1,-2.05) {\small{Iteration}};
			\node[rotate=90] at (-2.95,0) {\small{Classification error}};
			\end{tikzpicture}\vspace{-6pt}
			%\node at (1,1) {source};
			%\includegraphics[height=0.7\linewidth,width=1\linewidth]{figs/neural_net2_500.pdf}\vspace{-5pt}\
			\caption{Fraction of incorrect predictions}
			\label{fig4a}
		\end{subfigure}
	\end{centering}%\hspace{-5pt}
	\begin{centering}
		\begin{subfigure}[t]{2.15in}
			\begin{tikzpicture}
			\node at (0,0) {\includegraphics[height=0.7\linewidth,width=1\linewidth]{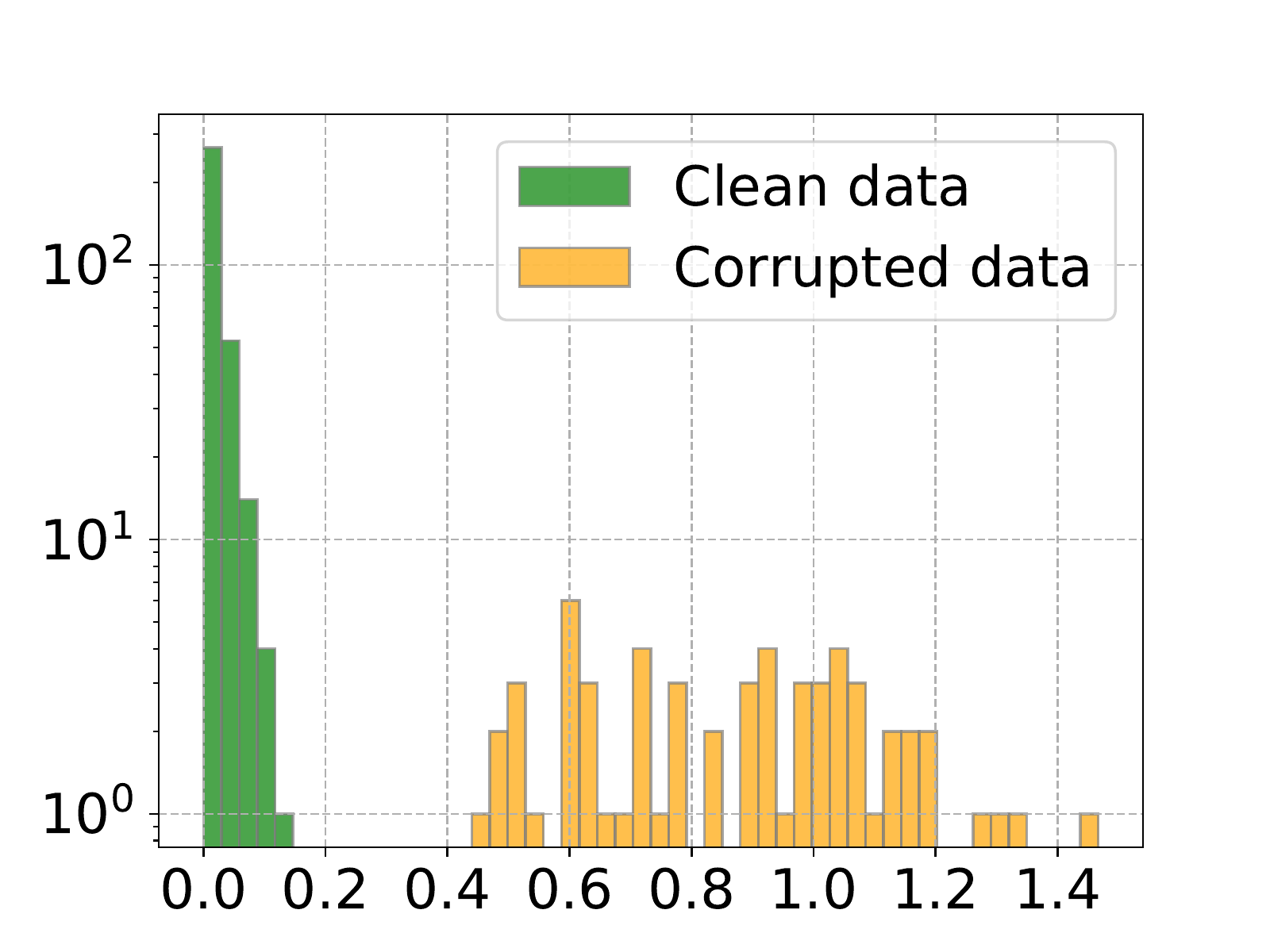}};
			\node at (0.1,-2.05) {\small{Least squares loss}};
			\node[rotate=90] at (-2.8,0) {\small{Loss histogram}};
			%\node at (1.5,-1.8) {Normalized distance $\left(\frac{\tf{\mtx{W}^{\ell}-\mtx{W}^{\ell}_0}}{\tf{\mtx{W}_0^{\ell}}}\right)$};
			\end{tikzpicture}\vspace{-6pt}
			\caption{Loss histogram at iteration 80}\label{fig8a}
		\end{subfigure}
	\end{centering}
	\begin{centering}
		\begin{subfigure}[t]{2.15in}
			\begin{tikzpicture}
			\node at (0,0) {\includegraphics[height=0.7\linewidth,width=1\linewidth]{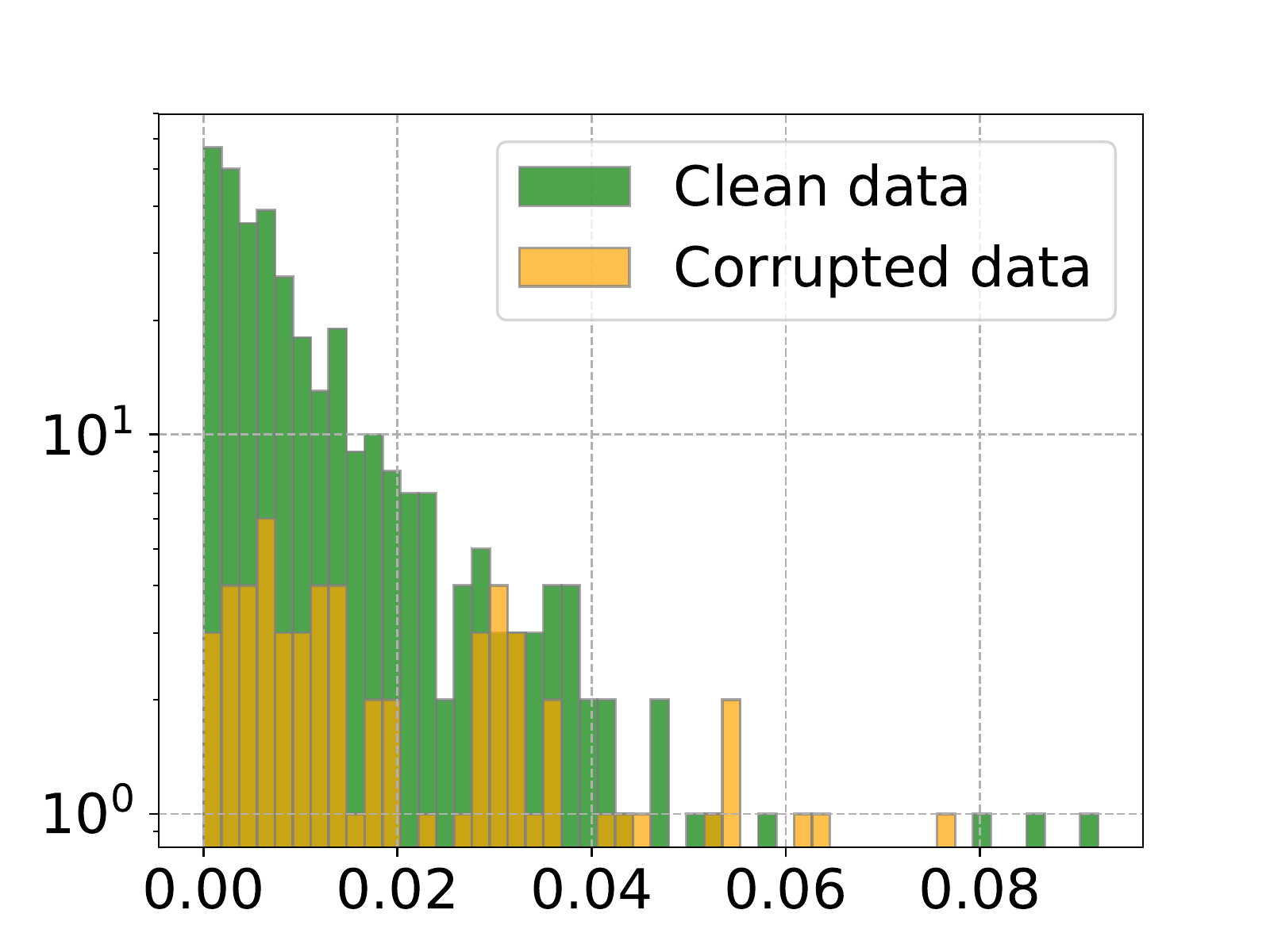}};
			\node at (0.1,-2.05) {\small{Least squares loss}};
			\node[rotate=90] at (-2.8,0) {\small{Loss histogram}};
			%\node at (1.5,-1.8) {Normalized distance $\left(\frac{\tf{\mtx{W}^{\ell}-\mtx{W}^{\ell}_0}}{\tf{\mtx{W}_0^{\ell}}}\right)$};
			\end{tikzpicture}\vspace{-6pt}
			\caption{Loss histogram at iteration 4500}\label{fig8b}
		\end{subfigure}
	\end{centering}\vspace{-0.1cm}\caption{We experiment with the corrupted dataset model of Definition \ref{noisy model}. We picked $K=2$ classes and set $n=400$ and $\eps_0=0.5$. Trained 30\% corrupted data with $k=1000$ hidden units. In average 15\% of labels actually flip which is highlighted by the dashed green line.}\label{fig4}
	\vspace{-0.3cm}
\end{figure}
\section{Conclusions}
\vspace{-0.2cm}
In this paper, we studied the robustness of overparameterized neural networks to label corruption from a theoretical lens. We provided robustness guarantees for training networks with gradient descent when early stopping is used and complemented these guarantees with lower bounds. Our results point to the distance between final and initial network weights as a key feature to determine robustness vs. overfitting which is inline with weight decay and early stopping heuristics. We also carried out extensive numerical experiments to verify the theoretical predictions as well as technical assumptions. While our results shed light on the intriguing properties of overparameterized neural network optimization, it would be appealing (i) to extend our results to deeper network architecture, (ii) to more complex data models, and also (iii) to explore other heuristics that can further boost the robustness of gradient descent methods.

\section*{Acknowledgements}
Authors would like to thank Yoav Freund for pointing out an issue with numerical experiments in the first draft of this manuscript. M. Soltanolkotabi is supported by the Packard Fellowship in Science and Engineering, a Sloan Research Fellowship in Mathematics, an NSF-CAREER under award \#1846369, the Air Force Office of Scientific Research Young Investigator Program (AFOSR-YIP)
under award \#FA9550-18-1-0078, an NSF-CIF award \#1813877, and a Google faculty research award. 
\newpage
\bibliographystyle{acm}
\bibliography{Bibfiles}
%\appendix
%\input{Appendix}
%\bibliographystyle{icml2019}
\newpage
\section{Improvements for perfectly cluster-able data}
We would like to note that in the limit of $\epsilon_0\rightarrow 0$ where the input data set is perfectly clustered one can improve the amount of overparamterization. Indeed, the result above is obtained via a perturbation argument from this more refined result stated below. 
%
%Conversely, $\la(\X)>0$ enforces a distance between cluster centers as $\la(\X)=0$ if two cluster centers are identical.
%, Definition \ref{cdata} does not explicitly enforce a distance between cluster centers.
%When $\W_0$ is randomly initialized and , it has been observed\SO{Add connection to related work!} 

%The following lemma relates the Jacobian of the network to NNCC.
%\begin{lemma} There exists a constant $c_0>0$ such that, with probability $1-K\exp(-n)$ over the generation of dataset $(\x_i,y_i)_{i=1}^n$, the neural net Jacobian obeys
%\[
%\|\bSi(\Cb)\|\Pb_{\Scp}\succeq \frac{1}{n}\Jc(\W)\Jc(\W)^T\succeq c_0\la(\Cb)\Pb_{\Scp}.
%\]
%\end{lemma}
%\begin{proof}
%\end{proof}
%{\textcolor{red}{LEFT HERE\\}}
%Our first theorem shows the robustness of gradient descent to label noise when input samples come from perfect clusters with $\eps_0=0$.
%\textcolor{red}{GET BACK HERE}
%Let $\y=[y_1~\dots~y_n]\in\R^n$ and let $\bar{\y}=[\bar{y}_1~\dots~\bar{y}_n]\in\R^n$ be the corresponding noiseless labels. 
\begin{theorem} [Training with perfectly clustered data] \label{main thm robust}Consier the setting and assumptions of Theorem \ref{main thm robust} with $\epsilon_0=0$. Starting from an initial weight matrix $\mtx{W}_0$ selected at random with i.i.d.~$\mathcal{N}(0,1)$ entries we run gradient descent updates of the form $\mtx{W}_{\tau+1}=\mtx{W}_\tau-\eta\nabla \mathcal{L}(\mtx{W}_\tau)$ on the least-squares loss \eqref{q loss} with step size $\eta\leq \frac{K}{2c_{up}n\Gamma^2\opnorm{\Cb}^2}$. Furthermore, assume the number of hidden nodes obey
\begin{align*}
k\ge C\Gamma^4\frac{K\log(K)\|\Cb\|^2}{\la(\Cb)^2},
\end{align*}
with $\lambda(\mtx{C})$ is the minimum cluster per Definition \ref{clust cov}. Then, with probability at least $1-2/K^{100}$ over randomly initialized $\W_0\distas\Nn(0,1)$, the iterates $\W_{\tau}$ obey the following properties.
\begin{itemize}
\item The distance to initial point $\W_0$ is upper bounded by
\[
\tf{\W_{\tau}-\W_0}\leq c\Gamma\sqrt{\frac{K\log K}{\la(\Cb)}}.
\]
\item After $\tau\geq\tau_0:= c{\frac{K}{\eta n\la(\Cb)}}\log \left(\frac{\Gamma\sqrt{n\log K}}{\rho}\right)$ iterations, the entrywise predictions of the learned network with respect to the  {\em{ground truth labels}} $\{\widetilde{y}_i\}_{i=1}^n$ satisfy 
\[
|f(\W_{\tau},\x_i)-\widetilde{y}_i|\leq 4\rho,%\order{\frac{Ks}{n}},
\]
for all $1\leq i\leq n$. Furthermore, if the noise level $\rho$ obeys $\rho\leq \delta/8$ the network predicts the correct label for all samples i.e.~%This implies that if $s\leq \order{\delta n}$, all labels (including noisy ones) will be correctly classified.
\begin{align}
\arg\min_{\alpha_\ell:1\leq \ell\leq \bar{K}}|f(\W_{\tau},\x_i)-\alpha_\ell|=\widetilde{y}_i\quad\text{for}\quad i=1,2,\ldots,n.\label{pls satisfy this eq}
\end{align}
\end{itemize}
\end{theorem}
This result shows that in the limit $\epsilon_0\rightarrow 0$ where the data points are perfectly clustered, the required amount of overparameterization can be reduced from $kd\gtrsim K^4$ to $kd\gtrsim K^2$. In this sense this can be thought of a nontrivial analogue of \cite{anon2019overparam} where the number of data points are replaced with the number of clusters and the condition number of the data points is replaced with a cluster condition number. This can be interpreted as ensuring that the network has enough capacity to fit the cluster centers $\{\cb_\ell\}_{\ell=1}^K$ and the associated true labels.
Interestingly, the robustness benefits continue to hold in this case. However, in this perfectly clustered scenario there is no need for early stopping and a robust network is trained as soon as the number of iterations are sufficiently large. In fact, in this case given the clustered nature of the input data the network never overfits to the corrupted data even after many iterations.

\section{To (over)fit to corrupted labels requires straying far from initialization}
\begin{lemma} \label{simple pert}Let $\cb\in\R^d$ be a cluster center. Consider $2s$ data points $\{\x_i\}_{i=1}^s$ and $\{\widetilde{\x}_i\}_{i=1}^{s}$ in $\R^d$ generated i.i.d.~around $\cb$ according to the following distribution
\[
\cb+\g\quad\text{with}\quad\g\sim\Nn(0,\frac{\eps_0^2}{d}\Iden_d).
\]
% be $\dots$ two unit Euclidian norm vectors satisfying $\tn{\x_1-\x_2}\leq \eps_0$. 
% Let $y_1,{y}_2$ be two scalars satisfying $|y_1|,|{y_2}|\leq 1$ and $|y_1-{y_2}|\geq \delta$.
%a unit Euclidian norm 
%, and assume the labels are $\delta$ separated i.e.~$|y-\tilde{y}|\geq \delta$
Assign $\{\x_i\}_{i=1}^s$ with labels $y_i=y$ and $\{\widetilde{\x}_i\}_{i=1}^s$ with labels $\widetilde{y}_i=\widetilde{y}$ and assume these two labels are $\delta$ separated i.e.~$\abs{y-\widetilde{y}}\ge \delta$. Also suppose $s\leq d$ and $|\phi'|\leq \Gamma$. Then, any $\W\in\R^{k\times d}$ satisfying
\[
f(\W,\x_i)=y_i\quad\text{and}\quad f(\W,\widetilde{\x}_i)=\widetilde{y}_i\quad\text{for}\quad i=1,\dots,s,
\]
obeys $\tf{\W}\geq \frac{\sqrt{s}\delta}{5\Gamma\eps_0}$ with probability at least $1-e^{-d/2}$. 
\end{lemma}
%This lemma highlights that more frequent label noise leads to a larger final network. 
Unlike Theorem \ref{double pert} this result lower bounds the network norm in lieu of the distance to the initialization $\mtx{W}_0$. However, using the triangular inequality we can in turn get a guarantee on the distance from initialization $\W_0$ via triangle inequality as long as $\tf{\W_0}\lesssim \order{{\sqrt{s}\delta/\eps_0}}$ (e.g.~by choosing a small $\eps_0$).%, one can also conclude that, it leads to a larger deviation from $\W_0$. 
%Observe that, unlike Theorem \ref{double pert}, this lemma guarantees distance from initialization only if initialization norm is smaller than  if the initialization $\W_0$ is sufficiently small (i.e.~$\|\W_0\|\leq \frac{\delta}{2\Gamma\eps_0}$), triangle inequality implies that distance from initialization $\|\W-\W_0\|$ is also large.

%when combined with Theorem \ref{main thm robust2} 
The above Theorem implies that the model has to traverse a distance of at least 
\[
\tf{\W_{\tau}-\W_0}\gtrsim  \sqrt{\frac{\rho n}{K}}\frac{\delta }{\eps_0},
\]
to perfectly fit corrupted labels. In contrast, we note that the conclusions of the upper bound in Theorem \ref{main thm robust2} show that to be able to fit to the uncorrupted true labels the distance to initialization grows at most by $\tau\eps_0$ after $\tau$ iterates. This demonstrates that there is a gap in the required distance to initialization for {\em{fitting enough to generalize}} and {\em{overfitting}}. To sum up, our results highlight that, one can find a network with good generalization capabilities and robustness to label corruption  within a small neighborhood of the initialization and that the size of this neighborhood is independent of the corruption. However, to fit to the corrupted labels, one has to travel much more, increasing the search space and likely decreasing generalization ability. Thus, early stopping can enable robustness without overfitting by restricting the distance to the initialization. 
% !TEX root = outlier.tex
\section{Technical Approach and General Theory}\label{sec gen thy}

In this section, we outline our approach to proving robustness of overparameterized neural networks. Towards this goal, we consider a general formulation where we aim to fit a general nonlinear model of the form $\vct{x}\mapsto f(\vct{\theta},\vct{x})$ with $\vct{\theta}\in\R^p$ denoting the parameters of the model. For instance in the case of neural networks $\vct{\theta}$ represents its weights. Given a data set of $n$ input/label pairs $\{(\x_i,y_i)\}_{i=1}^n\subset\R^d\times \R$, we fit to this data by minimizing a nonlinear least-squares loss of the form%as follows% and the problem
\begin{align*}
%\min_{\tn{\bteta}\leq R}\sum_{i=1}^n w_i(y_i-\x_i^T\bteta)^2.%\tn{\y-\X\bteta}^2.
\Lc(\bteta)=\frac{1}{2}\sum_{i=1}^n (y_i-f(\bteta,\x_i))^2.%+\la\tn{\bteta}^2.
\end{align*}
which can also be written in the more compact form
\begin{align*}
\Lc(\bteta)=\frac{1}{2}\twonorm{f(\vct{\theta})-\y}^2\quad\text{with}\quad f(\vct{\theta}):=
\begin{bmatrix}
f(\vct{\theta},\vct{x}_1)\\
f(\vct{\theta},\vct{x}_2)\\
\vdots\\
f(\vct{\theta},\vct{x}_n)
\end{bmatrix}.
\end{align*}
%Here, $\bteta\in\R^p$ denotes our model. 
To solve this problem we run gradient descent iterations with a constant learning rate $\eta$ starting from an initial point $\bteta_0$. These iterations take the form
\begin{align}
\label{GD}
\vct{\theta}_{\tau+1}=\vct{\theta}_{\tau}-\eta\nabla\mathcal{L}(\vct{\theta}_\tau)\quad\text{with}\quad\nabla\mathcal{L}(\vct{\theta})=\mathcal{J}^T(\vct{\theta})\left(f(\vct{\theta})-\vct{y}\right).
\end{align}
%, we will consider gradient descent to learn our model via the following iterations
%\[
%\bteta_{\tau+1}=\bteta_\tau-\eta\grad{\bteta_\tau}.
%\]
%Recent results indicate that deep networks can be optimized to achieve zero training error for a variety of loss functions. This property naturally extends to the label noise i.e.~if the labels are shuffled, an overparameterized DNN can still achieve perfect accuracy in the training data. In this work, we will address algorithms and theory for reconciling overfitting ability of deep networks and robustness to label noise. 
%We first introduce a general framework for label noise robust optimization. 
%Our approach is based on the hypothesis that the nonlinear model has a {\em{Jacobian matrix with bimodal spectrum}} where few singular values are large and remaining singular values are small. Jacobian is an $n\times p$ matrix associated to the nonlinear mapping $f$ and is given by
Here, $\mathcal{J}(\vct{\theta})$ is the $n\times p$ Jacobian matrix associated with the nonlinear mapping $f$ defined via
\begin{align}
\Jc(\bteta)=\begin{bmatrix}\frac{\pa f(\bteta,\x_1)}{\pa \bteta}~\dots~\frac{\pa f(\bteta,\x_n)}{\pa \bteta}\end{bmatrix}^T.\label{jacob eq}
\end{align}
\subsection{Bimodal jacobian structure}
Our approach is based on the hypothesis that the nonlinear model has a Jacobian matrix with {\em{bimodal spectrum where few singular values are large and remaining singular values are small}}. This assumption is inspired by the fact that realistic datasets are clusterable in a proper, possibly nonlinear, representation space. Indeed, one may argue that one reason for using neural networks is to automate the learning of such a representation (essentially the input to the softmax layer). We formalize the notion of bimodal spectrum below.
\begin{assumption}[Bimodal Jacobian] \label{lrank2} Let $\bp\geq\bn\geq \be>0$ be scalars. Let $f:\R^p\rightarrow\R^n$ be a nonlinear mapping and consider a set $\mathcal{D}\subset\R^p$ containing the initial point $\vct{\theta}_0$ (i.e.~$\vct{\theta}_0\in\mathcal{D}$). Let $\Scp\subset\R^n$ be a subspace  and $\Scn$ be its complement. We say the mapping $f$ has a Bimodal Jacobian with respect to the complementary subpspaces $\Scp$ and $\Scn$ as long as the following two assumptions hold for all $\vct{\theta}\in\mathcal{D}$.% we have% the conditionLipschitzness condition holds over any compact domain (for possibly large $\el$)
\begin{itemize}
\item {\bf{Spectrum over $\Scp$:}} For all $\vb\in\Scp$ with unit Euclidian norm we have
\[
\bn\leq \twonorm{\Jc^T(\bteta)\vb}\leq \bp.
\]
\item {\bf{Spectrum over $\Scn$:}} For all $\vb\in\Scn$ with unit Euclidian norm we have
\[
 \twonorm{\Jc^T(\bteta)\vb}\leq \be.
\]
\end{itemize}
We will refer to $\Scp$ as the {\em{signal subspace}} and $\Scn$ as the {\em{noise subspace}}.
\end{assumption}
When $\epsilon<<\alpha$ the Jacobian is approximately low-rank. An extreme special case of this assumption is where $\epsilon=0$ so that the Jacobian matrix is exactly low-rank. We formalize this assumption below for later reference.
%Our dataset model in Definition \ref{noisy model} naturally obeys this clusterability assumption. 
%For least-squares 
%\begin{assumption}[Jacobian Spectrum] \label{wcond} Consider a set $\mathcal{D}\subset\R^p$ containing the initial point $\vct{\theta}_0$ (i.e.~$\vct{\theta}_0\in\mathcal{D}$). We assume that for all $\bteta\in\Dc$ the following inequality holds
%\[
%\bn\le \sigma_{\min}\left(\mathcal{J}(\vct{\theta})\right)\le \|\mathcal{J}(\vct{\theta})\|\le \bp.
%\] 
%Here, $\sigma_{\min}(\cdot)$ and $\opnorm{\cdot}$ denote the minimum singular value and the spectral norm respectively.
%%holds for for all $\bteta\in\Dc$.
%\end{assumption}
%Given $\bteta_2,\bteta_1\in\Dc$, define $\Cb(\bteta_1,\bteta_2)=\Jc(\bteta_2)\Jc(\bteta_1)^T$. 
%\footnote{In general, we only need condition to hold for $\bdd=0$: If $\Dc$ is compact and $\frac{\partial \Jc(\bteta)}{\partial \bteta}$ is a continuous function, then the condition will hold for some (possibly small) $\bdd>0$.}
%All $\vb\in\Scp,\w\in\Scn$ with unit Euclidian norm satisfies\footnote{One can allow $\be$ small correlation between two subspaces. However, it results in additional complexity in the exposition that doesn't provide any intuition. Hence, we decided to make the stronger assumption.}
\begin{assumption}[Low-rank Jacobian] \label{lrank} Let $\bp\geq\bn>0$ be scalars. Consider a set $\mathcal{D}\subset\R^p$ containing the initial point $\vct{\theta}_0$ (i.e.~$\vct{\theta}_0\in\mathcal{D}$). Let $\Scp\subset\R^n$ be a subspace and $\Scn$ be its complement. For all $\vct{\theta}\in\mathcal{D}$, $\vb\in\Scp$ and $\w\in\Scn$ with unit Euclidian norm, we have that
\begin{align}
\bn\leq \twonorm{\Jc^T(\bteta)\vb}\leq \bp\quad\text{and}\quad \twonorm{\Jc^T(\bteta)\w}=0.\nn
\end{align}
%\footnote{}.
\end{assumption}
Our dataset model in Definition \ref{noisy model} naturally has a low-rank Jacobian when $\epsilon_0=0$ and each input example is equal to one of the $K$ cluster centers $\{\cb_{\ell}\}_{\ell=1}^K$. In this case, the Jacobian will be at most rank $K$ since each row will be in the span of $\big\{\frac{\pa f(\cb_{\ell},\bteta)}{\pa\bteta}\big\}_{\ell=1}^K$. The subspace $\Scp$ is dictated by the {\em{membership}} of each cluster as follows: Let $\Lambda_{\ell}\subset\{1,\dots,n\}$ be the set of coordinates $i$ such that $\x_i=\cb_{\ell}$. Then, subspace is characterized by
\[
\Scp=\{\vb\in\R^n\bgl \vb_{i_1}=\vb_{i_2}~~\text{for all}~~ i_1,i_2\in\Lambda_{\ell}~~\text{and}~~1\leq \ell\leq K\}.
\]%n approximately low-rank 
When $\epsilon_0>0$ and the data points of each cluster are not the same as the cluster center we have the bimodal Jacobian structure of Assumption \ref{lrank2} where over $\Scn$ the spectral norm is small but nonzero. 

In Section \ref{sec numer}, we verify that the Jacobian matrix of real datasets indeed have a bimodal structure i.e.~there are few large singular values and the remaining singular values are small which further motivate Assumption \ref{lrank}. This is inline with earlier papers which observed that Hessian matrices of deep networks have bimodal spectrum (approximately low-rank) \cite{sagun2017empirical} and is related to various results demonstrating that there are flat directions in the loss landscape \cite{hochreiter1997flat}.
%  a series of recent works \cite{shortest,du2018gradient,allen2018convergence,du2018gradient2,zou2018stochastic} which utilizes the core properties of the Jacobian matrix (or use related assumptions to quantify dataset diversity). However, unlike these works, our core idea is a low-rank prior on the Jacobian matrix which is supported by numerical experiments on real datasets (e.g.~see Figure \ref{}).

%A motivating example towards low-rank Jacobian hypothesis is clusterable data e.g.~mixture distributions. In the extreme case, This example provides useful intuition towards more complex problems. In fact, in Section \ref{},  we will utilize this distribution to provide convergence guarantees for learning overparameterized neural nets.  . %general analysis and will be utilized for proving results for overparameterized neural nets. 
%\begin{definition} [Clusterable data] Each input example $\x_i$
%\end{definition}

%\newpage
\subsection{Meta result on learning with label corruption}
Define the $n$-dimensional residual vector $\rb$ where
$\rb(\bteta)=\begin{bmatrix}f(\x_1,\bteta)-\y_1&\ldots&f(\x_n,\bteta)-\y_n\end{bmatrix}^T$.
A key idea in our approach is that we argue that (1) in the absence of any corruption $\rb(\bteta)$ approximately lies on the subspace $\Scp$ and (2) if the labels are corrupted by a vector $\eb$, then $\eb$ approximately lies on the complement space. Before we state our general result we need to discuss another assumption and definition.
\begin{assumption}[Smoothness] \label{spert} The Jacobian mapping $\mathcal{J}(\vct{\theta})$ associated to a nonlinear mapping $f:\R^p\rightarrow\R^n$ is $L$-smooth if for all $\vct{\theta}_1,\vct{\theta}_2\in\R^p$ we have $\opnorm{\mathcal{J}(\vct{\theta}_2)-\mathcal{J}(\vct{\theta}_1)}\le \el\twonorm{\vct{\theta}_2-\vct{\theta}_1}$.\footnote{Note that, if $\frac{\partial \Jc(\bteta)}{\partial \bteta}$ is continuous, the smoothness condition holds over any compact domain (albeit for a possibly large $\el$).}
%\footnote{}.
\end{assumption}

Additionally, to connect our results to the number of corrupted labels, we introduce the notion of subspace diffusedness defined below.%of the elements of the signal space described next. These can be verified for our clusterable dataset model.% we used for neural network training.% These are Our second assumption is a smoothness condition on Jacobian as follows.
%[Smoothness]
\begin{definition}[Diffusedness] \label{diff scp} $\Scp$ is $\gamma$ diffused if for any vector $\vb\in\Scp$ 
\begin{align*}
\tin{\vb}\leq\sqrt{\gamma/n}\tn{\vb},
\end{align*}
holds for some $\gamma>0$.
\end{definition}

The following theorem is our meta result on the robustness of gradient descent to sparse corruptions on the labels when the Jacobian mapping is exactly low-rank. Theorem \ref{main thm robust} for the perfectly clustered data ($\epsilon_0=0$) is obtained by combining this result with specific estimates developed for neural networks.%This result applies to outliers in regression tasks as well as label noise.
%$\Dc$=\Bc(\bteta_0,$. 
\begin{theorem} [Gradient descent with label corruption] \label{grad noise} Consider a nonlinear least squares problem of the form $\Lc(\bteta)=\frac{1}{2}\twonorm{f(\vct{\theta})-\y)}^2$ with the nonlinear mapping $f:\R^p\rightarrow\R^n$ obeying assumptions \ref{lrank} and \ref{spert} over a unit Euclidian ball of radius $\frac{4\tn{\rb_0}}{\bn}$ around an initial point $\bteta_0$ and $\y=[y_1~\dots~y_n]\in\R^n$ denoting the corrupted labels. Also let $\widetilde{\y}=[\widetilde{y}_1~\dots~\widetilde{y}_n]\in\R^n$ denote the uncorrupted labels and $\eb=\y-\widetilde{\y}$ the corruption. Furthermore, suppose the initial residual $f(\bteta_0)-\widetilde{\y}$ with respect to the uncorrupted labels obey $f(\bteta_0)-\widetilde{\y}\in\Scp$. Then, running gradient descent updates of the from \eqref{GD} with a learning rate $\eta\leq \frac{1}{2\bp^2}\min\left(1,\frac{\bn\beta}{L\twonorm{\vct{r}_0}}\right)$, all iterates obey
\[
\tn{\bteta_\tau-\bteta_0}\leq \frac{4\tn{\rb_0}}{\bn}.
\]
Furthermore, assume $\nu>0$ is a precision level obeying $\nu\geq \tin{\Pi_{\Scp}(\eb)}$. Then, after $\tau\geq \frac{5}{\eta\bn^2}\log \left(\frac{\tn{\rb_0}}{\nu}\right)$ iterations, $\bteta_\tau$ achieves the following error bound with respect to the true labels
\[
\tin{f(\bteta_\tau)-\widetilde{\y}}\leq 2\nu.%\tin{\Pi_{\Scp}(\eb)}.%\leq \frac{2\gamma\sqrt{s}}{n}\tn{\s}.
\]
Furthermore, if $\eb$ has at most $s$ nonzeros and $\Scp$ is $\gamma$ diffused per Definition \ref{diff scp}, then using $\nu=\tin{\Pi_{\Scp}(\eb)}$
\[
\tin{f(\bteta_\tau)-\widetilde{\y}}\leq 2\tin{\Pi_{\Scp}(\eb)}\le\frac{\gamma\sqrt{s}}{n}\tn{\eb}.%\tin{\Pi_{\Scp}(\eb)}.%\leq \frac{2\gamma\sqrt{s}}{n}\tn{\s}.
\]
\end{theorem}
This result shows that when the Jacobian of the nonlinear mapping is low-rank, gradient descent enjoys two intriguing properties. First, gradient descent iterations remain rather close to the initial point. Second, the estimated labels of the algorithm enjoy {\em{sample-wise}} robustness guarantees in the sense that the noise in the estimated labels are gracefully distributed over the dataset and the effects on individual label estimates are negligible. This theorem is the key result that allows us to prove Theorem \ref{main thm robust} when the data points are perfectly clustered ($\epsilon_0=0$). Furthermore, this theorem when combined with a perturbation analysis allows us to deal with data that is not perfectly clustered ($\epsilon_0>0$) and to conclude that with early stopping neural networks are rather robust to label corruption (Theorem \ref{main thm robust2}).
%We remark that due to our Jacobian prior, the operational regime of this result is highly overparameterized i.e.~we not only allow for $n\leq p$, the rank of the subspace $\Scp$ can also be much smaller than sample size. The proof idea is inspired from very recent works . 

Finally, we note that a few recent publication \cite{shortest, allen2018convergence,
du2018gradient} require the Jacobian to be well-conditioned to fit labels perfectly. In contrast, our low-rank model cannot perfectly fit the corrupted labels. Furthermore, when the Jacobian is bimodal (as seems to be the case for many practical data sets and neural network models) it would take a very long time to perfectly fit the labels and as demonstrated earlier such a model does not generalize and is not robust to corruptions. Instead we focus on proving robustness with early stopping.
\subsection{To (over)fit to corrupted labels requires straying far from initialization}
%We however stop the algorithm after a few iterations
%(or it would take a very long time ). Building on this, we show that gradient descent has nice robustness properties and the model tends to stay close to initialization. We believe Theorem \ref{grad noise} and its variations have potential applications in robust regression which is left to future work.
In this section we state a result that provides further justification as to why early stopping of gradient descent leads to more robust models without overfitting to corrupted labels.
This is based on the observation that while finding an estimate that fits the uncorrupted labels one does not have to move far from the initial estimate in the presence of corruption one has to stray rather far from the initialization with the distance from initialization increasing further in the presence of more corruption. We make this observation rigorous below by showing that it is more difficult to fit to the portion of the residual that lies on the noise space compared to the portion on the signal space (assuming $\bn\gg \be$). %$ shuffled labels result in residual to lie on the orthogonal space $\Scn$.
\begin{theorem} \label{lem how far?}Denote the residual at initialization $\bteta_0$ by $\rb_0=f(\bteta_0)-\y$. Define the residual projection over the signal and noise space as 
\[
E_+=\tn{\Pi_{\Scp}(\rb_0)}\quad\text{and}\quad E_-=\tn{\Pi_{\Scn}(\rb_0)}.
\]
Suppose Assumption \ref{lrank2} holds over an Euclidian ball $\Dc$ of radius $R< \max\left(\frac{E_+}{\bp},\frac{E_-}{\eps}\right)$ around the initial point $\bteta_0$ with $\bn\geq \be$. Then, over $\Dc$ there exists no $\bteta$ that achieves zero training loss. In particular, if $\Dc=\R^p$, any parameter $\bteta$ achieving zero training loss ($f(\bteta)=\y$) satisfies the distance bound
\[
\tn{\bteta-\bteta_0}\geq \max\left(\frac{E_+}{\bp},\frac{E_-}{\eps}\right).
\]
\end{theorem}
This theorem shows that the higher the corruption (and hence $E_{-}$) the further the iterates need to stray from the initial model to fit the corrupted data.
%Set $\eps=\frac{\gamma\sqrt{s}}{n}$. 
%Denote the residual at initialization $\bteta_0$ by $\rb_0=f(\bteta_0)-\y$. 

\section{Proofs}
\subsection{Proofs for General Theory}
We begin by defining the average Jacobian which will be used throughout our analysis.
\begin{definition} [Average Jacobian] \label{avg jacob}We define the average Jacobian along the path connecting two points $\vct{x},\vct{y}\in\R^p$ as
\begin{align}
&\Jc(\y,\x):=\int_0^1 \mathcal{J}(\x+\alpha(\y-\x))d\alpha.
\end{align}
\end{definition}
\begin{lemma} [Linearization of the residual]\label{lin res} Given gradient descent iterate $\hat{\bteta}=\bteta-\eta\grad{\bteta}$, define
\[
\Cb(\bteta)=\Jc(\hat{\bteta},\bteta)\Jc(\bteta)^T.
\]
The residuals $\hat{\rb}=f(\hat{\bteta})-\y$, $\rb=f(\bteta)-\y$ obey the following equation
\[
\hat{\rb}=(\Iden-\eta\Cb(\bteta))\rb.
\]
\end{lemma}
\begin{proof}
Following Definition \ref{avg jacob}, denoting $f(\hat{\bteta})-\y=\hat{\rb}$ and $f(\bteta)-\y=\rb$, we find that
\begin{align}
\hat{\rb}=&\rb-f(\bteta)+f(\hat{\bteta})\nn\\
\overset{(a)}{=}&\rb+\Jc(\hat{\bteta},\bteta)(\hat{\bteta}-\bteta)\nn\\
\overset{(b)}{=}&\rb-\eta\Jc(\hat{\bteta},\bteta)\Jc(\bteta)^T\rb\nn\\
%&=(\Iden-\eta\Jc(\bteta_{i},\bteta_{i})\Jc(\bteta_{i})^T)\rb_{i}\\
=&~(\Iden-\eta\Cb(\bteta))\rb.\label{line 4}
\end{align}
Here (a) uses the fact that Jacobian is the derivative of $f$ and (b) uses the fact that $\grad{\bteta}=\Jc(\bteta)^T\rb$.
\end{proof}
Using Assumption \ref{diff scp}, one can show that sparse vectors have small projection on $\Scp$. 
\begin{lemma} \label{lem sp proj}Suppose Assumption \ref{diff scp} holds. If $\rb\in\R^n$ is a vector with $s$ nonzero entries, we have that
\begin{align}
\tin{\Pi_{\Scp}(\rb)}\leq  \frac{\gamma\sqrt{s}}{n}\tn{\rb}.
\end{align}
\end{lemma}
\begin{proof} First, we bound the $\ell_2$ projection of $\rb$ on $\Scp$ as follows
\[
\tn{\Pi_{\Scp}(\rb)}=\sup_{\vb\in \Scp} \frac{\vb^T\rb}{\tn{\vb}}\leq \sqrt{\frac{\gamma}{n}}\tone{\rb}\leq \sqrt{\frac{\gamma s}{n}}\tn{\rb}.
\]
where we used the fact that $|\vb_i|\leq \sqrt{\gamma}\tn{\vb}/\sqrt{n}$. Next, we conclude with
\[
\tin{\Pi_{\Scp}(\rb)}\leq  \sqrt{\frac{\gamma}{n}}\tn{\Pi_{\Scp}(\rb)}\leq \frac{\gamma\sqrt{s}}{n}\tn{\rb}.
\]
\end{proof}
\subsubsection{Proof of Theorem \ref{grad noise}}
\begin{proof} The proof will be done inductively over the properties of gradient descent iterates and is inspired from the recent work \cite{shortest}. In particular, \cite{shortest} requires a well-conditioned Jacobian to fit labels perfectly. In contrast, we have a low-rank Jacobian model which cannot fit the noisy labels (or it would have trouble fitting if the Jacobian was approximately low-rank). Despite this, we wish to prove that gradient descent satisfies desirable properties such as robustness and closeness to initialization. Let us introduce the notation related to the residual. Set $\rb_\tau=f(\bteta_{\tau})-\y$ and let $\rb_0=f(\bteta_0)-\y$ be the initial residual. We keep track of the growth of the residual by partitioning the residual as $\rb_{\tau}=\rbb_\tau+\bar{\eb}_\tau$ where
\[
\bar{\eb}_\tau=\Pi_{\Scn}(\rb_\tau)\quad,\quad \rbb_\tau=\Pi_{\Scp}(\rb_\tau).
\]
We claim that for all iterations $\tau\geq 0$, the following conditions hold.
\begin{align}
\bar{\eb}_\tau=&\bar{\eb}_0\label{indind}\\
\twonorm{\rbb_\tau}^2\le&\left(1-\frac{\eta\bn^2}{2}\right)^\tau\twonorm{\rbb_0}^2,\label{err}\\
\frac{1}{4}\bn\twonorm{\vct{\theta}_\tau-\vct{\theta}_0}+\twonorm{\rbb_\tau}\le&\twonorm{\rbb_0}\leq \tn{\rb_0}.\label{close}
\end{align}
Assuming these conditions hold till some $\tau>0$, inductively, we focus on iteration $\tau+1$. First, note that these conditions imply that for all $\tau\geq i\geq 0$, $\bteta_i\in \Dc$ where $\Dc$ is the Euclidian ball around $\bteta_0$ of radius $\frac{4\tn{\rb_0}}{\bn}$. This directly follows from \eqref{close} induction hypothesis. Next, we claim that $\bteta_{\tau+1}$ is still within the set $\Dc$. This can be seen as follows:
\begin{claim}\label{next inside} Under the induction hypothesis \eqref{indind}, $\bteta_{\tau+1}\in\Dc$.
\end{claim} %and $\eta\leq 1/\bp^2$
\begin{proof} 
 Since range space of Jacobian is in $\Scp$ and $\eta\leq 1/\beta^2$, we begin by noting that
 \begin{align}
 \label{lem85temp2}
\tn{\bteta_{\tau+1}-\bteta_\tau}&=\eta \tn{\Jc^T(\bteta_\tau)\left(f(\vct{\theta_\tau})-\vct{y}\right)}\\
&\overset{(a)}{=}\eta \tn{\Jc^T(\bteta_\tau)\left(\Pi_{\Scp}(f(\vct{\theta_\tau})-\vct{y})\right)}\\
&\overset{(b)}{=}\eta \tn{\Jc^T(\bteta_\tau)\rbb_\tau}\\
&\overset{(c)}{\le} \eta \bp\tn{\rbb_\tau}\\
&\overset{(d)}{\le} \frac{\tn{\rbb_\tau}}{\bp}\\
&\overset{(e)}{\le} \frac{\tn{\rbb_\tau}}{\bn} 
%\\&\overset{(f)}{\le} \frac{\tn{\rbb_\tau}}{\alpb}.%\frac{\tn{\rb_\tau}}{\bn \lambda(1-\alpb/2)}
\end{align}
 In the above, (a) follows from the fact that row range space of Jacobian is subset of $\Scp$ via Assumption \ref{lrank}. (b) follows from the definition of $\rbb_\tau$. (c) follows from the upper bound on the spectral norm of the Jacobian over $\mathcal{D}$ per Assumption \ref{lrank}, (d) from the fact that $\eta\le \frac{1}{\bp^2}$, (e) from $\bn\le\bp$. The latter combined with the triangular inequality and induction hypothesis \eqref{close} yields (after scaling \eqref{close} by $4/\alpha$)
 %, and (f) from $\alpb \leq \alpha$
\begin{align*}
\tn{\bteta_{\tau+1}-\bteta_0}\leq \tn{\bteta_{\tau+1}-\bteta_\tau}+\tn{\bteta_{0}-\bteta_\tau}\le \tn{\bteta_\tau-\bteta_0}+\frac{\tn{\rbb_\tau}}{\bn}\leq \frac{4\tn{\rb_0}}{\bn},
\end{align*}
concluding the proof of $\vct{\theta}_{\tau+1}\in\mathcal{D}$. 
\end{proof}
To proceed, we shall verify that \eqref{close} holds for $\tau+1$ as well. Note that, following Lemma \ref{lin res}, gradient descent iterate can be written as
\[
\rb_{\tau+1}=(\Iden-\Cb(\bteta_{\tau}))\rb_\tau.
\]
Since both column and row space of $\Cb(\bteta_{\tau})$ is subset of $\Scp$, we have that
\begin{align}
\bar{\eb}_{\tau+1}&=\Pi_{\Scn}((\Iden-\Cb(\bteta_{\tau}))\rb_{\tau})\\
&=\Pi_{\Scn}(\rb_{\tau})\\
&=\bar{\eb}_{\tau},
\end{align}
This shows the first statement of the induction. Next, over $\Scp$, we have
\begin{align}
\rbb_{\tau+1}&=\Pi_{\Scp}((\Iden-\Cb(\bteta_{\tau}))\rb_{\tau})\\
&=\Pi_{\Scp}((\Iden-\Cb(\bteta_{\tau}))\rbb_{\tau})+\Pi_{\Scp}((\Iden-\Cb(\bteta_{\tau}))\bar{\eb}_{\tau})\\
&=\Pi_{\Scp}((\Iden-\Cb(\bteta_{\tau}))\rbb_{\tau})\\
&=(\Iden-\Cb(\bteta_{\tau}))\rbb_{\tau}\label{r tau+1}
\end{align}
where the second line uses the fact that $\bar{\eb}_{\tau}\in\Scn$ and last line uses the fact that $\rbb_{\tau}\in\Scp$. To proceed, we need to prove that $\Cb(\bteta_{\tau})$ has desirable properties over $\Scp$, in particular, it contracts this space.
\begin{claim} let $\Pb_{\Scp}\in\R^{n\times n}$ be the projection matrix to $\Scp$ i.e.~it is a positive semi-definite matrix whose eigenvectors over $\Scp$ is $1$ and its complement is $0$. Under the induction hypothesis and setup of the theorem, we have that\footnote{We say $\A\succeq \B$ if $\A-\B$ is a positive semi-definite matrix in the sense that for any real vector $\vb$, $\vb^T(\A-\B)\vb\geq 0$.}
\begin{align}
{\bp^2} \Pb_{\Scp}\succeq \Cb(\bteta_{\tau})\succeq \frac{1}{2}\Jc(\bteta_\tau)\Jc(\bteta_\tau)^T\succeq \frac{\bn^2}{2} \Pb_{\Scp}.\label{uplow}
\end{align}
\end{claim}
\begin{proof} The proof utilizes the upper bound on the learning rate. The argument is similar to the proof of Lemma 9.7 of \cite{shortest}. Suppose Assumption \ref{spert} holds. Then, for any $\vct{\theta}_1,\vct{\theta}_2\in\mathcal{D}$ we have
\begin{align}
\label{avgsmooth}
\opnorm{\mathcal{J}\left(\vct{\theta}_2,\vct{\theta}_1\right)-\mathcal{J}(\vct{\theta}_1)}=&\opnorm{\int_0^1\left(\mathcal{J}\left(\vct{\theta}_1+t\left(\vct{\theta}_2-\vct{\theta}_1\right)\right)-\mathcal{J}(\vct{\theta}_1)\right)dt},\nn\\
\le&\int_0^1\opnorm{\mathcal{J}\left(\vct{\theta}_1+t\left(\vct{\theta}_2-\vct{\theta}_1\right)\right)-\mathcal{J}(\vct{\theta}_1)} dt,\nn\\
\le&\int_0^1 tL\twonorm{\vct{\theta}_2-\vct{\theta}_1}dt\le\frac{L}{2}\twonorm{\vct{\theta}_2-\vct{\theta}_1}.
\end{align}
Thus, for $\eta\le \frac{\bn}{L\bp\twonorm{\vct{r}_0}}$, 
\begin{align}
\opnorm{\Jc(\bteta_{\tau+1},\bteta_\tau)-\Jc(\bteta_\tau)}&\leq \frac{\el}{2}\twonorm{\bteta_{\tau+1}-\bteta_\tau}\\
&=\frac{\eta\el}{2}\twonorm{\mathcal{J}^T(\bteta_\tau)\left(f(\vct{\theta_\tau})-\y\right)}\leq \frac{\eta \bp\el}{2} \twonorm{\rbb_{\tau}}\\
&\overset{(a)}{\leq} \frac{\eta \bp \el}{2}\tn{\rbb_0}\overset{(b)}{\leq} \frac{\bn}{2}.\label{establish}
\end{align}
where for (a) we utilized the induction hypothesis \eqref{close} and (b) follows from the upper bound on $\eta$. Now that \eqref{establish} is established, using following lemma, we find
\begin{align*}
\mathcal{C}(\bteta_\tau)=&\Jc(\bteta_{\tau+1},\bteta_\tau)\Jc(\bteta_\tau)^T\succeq (1/2)\Jc(\bteta_\tau)\Jc(\bteta_\tau)^T.
%\\
%=&\Jc(\bteta_{\tau+1},\bteta_\tau)\Jc(\bteta_\tau)^T-\Jc(\bteta_\tau)\Jc(\bteta_\tau)^T+\Jc(\bteta_\tau)\Jc(\bteta_\tau)^T,\\
%\succeq& \Jc(\bteta_\tau)\Jc(\bteta_\tau)^T-(1/2)\bn^2\Pb_{\Scp},\\%\mtx{I}_n\opnorm{\left(\Jc(\bteta_{\tau+1},\bteta_\tau)-\Jc(\bteta_\tau)\right)\Jc(\bteta_\tau)^T}
%\succeq& .
\end{align*}
The $\beta^2$ upper bound directly follows from Assumption \ref{lrank} by again noticing range space of Jacobian is subset of $\Scp$.
\begin{lemma} [Asymmetric PSD perturbation]\label{asym pert} Consider the matrices $\A,\Cb\in\R^{n\times p}$ obeying $\|\A-\Cb\|\leq \alpha/2$. Also suppose $\Cb\Cb^T\succeq \alpha^2\Pb_{\Scp}$. Furthermore, assume range spaces of $\A,\Cb$ lies in $\Scp$. Then,% for all $\rb\in\Scp$,
\[
\A\Cb^T\succeq \frac{\Cb\Cb^T}{2}\succeq \frac{\alpha^2}{2}\Pb_{\Scp}.%- \tn{\Cb^T\rb}^2\leq2\eps\tn{\Cb^T\rb}\tn{\rb}+\eps^2\tn{\rb}^2.
\]
\end{lemma}
\begin{proof} For $\rb\in\Scp$ with unit Euclidian norm, we have
\begin{align*}
\rb^T\A\Cb^T\rb&=\tn{\Cb^T\rb}^2+\rb^T(\A-\Cb)\Cb^T\rb\geq\tn{\Cb^T\rb}^2-\tn{\Cb^T\rb}\tn{\rb^T(\A-\Cb)}\\%+\rb^T\Cb (\A-\Cb)^T\rb+\rb^T(\B-\Cb)\Cb^T\rb.
&=(\tn{\Cb^T\rb}-\tn{\rb^T(\A-\Cb)})\tn{\Cb^T\rb}\\
&\geq (\tn{\Cb^T\rb}-\alpha/2)\tn{\Cb^T\rb}\\
&\geq\tn{\Cb^T\rb}^2/2.
\end{align*}
%\rb^T\A^T\Cb^T\rb\geq \alpha\tn{\rb}^2/2
Also, for any $\rb$, by range space assumption $\rb^T\A\Cb^T\rb=\Pi_{\Scp}(\rb)^T\A\Cb^T\Pi_{\Scp}(\rb)$ (same for $\Cb\Cb^T$). Combined with above, this concludes the claim.
\end{proof}
%we can bound $\Cb(\bteta_{\tau})$ as follows: First, we bound the distance between $\Cb(\bteta_{\tau})$ and $\Jc(\bteta_\tau)\Jc(\bteta_\tau)^T$ via
%\begin{align*}
%\opnorm{\left(\Jc(\bteta_{\tau+1},\bteta_\tau)-\Jc(\bteta_\tau)\right)\Jc(\bteta_\tau)^T}&\le \opnorm{\Jc(\bteta_{\tau+1},\bteta_\tau)-\Jc(\bteta_\tau)}\opnorm{\Jc(\bteta_\tau)}\\
%&\le \frac{(1/2)\bn^2}{\bp}\bp\\
%&=(1/2)\bn^2.%\\
%%&\leq (1/2)\sigma_{\min}^2\left(\Jc(\bteta_\tau)\right).
%\end{align*}
%Since column and row space of $\left(\Jc(\bteta_{\tau+1},\bteta_\tau)-\Jc(\bteta_\tau)\right)\Jc(\bteta_\tau)^T$ is subset of $\Scp$, this implies
%\[
%\left(\Jc(\bteta_{\tau+1},\bteta_\tau)-\Jc(\bteta_\tau)\right)\Jc(\bteta_\tau)^T\preceq \opnorm{\left(\Jc(\bteta_{\tau+1},\bteta_\tau)-\Jc(\bteta_\tau)\right)\Jc(\bteta_\tau)^T}\Pb_{\Scp}\preceq (1/2)\bn^2\Pb_{\Scp}.
%\]
%Thus, using $\Jc(\bteta_\tau)\Jc(\bteta_\tau)^T\succeq \bn^2\Pb_{\Scp}$, we have
%\begin{align*}
%\mathcal{C}(\bteta_\tau)=&\Jc(\bteta_{\tau+1},\bteta_\tau)\Jc(\bteta_\tau)^T,\\
%=&\Jc(\bteta_{\tau+1},\bteta_\tau)\Jc(\bteta_\tau)^T-\Jc(\bteta_\tau)\Jc(\bteta_\tau)^T+\Jc(\bteta_\tau)\Jc(\bteta_\tau)^T,\\
%\succeq& \Jc(\bteta_\tau)\Jc(\bteta_\tau)^T-(1/2)\bn^2\Pb_{\Scp},\\%\mtx{I}_n\opnorm{\left(\Jc(\bteta_{\tau+1},\bteta_\tau)-\Jc(\bteta_\tau)\right)\Jc(\bteta_\tau)^T}
%\succeq& (1/2)\Jc(\bteta_\tau)\Jc(\bteta_\tau)^T.
%\end{align*}
\end{proof}
What remains is proving the final two statements of the induction \eqref{close}. Note that, using the claim above and recalling \eqref{r tau+1} and using the fact that $\opnorm{\Jc(\bteta_{\tau+1},\bteta_\tau)}\leq \bp$, the residual satisfies
\begin{align}
\tn{\rbb_{\tau+1}}^2=\tn{(\Iden-\eta\Cb(\bteta_{\tau}))\rbb_{\tau}}^2&=\tn{\rbb_{\tau}}^2-2\eta\rbb_{\tau}^T\Cb_\tau\rbb_{\tau}+\eta^2\rbb_{\tau}^T\Cb_\tau^T\Cb_\tau\rbb_{\tau}\\
&\leq \tn{\rbb_{\tau}}^2-\eta\rbb_{\tau}^T\Jc(\bteta_\tau)\Jc(\bteta_\tau)^T\rbb_{\tau}+\eta^2\bp^2\rbb_{\tau}^T\Jc(\bteta_\tau)\Jc(\bteta_\tau)^T\rbb_{\tau}\\
&\leq \tn{\rbb_{\tau}}^2-(\eta-\eta^2\bp^2)\tn{\Jc(\bteta_\tau)^T\rbb_{\tau}}^2\\
&\leq \tn{\rbb_{\tau}}^2-\frac{\eta}{2}\tn{\Jc(\bteta_\tau)^T\rbb_{\tau}}^2.\label{contractive C}
\end{align}
where we used the fact that $\eta\leq \frac{1}{2\bp^2}$. Now, using the fact that $\Jc(\bteta_\tau)\Jc(\bteta_\tau)^T\succeq \bn^2\Pb_{\Scp}$, we have
\[
\tn{\rbb_{\tau}}^2-\frac{\eta}{2}\tn{\Jc(\bteta_\tau)^T\rbb_{\tau}}^2\leq (1-\frac{\eta\bn^2}{2})\tn{\rbb_{\tau}}^2\leq (1-\frac{\eta\bn^2}{2})^{\tau+1}\tn{\rbb_{0}}^2,
\]
which establishes the second statement of the induction \eqref{close}. What remains is obtaining the last statement of \eqref{close}. To address this, completing squares, observe that
\[
\tn{\rbb_{\tau+1}}\leq \sqrt{\tn{\rbb_{\tau}}^2-\frac{\eta}{2}\tn{\Jc(\bteta_\tau)^T\rbb_{\tau}}^2}\leq\tn{\rbb_{\tau}}-\frac{\eta}{4}\frac{\tn{\Jc(\bteta_\tau)^T\rbb_{\tau}}^2 }{\tn{\rbb_{\tau}}}.
\]
On the other hand, the distance to initial point satisfies
\[
\tn{\bteta_{\tau+1}-\bteta_0}\leq\tn{\bteta_{\tau+1}-\bteta_\tau}+\tn{\bteta_{\tau}-\bteta_0}\leq \tn{\bteta_{\tau}-\bteta_0}+\eta \tn{\Jc(\bteta_\tau)\rbb_\tau}.
\]
Combining the last two lines (by scaling the second line by $\frac{1}{4}\bn$) and using induction hypothesis \eqref{close}, we find that
\begin{align}
\frac{1}{4}\bn\twonorm{\vct{\theta}_{\tau+1}-\vct{\theta}_0}+\twonorm{\rbb_{\tau+1}}&\leq \frac{1}{4}\bn(\tn{\bteta_{\tau}-\bteta_0}+\eta \tn{\Jc(\bteta_\tau)\rbb_\tau})+\tn{\rbb_{\tau}}-\frac{\eta}{4}\frac{\tn{\Jc(\bteta_\tau)^T\rbb_{\tau}}^2 }{\tn{\rbb_{\tau}}}\\
&\leq \left[\frac{1}{4}\bn\tn{\bteta_{\tau}-\bteta_0}+\tn{\rbb_{\tau}}\right]+\frac{\eta}{4} \left[\bn\tn{\Jc(\bteta_\tau)\rbb_\tau}-\frac{\tn{\Jc(\bteta_\tau)^T\rbb_{\tau}}^2 }{\tn{\rbb_{\tau}}}\right]\\
&\leq \left[\frac{1}{4}\bn\tn{\bteta_{\tau}-\bteta_0}+\tn{\rbb_{\tau}}\right]+\frac{\eta}{4}\tn{\Jc(\bteta_\tau)\rbb_\tau} \left[\bn-\frac{\tn{\Jc(\bteta_\tau)^T\rbb_{\tau}} }{\tn{\rbb_{\tau}}}\right]\\
&\leq \frac{1}{4}\bn\tn{\bteta_{\tau}-\bteta_0}+\tn{\rbb_{\tau}}\\
&\leq \tn{\rbb_0}\leq \tn{\rb_0}.
%\label{secondd line}
\end{align}
This establishes the final line of the induction and concludes the proof of the upper bound on $\tn{\bteta_\tau-\bteta_0}$. To proceed, we shall bound the infinity norm of the residual. Using $\Pi_{\Scn}(\eb)=\Pi_{\Scn}(\rb_0)=\bar{\eb}_\tau$, note that
\begin{align}
\tin{f(\bteta_{\tau})-\y-\eb}&=\tin{\rb_{\tau}-\eb}\\
&\leq \tin{\rbb_{\tau}}+\tin{\eb-\bar{\eb}_{\tau}}\\
&= \tin{\rbb_{\tau}}+\tin{\eb-\Pi_{\Scn}(\eb)}\\
&= \tin{\rbb_{\tau}}+\tin{\Pi_{\Scp}(\eb)}.
\end{align}
What remains is controlling $\tin{\rbb_{\tau}}$. For this term, we shall use the naive upper bound $\tn{\rbb_{\tau}}$. Using the rate of convergence of the algorithm \eqref{close}, we have that
\[
\tn{\rbb_{\tau}}\leq (1-\frac{\eta\bn^2}{4})^\tau \tn{\rb_0}.
\]
We wish the right hand side to be at most $\nu>0$ where $\nu \geq\tin{\Pi_{\Scp}(\eb)}$. This implies that we need%$\frac{\gamma\sqrt{s}}{n}\tn{\eb}$
\begin{align}
(1-\frac{\eta\bn^2}{4})^\tau \tn{\rb_0}\leq \nu&\iff \tau \log(1-\frac{\eta\bn^2}{4})\leq \log (\frac{\nu}{\tn{\rb_0}})\\
&\iff\tau\log(\frac{1}{1-\frac{\eta\bn^2}{4}})\geq  \log (\frac{\tn{\rb_0}}{\nu})
\end{align}
To conclude, note that since $\frac{\eta\bn^2}{4}\leq 1/8$ (as $\eta\leq 1/2\bp^2$), we have
\[
\log(\frac{1}{1-\frac{\eta\bn^2}{4}})\geq  \log(1+\frac{\eta\bn^2}{4})\geq \frac{\eta\bn^2}{5}.
\]
Consequently, if $\tau\geq \frac{5}{\eta\bn^2}\log (\frac{\tn{\rb_0}}{\nu})$, we find that $\tin{\rbb_{\tau}}\leq \tn{\rbb_{\tau}}\leq \nu$, which guarantees
\[
\tin{\rb_\tau-\eb}\leq 2\nu.%\leq 2\frac{\gamma\sqrt{s}}{n}\tn{\eb},
\]
which is the advertised result. If $\eb$ is $s$ sparse and $\Scp$ is diffused, applying Lemma \ref{diff scp} we have
\[
\tin{\Pi_{\Scp}(\eb)}\leq \frac{\gamma\sqrt{s}}{n}\tn{\eb}.
\]
%First, we verify that $\eb$
\end{proof}
\subsubsection{Proof of Generic Lower Bound -- Theorem \ref{lem how far?}}
\begin{proof} Suppose $\bteta\in\Dc$ satisfies $\y=f(\bteta)$. Define $\Jb_\tau=\Jc((1-\tau)\bteta+\tau\bteta_0)$ and $\Jb=\Jc(\bteta,\bteta_0)=\int_0^1 \Jb_\tau d\tau$. Since Jacobian is derivative of $f$, we have that
\[
{f(\bteta)-f(\bteta_0)}={\int_0^1 \Jb_\tau(\bteta-\bteta_0) d\tau}= { \Jb(\bteta-\bteta_0)}.
\]
Now, define the matrices $\Jb_+=\Pi_{\Scp}( \Jb)$ and $\Jb_-=\Pi_{\Scn}( \Jb)$. Using Assumption \ref{lrank2}, we bound the spectral norms via
\[
\|\Jb_+\|=\sup_{\vb\in \Scp, \tn{\vb}\leq 1}{\tn{\Jb^T\vb}}\leq \bp\quad,\quad \|\Jb_-\|=\sup_{\vb\in \Scn, \tn{\vb}\leq 1}{\tn{\Jb^T\vb}}\leq \be.
\]
To proceed, projecting the residual on $\Scp$, we find for any $\bteta$ with $f(\bteta)=\y$
\[
\Pi_{\Scp}(f(\bteta)-f(\bteta_0))=\Pi_{\Scp}(\Jb)(\bteta-\bteta_0)\implies \tn{\bteta-\bteta_0}\geq \frac{\tn{\Pi_{\Scp}(f(\bteta)-f(\bteta_0))}}{\bp}\geq \frac{E_+}{\bp}.
\]
The identical argument for $\Scn$ yields $\tn{\bteta-\bteta_0}\geq \frac{E_-}{\be}$. Together this implies 
\begin{align}
\tn{\bteta-\bteta_0}\geq \max(\frac{E_-}{\be},\frac{E_+}{\bp}).\label{fin conc}
\end{align}
If $R$ is strictly smaller than right hand side, we reach a contradiction as $\bteta\not\in\Dc$. If $\Dc=\R^p$, we still find \eqref{fin conc}.
\end{proof}

This shows that if $\be$ is small and $E_-$ is nonzero, gradient descent has to traverse a long distance to find a good model. Intuitively, if the projection over the noise space indeed contains the label noise, we actually don't want to fit that. Algorithmically, our idea fits the residual over the signal space and not worries about fitting over the noise space. Approximately speaking, this intuition corresponds to the $\ell_2$ regularized problem
\[
\min_{\bteta} \Lc(\bteta)\quad\quad \tn{\bteta-\bteta_0}\leq R.
\]
If we set $R=\frac{E_+}{\bp}$, we can hope that solution will learn only the signal and does not overfit to the noise. The next section builds on this intuition and formalizes our algorithmic guarantees.
%For instance, in \eqref{neighbor}, the clean input is $\tilde{\x}=\cb_{\ell}$
\subsection{Proofs for Neural Networks}
Throughout, $\smn{\cdot}$ denotes the smallest singular value of a given matrix. We first introduce helpful definitions that will be used in our proofs.
\begin{definition}[Support subspace] \label{supp space}Let $\{\x_i\}_{i=1}^n$ be an input dataset generated according to Definition \ref{cdata}. Also let $\{\widetilde{\x}_i\}_{i=1}^n$ be the associated cluster centers, that is, $\widetilde{\x}_i=\cb_\ell$ iff $\x_i$ is from the $\ell$th cluster. We define the support subspace $\Scp$ as a subspace of dimension $K$, dictated by the cluster membership as follows. Let $\Lambda_{\ell}\subset\{1,\dots,n\}$ be the set of coordinates $i$ such that $\tilde{\x}_i=\cb_{\ell}$. Then, $\Scp$ is characterized by
\[
\Scp=\{\vb\in\R^n\bgl \vb_{i_1}=\vb_{i_2}\quad\text{for all}\quad i_1,i_2\in\Lambda_{\ell}\quad\text{and for all}~1\leq \ell\leq K\}.
\]
\end{definition}
%\begin{lemma} [Diffusedness of the support space]\label{diff subspace} Let input dataset $(\x_i)_{i=1}^n$ be generated according to Definition \ref{cdata}. Then, the associated support space is ${\sqrt{\frac{K}{c_{low}n}}}$ diffused.
%\end{lemma}
\begin{definition}[Neural Net Jacobian] \label{nnj def}Given input samples $(\x_i)_{i=1}^n$, form the input matrix $\X=[\x_1~\dots~\x_n]^T\in\R^{n\times d}$. The Jacobian of the learning problem \eqref{q loss}, at a matrix $\W$ is denoted by $\Jc(\W,\X)\in\R^{n\times kd}$ and is given by
\[
\Jc(\W,\X)^T=(\text{diag}(\vb)\phi'(\W\X^T))*\X^T.
\]
Here $*$ denotes the Khatri-Rao product.
\end{definition}
The following theorem is borrowed from \cite{anon2019overparam} and characterizes three key properties of the neural network Jacobian. These are smoothness, spectral norm, and minimum singular value at initialization which correspond to Lemmas 6.6, 6.7, and 6.8 in that paper.
\begin{theorem}[Jacobian Properties at Cluster Center]\label{JLlem} Suppose $\X=[\x_1~\dots~\x_n]^T\in\R^{n\times d}$ be an input dataset satisfying $\la(\X)>0$. Suppose $|\phi'|,|\phi''|\leq \Gamma$. The Jacobian mapping with respect to the input-to-hidden weights obey the following properties.
\begin{itemize}
\item Smoothness is bounded by
\begin{align*}
\opnorm{\mathcal{J}(\widetilde{\mtx{W}},\X)-\mathcal{J}(\mtx{W},\X)}\le \frac{\Gamma}{\sqrt{k}}\opnorm{\X}\fronorm{\widetilde{\mtx{W}}-\mtx{W}}\quad\text{for all}\quad \widetilde{\W},\W\in\R^{k\times d}.
\end{align*}
\item Top singular value is bounded by
\begin{align*}
\opnorm{\mathcal{J}(\mtx{W},\X)}\le \Gamma\opnorm{\X}.
\end{align*}
\item Let $C>0$ be an absolute constant. As long as
\begin{align*}
k\ge \frac{C\Gamma^2{\log n\opnorm{\X}^2}}{\la(\X)}
\end{align*}
At random Gaussian initialization $\W_0\sim\Nn(0,1)^{k\times d}$, with probability at least $1-1/K^{100}$, we have
\begin{align*}
\sigma_{\min}\left(\mathcal{J}(\W_0,\X)\right)\ge  \sqrt{\la(\X)/2}.%\frac{1}{\sqrt{2}}\mu_2(\phi)\sigma_{\min}\left(\X*\X\right),
\end{align*}
\end{itemize}
\end{theorem}
In our case, the Jacobian is {\bf{not}} well-conditioned. However, it is pretty well-structured as described previously. To proceed, given a matrix $\X\in\R^{n\times d}$ and a subspace $\Sc\subset\R^n$, we define the minimum singular value of the matrix over this subspace by $\smn{\X,\Sc}$ which is defined as
\[
\smn{\X,\Sc}=\sup_{\tn{\vb}=1,\Ub\Ub^T=\Pb_{\Sc}} \tn{\vb^T \Ub^T \X}.
\]
Here, $\Pb_{\Sc}\in\R^{n\times n}$ is the projection operator to the subspace. Hence, this definition essentially projects the matrix on $\Sc$ and then takes the minimum singular value over that projected subspace.
The following theorem states the properties of the Jacobian at a clusterable dataset.
\begin{theorem}[Jacobian Properties at Clusterable Dataset]\label{JLcor} Let input samples $(\x_i)_{i=1}^n$ be generated according to $(\eps_0,\delta)$ clusterable dataset model of Definition \ref{cdata} and define $\X=[\x_1~\dots~\x_n]^T$. Let $\Scp$ be the support space and $(\tilde{\x}_i)_{i=1}^n$ be the associated clean dataset as described by Definition \ref{supp space}. Set $\tilde{\X}=[\tilde{\x}_1~\dots~\tilde{\x}_n]^T$. Assume $|\phi'|,|\phi''|\leq \Gamma$ and $\la(\Cb)>0$. The Jacobian mapping at $\tilde{\X}$ with respect to the input-to-hidden weights obey the following properties.
\begin{itemize}
\item Smoothness is bounded by
\begin{align*}
\opnorm{\mathcal{J}(\widetilde{\mtx{W}},\tilde{\X})-\mathcal{J}(\mtx{W},\tilde{\X})}\le \Gamma\sqrt{\frac{c_{up} n}{{kK}}}\opnorm{\Cb}\fronorm{\widetilde{\mtx{W}}-\mtx{W}}\quad\text{for all}\quad \widetilde{\W},\W\in\R^{k\times d}.
\end{align*}
\item Top singular value is bounded by
\begin{align*}
\opnorm{\mathcal{J}(\mtx{W},\tilde{\X})}\le \sqrt{\frac{c_{up}n}{K}}\Gamma\opnorm{\Cb}.
\end{align*}
\item As long as
\begin{align*}
k\ge \frac{C\Gamma^2{\log K\opnorm{\Cb}}^2}{\la(\Cb)}
\end{align*}
At random Gaussian initialization $\W_0\sim\Nn(0,1)^{k\times d}$, with probability at least $1-1/K^{100}$, we have
\begin{align*}
\sigma_{\min}\left(\mathcal{J}(\W_0,\tilde{\X}),\Scp\right)\ge  \sqrt{\frac{c_{low}n\la(\Cb)}{2K}}%\frac{1}{\sqrt{2}}\mu_2(\phi)\sigma_{\min}\left(\X*\X\right),
\end{align*}
\item The range space obeys $\text{range}(\mathcal{J}(\W_0,\tilde{\X}))\subset\Scp$ where $\Scp$ is given by Definition \ref{supp space}.
\end{itemize}
\end{theorem}
\begin{proof} Let $\mathcal{J}(\W,\Cb)$ be the Jacobian at the cluster center matrix. Applying Theorem \ref{JLlem}, this matrix already obeys the properties described in the conclusions of this theorem with desired probability (for the last conclusion). We prove our theorem by relating the cluster center Jacobian to the clean dataset Jacobian matrix $\mathcal{J}(\W,\tilde{\X})$.

Note that $\tilde{\X}$ is obtained by duplicating the rows of the cluster center matrix $\Cb$. This implies that $\mathcal{J}(\W,\tilde{\X})$ is obtained by duplicating the rows of the cluster center Jacobian. The critical observation is that, by construction in Definition \ref{cdata}, each row is duplicated somewhere between $c_{low}n/K$ and $c_{up}n/K$. 

To proceed, fix a vector $\vb$ and let $\tilde{\pb}=\mathcal{J}(\W,\tilde{\X})\vb\in\R^n$ and $\pb=\mathcal{J}(\W,\Cb)\vb\in\R^K$. Recall the definition of the support sets $\Lambda_{\ell}$ from Definition \ref{supp space}. We have the identity
\[
\tilde{\pb}_i=\pb_{\ell}\quad\text{for all}\quad i\in \Lambda_{\ell}.
\]
This implies $\tilde{\pb}\in \Scp$ hence $\text{range}(\mathcal{J}(\W,\tilde{\X}))\subset\Scp$. Furthermore, the entries of $\tilde{\pb}$ repeats the entries of ${\pb}$ somewhere between $c_{low}n/K$ and $c_{up}n/K$. This implies that,
\[
\sqrt{\frac{c_{up}n}{K}}\tn{\pb}\geq \tn{\tilde{\pb}}\geq \sqrt{\frac{c_{low}n}{K}}\tn{\pb},
\]
and establishes the upper and lower bounds on the singular values of $\mathcal{J}(\W,\tilde{\X})$ over $\Scp$ in terms of the singular values of $\mathcal{J}(\W,\Cb)$. Finally, the smoothness can be established similarly. Given matrices $\W,\tilde{\W}$, the rows of the difference
\[
\opnorm{\mathcal{J}(\widetilde{\mtx{W}},\tilde{\X})-\mathcal{J}(\mtx{W},\tilde{\X})}
\]
is obtained by duplicating the rows of $\opnorm{\mathcal{J}(\widetilde{\mtx{W}},\Cb)-\mathcal{J}(\mtx{W},\Cb)}$ by at most $c_{up}n/K$ times. Hence the  spectral norm is scaled by at most $\sqrt{c_{up}n/K}$.
\end{proof}
\begin{lemma}[Upper bound on initial misfit]\label{upresz} Consider a one-hidden layer neural network model of the form $\vct{x}\mapsto \vct{v}^T\phi\left(\W\x\right)$ where the activation $\phi$ has bounded derivatives obeying $|\phi(0)|,|\phi'(z)|\le \Gamma$. Suppose entries of $\vb\in\R^k$ are half $1/\sqrt{k}$ and half $-1/\sqrt{k}$ so that $\tn{\vb}=1$. Also assume we have $n$ data points $\vct{x}_1, \vct{x}_2,\ldots,\vct{x}_n\in\R^d$ with unit euclidean norm ($\twonorm{\vct{x}_i}=1$) aggregated as rows of a matrix $\X\in\R^{n\times d}$ and the corresponding labels given by $\vct{y}\in\R^n$ generated accoring to $(\rho,\eps_0=0,\delta)$ noisy dataset (Definition \ref{noisy model}). Then for $\mtx{W}_0\in\R^{k\times d}$ with i.i.d.~$\mathcal{N}(0,1)$ entries 
\begin{align*}
\twonorm{\vct{v}^T\phi\left(\W_0\X^T\right)-\y}\le\order{\Gamma\sqrt{n\log K}},%{\log K}
\end{align*}
holds with probability at least $1-K^{-100}$.
\end{lemma}
\begin{proof} This lemma is based on a fairly straightforward union bound. First, by construction $\tn{\y}\leq \sqrt{n}$. What remains is bounding $\tn{\vct{v}^T\phi\left(\W_0\X^T\right)}$. Since $\eps_0=0$ there are $K$ unique rows. We will show that each of the unique rows is bounded with probability $1-K^{-101}$ and union bounding will give the final result. Let $\w$ be a row of $\W_0$ and $\x$ be a row of $\X$. Since $\phi$ is $\Gamma$ Lipschitz and $|\phi(0)|\leq \Gamma$, each entry of $\phi\left(\X\w\right)$ is $\order{\Gamma}$-subgaussian. Hence $\vb^T\phi(\W_0\x)$ is weighted average of $k$ i.i.d.~subgaussians which are entries of $\phi(\W_0\x)$. Additionally it is zero mean since $\sum_{i=1}^n\vb_i=0$. This means $\vb^T\phi(\W_0\x)$ is also $\order{\Gamma}$ subgaussian and obeys
\[
\Pro(|\vb^T\phi(\W_0\x)|\geq c\Gamma\sqrt{\log K})\leq K^{-101},
\]
for some constant $c>0$, concluding the proof.
\end{proof}

\subsubsection{Proof of Theorem \ref{main thm robust}}
We first prove a lemma regarding the projection of label noise on the cluster induced subspace.
\begin{lemma}\label{label noise lem} Let $\{(\x_i,y_i)\}_{i=1}^n$ be an $(\rho,\eps_0=0,\delta)$ clusterable noisy dataset as described in Definition \ref{noisy model}. Let $\{\tilde{y}_i\}_{i=1}^n$ be the corresponding noiseless labels. Let $\Jc(\W,\Cb)$ be the Jacobian at the cluster center matrix which is rank $K$ and $\Scp$ be its column space. Then, the difference between noiseless and noisy labels satisfy the bound
\[
\tin{\Pi_{\Scp}(\y-\tilde{\y})}\leq 2\rho.
\]
\end{lemma}
\begin{proof} Let $\eb=\y-\tilde{\y}$. Observe that by assumption, $\ell$th cluster has at most $s_\ell=\rho n_\ell$ errors. Let $\Ic_\ell$ denote the membership associated with cluster $\ell$ i.e.~$\Ic_\ell\subset\{1,\dots,n\}$ and $i\in \Ic_\ell$ if and only if $\x_i$ belongs to $\ell$th cluster. Let $\onebb(\ell)\in\R^n$ be the indicator function of the $\ell$th class where $i$th entry is $1$ if $i\in\Ic_\ell$ and $0$ else for $1\leq i\leq n$. Then, denoting the size of the $\ell$th cluster by $n_\ell$, the projection to subspace $\Scp$ can be written as the $\Pb$ matrix where
\[
\Pb=\sum_{\ell=1}^K\frac{1}{n_\ell} \onebb(\ell)\onebb(\ell)^T.
\]
Let $\eb_\ell$ be the error pattern associated with $\ell$th cluster i.e.~$\eb_\ell$ is equal to $\eb$ over $\Ic_\ell$ and zero outside. Since cluster membership is non-overlapping, we have that
\[
\Pb\eb=\sum_{\ell=1}^K\frac{1}{n_\ell} \onebb(\ell)\onebb(\ell)^T\eb_\ell.
\]
Similarly since supports of $\onebb(\ell)$ are non-overlapping, we have that
\[
\tin{\Pb\eb}=\max_{1\leq \ell\leq K}\frac{1}{n_\ell} \onebb(\ell)\onebb(\ell)^T\eb_\ell.
\]
Now, using $\tin{\eb}\leq 2$ (max distance between two labels), observe that
\[
\tin{\onebb(\ell)\onebb(\ell)^T\eb_\ell}\leq 2\tin{\onebb(\ell)}\tone{\eb_\ell}=2\tone{\eb_\ell}.
\]
Since number of errors within cluster $\ell$ is at most $n_\ell\rho$, we find that
\[
\tin{\Pb\eb}=\sum_{\ell=1}^K\tin{\frac{1}{n_\ell} \onebb(\ell)\onebb(\ell)^T\eb_\ell}\leq \frac{\tone{\eb_\ell}}{n_\ell}\leq 2\rho.
\]
The final line yields the bound
\[
\tin{\Pc_{\Scp}(\y-\tilde{\y})}=\tin{\Pc_{\Scp}(\eb)}=\tin{\Pb\eb}\leq 2\rho.
\]
\end{proof}
With this, we are ready to state the proof of Theorem \ref{main thm robust}.\\
\begin{proof} The proof is based on the meta Theorem \ref{grad noise}, hence we need to verify its Assumptions \ref{lrank} and \ref{spert} with proper values and apply Lemma \ref{label noise lem} to get $\tin{\Pc_{\Scp}(\eb)}$. We will also make significant use of Corollary \ref{JLcor}.

%Let us first verify Assumption \ref{diff scp}. Let $\Scp$ be the support space of the dataset per Definition \ref{supp space} which is the range space of Jacobian via Corollary \ref{JLcor}. It follows from Lemma \ref{diff subspace} that $\Scp$ is ${\sqrt{\frac{K}{c_{low}n}}}$ diffused by construction. Since Jacobian range space is subset of $\Scp$ (as $\eps_0=0$), this also implies $\be$ can be chosen to be $0$ in Assumption \ref{lrank}.

Using Corollary \ref{JLcor}, Assumption \ref{spert} holds with $L=\Gamma\sqrt{\frac{c_{up} n}{{kK}}}\opnorm{\Cb}$ where $L$ is the Lipschitz constant of Jacobian spectrum. Denote $\rb_\tau=f(\W_{\tau})-\y$. Using Lemma \ref{upresz} with probability $1-K^{-100}$, we have that $\tn{\rb_0}=\tn{\y-f(\W_0)}\le \Gamma{\sqrt{c_0n\log K/128}}$ for some $c_0>0$. Corollary \ref{JLcor} guarantees a uniform bound for $\bp$, hence in Assumption \ref{lrank}, we pick
\[
\bp\leq \sqrt{\frac{c_{up}n}{K}}\Gamma\opnorm{\Cb}.
\]
We shall also pick the minimum singular value over $\Scp$ to be
\[
\bn=\frac{\alpha_0}{2}\quad\text{where}\quad \alpha_0=\sqrt{\frac{c_{low}n\la(\Cb)}{2K}},
\]
We wish to verify Assumption \ref{lrank} over the radius of
\[
R=\frac{4\tn{f(\W_0)-\y}}{\bn}\leq\frac{\Gamma{\sqrt{c_0n\log K/8}}}{\bn}=\Gamma\sqrt{\frac{{{c_0n\log K/2}}}{{\frac{c_{low}n\la(\Cb)}{2K}}}}=\Gamma\sqrt{\frac{c_0K\log K}{c_{low}\la(\Cb)}},%=\frac{4c_0}{c_1\sqrt{\la(\Cb)}}.
\]
%Since input dataset is noiseless (i.e.~$\eps_0=0$), we already have $\eps=0$ in Assumption \ref{lrank}. 
neighborhood of $\W_0$. What remains is ensuring that Jacobian over $\Scp$ is lower bounded by $\bn$. Our choice of $k$ guarantees that at the initialization, with probability $1-K^{-100}$, we have
\[
\smn{\Jc(\W_0,\X),\Scp}\geq \bn_0.
\]
Suppose $LR\leq \bn=\bn_0/2$. Using triangle inequality on Jacobian spectrum, for any $\W\in\Dc$, using $\tf{\W-\W_0}\leq R$, we would have
\[
\smn{\Jc(\W,\X),\Scp}\geq \smn{\Jc(\W_0,\X),\Scp}-LR\geq \bn_0-\bn=\bn.
\]
Now, observe that
\begin{align}
LR=\Gamma\sqrt{\frac{c_{up} n}{{kK}}}\opnorm{\Cb} \Gamma\sqrt{\frac{c_0K\log(K)}{c_{low}\la(\Cb)}}=\Gamma^2\|\Cb\|\sqrt{\frac{c_{up}c_0n\log K}{c_{low}k\la(\Cb)}}\leq \frac{\bn_0}{2}=\sqrt{\frac{c_{low}n\la(\Cb)}{8K}},
%\\
%\frac{4c_0}{c_1\sqrt{\la(\Cb)}}\frac{\Gamma\sqrt{n}}{\sqrt{k}}\opnorm{\Cb}\leq \alpha=c_1\sqrt{n\la(\Cb)}.
\end{align}
as $k$ satisfies
\[
k\geq \order{\Gamma^4\|\Cb\|^2\frac{c_{up}K\log(K)}{c_{low}^2\la(\Cb)^2}}\ge  \order{\frac{\Gamma^4{K\log (K)\opnorm{\Cb}^2}}{\la(\Cb)^2}}.
\]
Finally, since $LR=4L\tn{\rb_0}/\alpha\leq \bn$, the learning rate is
\[
\eta\leq \frac{1}{2\bp^2}\min(1,\frac{\bn\beta}{L\twonorm{\vct{r}_0}})=\frac{1}{2\bp^2}=\frac{K}{2c_{up}n\Gamma^2\opnorm{\Cb}^2}.
\]
%Denote the minimum singular of a matrix $\M$ over a left singular subspace $S$ via $\smn{\X,S}$. Applying Corollary \ref{JLcor}, with desired probability 
%\[
%\smn{\Jc(\W_0,\X),\Scp}\geq 2c_1\sqrt{n\la(\Cb)}.
%\]
Overall, the assumptions of Theorem \ref{grad noise} holds with stated $\bn,\bp,L$ with probability $1-2K^{-100}$ (union bounding initial residual and minimum singular value events). This implies for all $\tau>0$ the distance of current iterate to initial obeys
\[
\tf{\W_\tau-\W_0}\leq R.
\]
 The final step is the properties of the label corruption. Using Lemma \ref{label noise lem}, we find that
\[
\tin{\Pi_{\Scp}(\tilde{\y}-\y)}\leq 2\rho.
\]
%We are given that there are $s$ noisy labels and each label noise moves the label by $\delta$. Hence, the label noise obeys $\tn{\y-\tilde{\y}}\leq  2\sqrt{s}$. 
Substituting the values corresponding to $\bn,\bp,L$ yields that, for all gradient iterations with 
\[% \Gamma{\sqrt{c_0n\log K/32}}
\frac{5}{\eta\bn^2}\log (\frac{\tn{\rb_0}}{2\rho})\leq\frac{5}{\eta\bn^2}\log (\frac{ \Gamma{\sqrt{c_0n\log K/32}}}{2\rho})=\order{{\frac{K}{\eta n\la(\Cb)}}\log (\frac{\Gamma\sqrt{n\log K}}{\rho}})\leq \tau,
\]
denoting the clean labels by $\tilde{\y}$  and applying Theorem \ref{grad noise}, we have that, the infinity norm of the residual obeys (using $\tin{\Pi_{\Scp}(\eb)}\leq 2\rho$)
\[
\tin{f(\W)-\tilde{\y}}\leq4\rho.
\]
%\begin{itemize}
%\item the distance of current iterate to initial obeys
%\[
%\tf{\W_\tau-\W_0}\leq R.
%\]
%\item and % $\Scp$ is $\order{\sqrt{K/n}}$ diffused)
This implies that if $\rho\leq \delta/8$, the network will miss the correct label by at most $\delta/2$, hence all labels (including noisy ones) will be correctly classified.
%\end{itemize}
\end{proof}

\subsubsection{Proof of Theorem \ref{double pert}}
Consider
\begin{align*}
f(\mtx{W},\vct{x})=\vct{v}^T\phi\left(\mtx{W}\vct{x}\right)
\end{align*}
and note that
\begin{align*}
\nabla_{\vct{x}} f(\mtx{W},\vct{x})=\mtx{W}^T\text{diag}\left(\phi'\left(\mtx{W}\vct{x}\right)\right)\vct{v}%=\sum_{\ell=1}^k\vct{v}_\ell \phi'\left(\langle\vct{w}_\ell,\vct{x}\rangle\right)\vct{w}_\ell^T
\end{align*}
Thus
\begin{align*}
\frac{\partial}{\partial \vct{x}} f(\mtx{W},\vct{x})\vct{u}=&\vct{v}^T\text{diag}\left(\phi'\left(\mtx{W}\vct{x}\right)\right)\mtx{W}\vct{u}\\
=&\sum_{\ell=1}^k \vct{v}_\ell\phi'\left(\langle\vct{w}_\ell,\vct{x}\rangle\right)\vct{w}_\ell^T\vct{u}
\end{align*}
Thus
\begin{align*}
\nabla_{\vct{w}_\ell}\left(\frac{\partial}{\partial \vct{x}} f(\mtx{W},\vct{x})\vct{u}\right)=\vct{v}_\ell\left(\phi''(\vct{w}_\ell^T\vct{x})(\vct{w}_\ell^T\vct{u})\vct{x}+\phi'(\vct{w}_\ell^T\vct{x})\vct{u}\right)
\end{align*}
Thus, denoting vectorization of a matrix by $\text{vect}(\cdot)$
\begin{align*}
\text{vect}(\mtx{U})^T\left(\frac{\partial}{\partial \text{vect}(\mtx{W})} \frac{\partial}{\partial \vct{x}} f(\mtx{W},\vct{x})\right)\vct{u}=&\sum_{\ell=1}^k\vct{v}_\ell\left(\phi''(\vct{w}_\ell^T\vct{x})(\vct{w}_\ell^T\vct{u})(\vct{u}_\ell^T\vct{x})+\phi'(\vct{w}_\ell^T\vct{x})(\vct{u}_\ell^T\vct{u})\right)\\
=&\vct{u}^T\mtx{W}^T\text{diag}\left(\vct{v}\right)\text{diag}\left(\phi''(\mtx{W}\vct{x})\right)\mtx{U}\vct{x}+\vct{v}^T\text{diag}\left(\phi'\left(\mtx{W}\vct{x}\right)\right)\mtx{U}\vct{u}
\end{align*}
Thus by the general mean value theorem there exists a point $(\widetilde{\W},\widetilde{\vct{x}})$ in the square $(\mtx{W}_0,\vct{x}_1), (\mtx{W}_0,\vct{x}_2),(\mtx{W},\vct{x}_1)$ and $(\mtx{W},\vct{x}_2)$ such that
\begin{align*}
&\left(f(\mtx{W},\vct{x}_2)-f(\mtx{W}_0,\vct{x}_2)\right)-\left(f(\mtx{W},\vct{x}_1)-f(\mtx{W}_0,\vct{x}_1)\right)\\
&\quad\quad=(\vct{x}_2-\vct{x}_1)^T\widetilde{\mtx{W}}^T\text{diag}\left(\vct{v}\right)\text{diag}\left(\phi''(\widetilde{\mtx{W}}\widetilde{\vct{x}})\right)(\mtx{W}-\mtx{W}_0)\widetilde{\vct{x}}+\vct{v}^T\text{diag}\left(\phi'\left(\widetilde{\mtx{W}}\widetilde{\vct{x}}\right)\right)(\mtx{W}-\mtx{W}_0)(\vct{x}_2-\vct{x}_1)
\end{align*}
Using the above  we have that
\begin{align}
\label{meanvalineq}
\Big|\left(f(\mtx{W},\vct{x}_2)-f(\mtx{W}_0,\vct{x}_2)\right)&-\left(f(\mtx{W},\vct{x}_1)-f(\mtx{W}_0,\vct{x}_1)\right)\Big|\nonumber\\
\overset{(a)}{\le}&\abs{(\vct{x}_2-\vct{x}_1)^T\widetilde{\mtx{W}}^T\text{diag}\left(\vct{v}\right)\text{diag}\left(\phi''(\widetilde{\mtx{W}}\widetilde{\vct{x}})\right)(\mtx{W}-\mtx{W}_0)\widetilde{\vct{x}}}\nonumber\\
&+\abs{\vct{v}^T\text{diag}\left(\phi'\left(\widetilde{\mtx{W}}\widetilde{\vct{x}}\right)\right)(\mtx{W}-\mtx{W}_0)(\vct{x}_2-\vct{x}_1)}\nonumber\\
\overset{(b)}{\le}&\left(\infnorm{\vct{v}}\twonorm{\widetilde{\vct{x}}}\opnorm{\widetilde{\W}}+\twonorm{\vct{v}}\right)\Gamma\twonorm{\vct{x}_2-\vct{x}_1}\opnorm{\mtx{W}-\mtx{W}_0}\nonumber\\
\overset{(c)}{\le} &\left(\frac{1}{\sqrt{k}}\twonorm{\widetilde{\vct{x}}}\opnorm{\widetilde{\W}}+1\right)\Gamma\twonorm{\vct{x}_2-\vct{x}_1}\opnorm{\mtx{W}-\mtx{W}_0}\nonumber\\
\overset{(d)}{\le}&\left(\frac{1}{\sqrt{k}}\opnorm{\widetilde{\W}}+1\right)\Gamma\twonorm{\vct{x}_2-\vct{x}_1}\opnorm{\mtx{W}-\mtx{W}_0}\nonumber\\
\overset{(e)}{\le} &\left(\frac{1}{\sqrt{k}}\opnorm{\mtx{W}_0}+\frac{1}{\sqrt{k}}\opnorm{\widetilde{\W}-\mtx{W}_0}+1\right)\Gamma\twonorm{\vct{x}_2-\vct{x}_1}\opnorm{\mtx{W}-\mtx{W}_0}\nonumber\\
\overset{(f)}{\le} &\left(\frac{1}{\sqrt{k}}\opnorm{\mtx{W}_0}+\frac{1}{\sqrt{k}}\fronorm{\widetilde{\W}-\mtx{W}_0}+1\right)\Gamma\twonorm{\vct{x}_2-\vct{x}_1}\opnorm{\mtx{W}-\mtx{W}_0}\nonumber\\
\overset{(g)}{\le} &\left(\frac{1}{\sqrt{k}}\fronorm{\widetilde{\W}-\mtx{W}_0}+3+2\sqrt{\frac{d}{k}}\right)\Gamma\twonorm{\vct{x}_2-\vct{x}_1}\opnorm{\mtx{W}-\mtx{W}_0}\nonumber\\
\overset{(h)}{\le} &C\Gamma\twonorm{\vct{x}_2-\vct{x}_1}\opnorm{\mtx{W}-\mtx{W}_0}
\end{align}
Here, (a) follows from the triangle inequality, (b) from simple algebraic manipulations along with the fact that $\abs{\phi'(z)}\le \Gamma$ and $\abs{\phi''(z)}\le \Gamma$, (c) from the fact that $\vct{v}_\ell=\pm \frac{1}{\sqrt{k}}$, (d) from $\twonorm{\vct{x}_2}=\twonorm{\vct{x}_1}=1$ which implies $\twonorm{\widetilde{\vct{x}}}\le 1$, (e) from triangular inequality, (f) from the fact that Frobenius norm dominates the spectral norm, (g) from the fact that with probability at least $1-2e^{-(d+k)}$, $\opnorm{\mtx{W}_0}\le 2(\sqrt{k}+\sqrt{d})$, and (h) from the fact that $\opnorm{\widetilde{\W}-\W_0}\le \fronorm{\W-\W_0}\le \widetilde{c}\sqrt{k}$ and $k\ge cd$.

Next we note that for a Gaussian random vector $\vct{g}\sim\mathcal{N}(\vct{0},\mtx{I}_d)$ we have
\begin{align}
\label{subgauss}
\|\phi(\vct{g}^T\vct{x}_2)-\phi(\vct{g}^T\vct{x}_1)\|_{\psi_2}=&\|\phi(\vct{g}^T\vct{x}_2)-\phi(\vct{g}^T\vct{x}_1)\|_{\psi_2}\nonumber\\
=&\|\phi'\left(t\vct{g}^T\vct{x}_2+(1-t)\vct{g}^T\vct{x}_1\right)\vct{g}^T(\vct{x}_2-\vct{x}_1)\|_{\psi_2}\nonumber\\
\le&\Gamma\|\vct{g}^T(\vct{x}_2-\vct{x}_1)\|_{\psi_2}\nonumber\\
\le& c\Gamma\twonorm{\vct{x}_2-\vct{x}_1}.
\end{align}
Also note that
\begin{align*}
f(\mtx{W}_0,\vct{x}_2)-f(\mtx{W}_0,\vct{x}_1)=&\vct{v}^T\left(\phi\left(\mtx{W}_0\vct{x}_2\right)-\phi\left(\mtx{W}_0\vct{x}_1\right)\right)\\
\sim&\sum_{\ell=1}^k \vct{v}_\ell \left(\phi(\vct{g}_\ell^T\vct{x}_2)-\phi(\vct{g}_\ell^T\vct{x}_1)\right)
\end{align*}
where $\vct{g}_1,\vct{g}_2,\ldots,\vct{g}_k$ are i.i.d.~vectors with $\Nn(0,\Iden_d)$ distribution. Also for $\vct{v}$ obeying $\vct{1}^T\vct{v}=0$ this random variable has mean zero. Hence, using the fact that weighted sum of subGaussian random variables are subgaussian combined with \eqref{subgauss} we conclude that $f(\mtx{W}_0,\vct{x}_2)-f(\mtx{W}_0,\vct{x}_1)$ is also subGaussian obeying $\|f(\mtx{W}_0,\vct{x}_2)-f(\mtx{W}_0,\vct{x}_1)\|_{\psi_2}\le c\Gamma\twonorm{\vct{v}}\twonorm{\vct{x}_2-\vct{x}_1} $. Thus
\begin{align}
\label{w0bnd}
\abs{f(\mtx{W}_0,\vct{x}_2)-f(\mtx{W}_0,\vct{x}_1)}\le c t\Gamma\twonorm{\vct{v}}\twonorm{\vct{x}_2-\vct{x}_1}=c t\Gamma\twonorm{\vct{x}_2-\vct{x}_1},
\end{align}
with probability at least $1-e^{-\frac{t^2}{2}}$.

Now combining \eqref{meanvalineq} and \eqref{w0bnd} we have
\begin{align*}
\delta\le&\abs{y_2-y_2}\\
=&\abs{f(\mtx{W},\vct{x}_1)-f(\mtx{W},\vct{x}_2)}\\
=&\abs{\vct{v}^T\left(\phi(\mtx{W}\vct{x}_2)-\phi(\mtx{W}\vct{x}_1)\right)}\\
\le&\abs{\left(f(\mtx{W},\vct{x}_2)-f(\mtx{W}_0,\vct{x}_2)\right)-\left(f(\mtx{W},\vct{x}_1)-f(\mtx{W}_0,\vct{x}_1)\right)}+\abs{\vct{v}^T\left(\phi(\mtx{W}_0\vct{x}_2)-\phi(\mtx{W}_0\vct{x}_1)\right)}\\
\le&C\Gamma\twonorm{\vct{x}_2-\vct{x}_1}\opnorm{\mtx{W}-\mtx{W}_0}+c t\Gamma\twonorm{\vct{x}_2-\vct{x}_1}\\
\le&C\Gamma\eps_0\left(\opnorm{\mtx{W}-\W_0}+\frac{1}{1000}t\right)
\end{align*}
Thus
\begin{align*}
\opnorm{\W-\W_0}\ge \frac{\delta}{C\Gamma\eps_0}-\frac{t}{1000},
\end{align*}
with high probability.
\subsection{Perturbation analysis for perfectly clustered data (Proof of Theorem \ref{main thm robust2})}\label{sec perturb me}
Denote average neural net Jacobian at data $\X$ via
\[
\Jc(\W_1,\W_2,\X)=\int_{0}^1\Jc(\alpha\W_1+(1-\alpha)\W_2,\X)d\alpha.
\]
\begin{lemma} [Perturbed Jacobian Distance] \label{pert dist lem}Let $\X=[\x_1~\dots~\x_n]^T$ be the input matrix obtained from Definition \ref{cdata}. Let $\tilde{\X}$ be the noiseless inputs where $\tilde{\x}_i$ is the cluster center corresponding to $\x_i$. Given weight matrices $\W_1,\W_2,\tilde{\W}_1,\tilde{\W}_2$, we have that
\[
\|\Jc(\W_1,\W_2,\X)-\Jc(\tilde{\W}_1,\tilde{\W}_2,\tilde{\X})\|\leq \Gamma\sqrt{n}(\frac{\tf{\tilde{\W}_1-\W_1}+\tf{\tilde{\W}_2-\W_2}}{2\sqrt{k}}+\eps_0).
\]
\end{lemma}
\begin{proof} Given $\W,\tilde{\W}$, we write
\[
\|\Jc(\W,\X)-\Jc(\tilde{\W},\tilde{\X})\|\leq \|\Jc(\W,\X)-\Jc(\tilde{\W},{\X})\|+\|\Jc(\tilde{\W},\X)-\Jc(\tilde{\W},\tilde{\X})\|.
\]
We first bound
\begin{align}
\|\Jc(\W,\X)-\Jc(\tilde{\W},{\X})\|&=\|\text{diag}(\vb)\phi'(\W\X^T)*\X^T-\text{diag}(\vb)\phi'(\tilde{\W}\X^T)*\X^T\|\\
&=\frac{1}{\sqrt{k}}\|(\phi'(\W\X^T)-\phi'(\tilde{\W}\X^T))*\X^T\|
\end{align}
To proceed, we use the results on the spectrum of Hadamard product of matrices due to Schur \cite{Schur1911}. Given $\A\in\R^{k\times d},\B\in \R^{n\times d}$ matrices where $\B$ has unit length rows, we have
\[
\|\A*\B\|=\sqrt{\|(\A*\B)^T(\A*\B)\|}=\sqrt{\|(\A^T\A)\odot(\B^T\B)\|}\leq \sqrt{\|\A^T\A\|}=\|\A\|.
\]
Substituting $\A=\phi'(\W\X^T)-\phi'(\tilde{\W}\X^T)$ and $\B=\X^T$, we find
\[
\|(\phi'(\W\X^T)-\phi'(\tilde{\W}\X^T))*\X^T\|\leq \|\phi'(\W\X^T)-\phi'(\tilde{\W}\X^T)\|\leq \Gamma \tf{(\tilde{\W}-\W)\X^T}\leq \Gamma\sqrt{n}\tf{\tilde{\W}-\W}.
\]
Secondly,
\[
\|\Jc(\tilde{\W},\X)-\Jc(\tilde{\W},\tilde{\X})\|=\frac{1}{\sqrt{k}}\|\phi'(\tilde{\W}\X^T)*(\X-\tilde{\X})\|
\]
where reusing Schur's result and boundedness of $|\phi'|\leq \Gamma$
\[
\|\phi'(\tilde{\W}\X^T)*(\X-\tilde{\X})\|\leq \Gamma\sqrt{k}\|\X-\tilde{\X}\|\leq \Gamma \sqrt{kn}\eps_0. 
\]
Combining both estimates yields
\[
\|\Jc(\W,\X)-\Jc(\tilde{\W},\tilde{\X})\|\leq \Gamma\sqrt{n}(\frac{\tf{\tilde{\W}-\W}}{\sqrt{k}}+\eps_0).
\]
To get the result on $\|\Jc(\W_1,\W_2,\X)-\Jc(\tilde{\W}_1,\tilde{\W}_2,\tilde{\X})\|$, we integrate
\begin{align}
\|\Jc(\W_1,\W_2,\X)-\Jc(\tilde{\W}_1,\tilde{\W}_2,\tilde{\X})\|&\leq \int_0^1\Gamma\sqrt{n}(\frac{\tf{\alpha(\tilde{\W}_1-\W_1)+(1-\alpha)(\tilde{\W}_1-\W_1)}}{\sqrt{k}}+\eps_0)d\alpha\\
&\leq \Gamma\sqrt{n}(\frac{\tf{\tilde{\W}_1-\W_1}+\tf{\tilde{\W}_2-\W_2}}{2\sqrt{k}}+\eps_0).
\end{align}
\end{proof}
\begin{theorem} [Robustness of gradient path to perturbation] \label{robust path}Generate samples $(\x_i,y_i)_{i=1}^n$ according to $(\rho,\eps_0,\delta)$ noisy dataset model and form the concatenated input/labels $\X\in\R^{d\times n},\y\in\R^n$. Let $\tilde{\X}$ be the clean input sample matrix obtained by mapping $\x_i$ to its associated cluster center. Set learning rate $\eta\leq \frac{K}{2c_{up}n\Gamma^2\opnorm{\Cb}^2}$ and maximum iterations $\tau_0$ satisfying
\[
\eta\tau_0=C_1\frac{K}{ n\la(\Cb)}\log(\frac{\Gamma\sqrt{n\log K}}{\rho}).
\]
where $C_1\geq 1$ is a constant of our choice. Suppose input noise level $\eps_0$ and number of hidden nodes obey
\[
\eps_0\leq \order{\frac{\la(\Cb)}{\Gamma^2K\log(\frac{\Gamma\sqrt{n\log K}}{\rho})}}\quad\text{and}\quad k\geq \order{\Gamma^{10}{\frac{K^2\|\Cb\|^4}{\la(\Cb)^4}}\log(\frac{\Gamma\sqrt{n\log K}}{\rho})^6}.
\]
Set $\W_0\distas\Nn(0,1)$. Starting from $\W_0=\tilde{\W}_0$ consider the gradient descent iterations over the losses
\begin{align}
&\W_{\tau+1}=\W_{\tau}-\eta\nabla \Lc(\W_{\tau})\quad\text{where}\quad \Lc(\W)=\frac{1}{2}\sum_{i=1}^n (y_i-f(\W,\tilde{\x}_i))^2\\
&\tilde{\W}_{\tau+1}=\tilde{\W}_{\tau}-\nabla\tilde{ \Lc}(\tilde{\W}_{\tau})\quad\text{where}\quad \tilde{\Lc}(\tilde{\W})=\frac{1}{2}\sum_{i=1}^n (y_i-f(\tilde{\W},\tilde{\x}_i))^2
\end{align}
Then, for all gradient descent iterations satisfying $\tau\leq \tau_0$, we have that
\[
\tn{f(\W_{\tau},\X)-f(\tilde{\W}_{\tau},\tilde{\X})}\leq c_0\tau\eta\eps_0 \Gamma^3n^{3/2}\sqrt{\log K},%\leq \order{\frac{\eps_0\Gamma^3K\sqrt{n}}{\la(\Cb)}\log(\frac{\Gamma\sqrt{n\log K}}{\rho})^2}.
\]
and
\[
\tf{\W_\tau-\tilde{\W}_\tau} \leq \order{\tau\eta\eps_0 \frac{\Gamma^4Kn}{\la(\Cb)}\log(\frac{\Gamma\sqrt{n\log K}}{\rho})^2}.
\]
\end{theorem}
\begin{proof} Since $\tilde{\W}_\tau$ are the noiseless iterations, with probability $1-2K^{-100}$, the statements of Theorem \ref{main thm robust} hold on $\tilde{\W}_\tau$. 
%we know that\[
%\tf{\tilde{\W}_{\tau}-\W_0}\leq C\Gamma{\sqrt{\frac{K\log K}{\la(\Cb)}}}.
%\]
To proceed with proof, we first introduce short hand notations. We use
\begin{align}
&\rb_i=f(\W_i,\X)-\y,~\tilde{\rb}_i=f(\tilde{\W}_i,\tilde{\X}_i)-\y\\
&\Jc_i=\Jc(\W_i,\X),~\Jc_{i+1,i}=\Jc(\W_{i+1},\W_i,\X),~\tilde{\Jc}_i=\Jc(\tilde{\W}_i,\tilde{\X}),~\tilde{\Jc}_{i+1,i}=\Jc(\tilde{\W}_{i+1},\tilde{\W}_i,\tilde{\X})\\
&d_i=\tf{\W_i-\tilde{\W}_i},~p_i=\tf{\rb_i-\tilde{\rb}_i},~\beta=\Gamma\|\Cb\|\sqrt{c_{up}n/K},~L=\Gamma\|\Cb\|\sqrt{c_{up}n/Kk}.
\end{align}
Here $\beta$ is the upper bound on the Jacobian spectrum and $L$ is the spectral norm Lipschitz constant as in Theorem \ref{JLcor}. Applying Lemma \ref{pert dist lem}, note that
\begin{align}
&\|\Jc(\W_\tau,\X)-\Jc(\tilde{\W}_\tau,\tilde{\X})\|\leq L{\tf{\tilde{\W}-\W}}+\Gamma\sqrt{n}\eps_0\leq Ld_\tau+\Gamma\sqrt{n}\eps_0\\
&\|\Jc(\W_{\tau+1},\W_\tau,\X)-\Jc(\tilde{\W}_{\tau+1},\tilde{\W}_\tau,\tilde{\X})\|\leq L(d_\tau+d_{\tau+1})/2+\Gamma\sqrt{n}\eps_0.
\end{align}
%Let $d_i=\tf{\W_i-\tilde{\W}_i}$ and $p_i=\tn{\rb_i-\tilde{\rb}_i}$. Let $\Jc_i,\tilde{\Jc}_i$ denote the Jacobian's at $\W_i,\tilde{\W}_i$ respectively. 
%We keep track of the parameter and residual differences as follows by decomposing the Jacobian
%\begin{align}
%\|\Jc(\W_\tau,\X)-\tilde{\Jc}(\tilde{\W}_\tau,\tilde{\X})\|&\leq \|\Jc(\W_\tau,\X)-\tilde{\Jc}(\tilde{\W}_\tau,{\X})\|+\|\Jc(\tilde{\W}_\tau,\X)-\tilde{\Jc}(\tilde{\W}_\tau,\tilde{\X})\|\\
%&\leq L\tf{\W_\tau-\tilde{\W}_\tau}+\Gamma\sqrt{n}\eps_0.\\
%&\leq Ld_\tau+\Theta\eps_0
%\end{align}
%Similarly, below, we show
%\begin{align}
%\|\Jc(\W_{\tau+1},\W_\tau,\X)-\tilde{\Jc}(\tilde{\W}_{\tau+1},\tilde{\W}_\tau,\tilde{\X})\|&= \frac{L}{2}(\tf{\W_\tau-\tilde{\W}_\tau}+\tf{\W_{\tau+1}-\tilde{\W}_{\tau+1}})+\Theta\eps_0.\\
%&\leq \frac{L(d_\tau+d_{\tau+1})}{2}+\Theta\eps_0
%\end{align}
Following this and using that noiseless residual is non-increasing and satisfies $\tn{\tilde{\rb}_\tau}\leq \tn{\tilde{\rb}_0}$, note that parameter satisfies
\begin{align}
&\W_{i+1}=\W_i-\eta \Jc_i\rb_i\quad,\quad \tilde{\W}_{i+1}=\tilde{\W}_i-\eta \tilde{\Jc}_i^T\tilde{\rb}_i\\
&\tf{\W_{i+1}-\tilde{\W}_{i+1}}\leq \tf{\W_i-\tilde{\W}_i}+\eta \|\Jc_i-\tilde{\Jc}_i\|\tn{\tilde{\rb}_i}+\eta \|{\Jc}_i\|\tn{\rb_i-\tilde{\rb}_i}\\
&d_{i+1}\leq d_i+\eta ((Ld_i+\Gamma\sqrt{n}\eps_0)\tn{\tilde{\rb}_0}+\beta p_i),\label{dd rec}
%\quad,\quad \tilde{\W}_{i+1}=\tilde{\W}_i-\eta \tilde{\Jc}_i\tilde{\rb}_i\\
\end{align}
and residual satisfies (using $\Iden\succeq\tilde{\Jc}_{i+1,i}\tilde{\Jc}_i^T/\beta^2\succeq0$)
\begin{align}
\rb_{i+1}&=\rb_i-\eta \Jc_{i+1,i}\Jc_i^T\rb_i\implies\\%\eta \Jc_{i+1,i}\grad{\W_i}=
\rb_{i+1}-\rt_{i+1}&=(\rb_i-\rt_i)-\eta( \Jc_{i+1,i}-\tilde{\Jc}_{i+1,i})\Jc_i^T\rb_i-\eta\tilde{\Jc}_{i+1,i}(\Jc_i^T-\tilde{\Jc}_i^T)\rb_i-\eta\tilde{\Jc}_{i+1,i}\tilde{\Jc}_i^T(\rb_i-\tilde{\rb}_i).\\
\rb_{i+1}-\rt_{i+1}&=(\Iden-\eta\tilde{\Jc}_{i+1,i}\tilde{\Jc}_i^T)(\rb_i-\rt_i)-\eta( \Jc_{i+1,i}-\tilde{\Jc}_{i+1,i})\Jc_i^T\rb_i-\eta\tilde{\Jc}_{i+1,i}(\Jc_i^T-\tilde{\Jc}_i^T)\rb_i.\\
\tn{\rb_{i+1}-\rt_{i+1}}&\leq \tn{\rb_i-\rt_i}+\eta\beta\tn{\rb_i}( L(3d_\tau+d_{\tau+1})/2+2\Gamma\sqrt{n}\eps_0).\\
\tn{\rb_{i+1}-\rt_{i+1}}&\leq \tn{\rb_i-\rt_i}+\eta\beta(\tn{\tilde{\rb}_0}+p_i)( L(3d_\tau+d_{\tau+1})/2+2\Gamma\sqrt{n}\eps_0).\label{pp rec}
\end{align}
where we used $\tn{\rb_i}\leq p_i+\tn{\tilde{\rb}_0}$ and $\tn{(\Iden-\eta\tilde{\Jc}_{i+1,i}\tilde{\Jc}_i^T)\vb}\leq \tn{\vb}$ which follows from \eqref{contractive C}. This implies
\begin{align}
p_{i+1}\leq p_i+\eta\beta(\tn{\tilde{\rb}_0}+p_i)( L(3d_\tau+d_{\tau+1})/2+2\Gamma\sqrt{n}\eps_0).\label{perturbme}
\end{align}
\noindent {\bf{Finalizing proof:}} Next, using Lemma \ref{upresz}, we have $\tn{\tilde{\rb}_0}\leq \Theta:=C_0\Gamma \sqrt{n\log K}$. We claim that if 
\begin{align}
\boxed{\eps_0\leq  \order{\frac{1}{\tau_0\eta\Gamma^2n}}\leq\frac{1}{8\tau_0\eta\beta\Gamma\sqrt{n}}\quad\text{and}\quad L\leq\frac{2}{5\tau_0\eta\Theta(1+8\eta\tau_0\beta^2)}\leq\frac{1}{30(\tau_0\eta\beta)^2\Theta},}% \min(\frac{1}{6\tau_0\eta(\Theta+4\eta\tau_0\beta)},\frac{1}{2\eta\tau_0\Theta})=
\end{align}
%and setting $\Phi=\Gamma\sqrt{n}\eps_0\Theta$, 
(where we used $\eta\tau_0\beta^2\geq 1$), for all $t\leq \tau_0$, we have that
\begin{align}
p_t\leq 8t\eta\Gamma\sqrt{n}\eps_0\Theta\beta \leq \Theta\quad,\quad d_t\leq 2t\eta \Gamma\sqrt{n}\eps_0\Theta(1+8\eta\tau_0\beta^2).\label{induct induct}
\end{align}
The proof is by induction. Suppose it holds until $t\leq \tau_0-1$. At $t+1$, via \eqref{dd rec} we have that
\[
\frac{d_{t+1}-d_t}{\eta}\leq Ld_t\Theta+\Gamma\sqrt{n}\eps_0\Theta+ 8\tau_0\eta\beta^2\Gamma\sqrt{n}\eps_0\Theta\overset{?}{\leq}2 \Gamma\sqrt{n}\eps_0\Theta(1+8\eta\tau_0\beta^2).
\]
Right hand side holds since $L\leq \frac{1}{2\eta\tau_0\Theta}$. This establishes the induction for $d_{t+1}$.
%\begin{align}
%L\leq \frac{1}{2\eta\tau_0\Theta} \implies Ld_t\tn{\tilde{\rb}_0}\leq Ld_t\Theta\leq \Gamma\sqrt{n}\eps_0\Theta(1+4\eta\tau_0\beta^2)
%\end{align}
%which establishes
%\begin{align}
%d_{t+1}&\leq2(t+1)\eta \Gamma\sqrt{n}\eps_0(\Theta+4\eta\tau_0\beta) 
%\end{align}

Next, we show the induction on $p_t$. Observe that $3d_t+d_{t+1}\leq 10\tau_0\eta \Gamma\sqrt{n}\eps_0\Theta(1+8\eta\tau_0\beta^2) $. Following \eqref{perturbme} and using $p_t\leq \Theta$, we need
\begin{align}
\frac{p_{t+1}-p_t}{\eta}\leq \beta\Theta( L(3d_\tau+d_{\tau+1})+4\Gamma\sqrt{n}\eps_0)& \overset{?}{\leq}8\Gamma\sqrt{n}\eps_0\Theta\beta\iff\\
L(3d_\tau+d_{\tau+1})+4\Gamma\sqrt{n}\eps_0&\overset{?}{\leq}8\Gamma\sqrt{n}\eps_0\iff\\
L(3d_\tau+d_{\tau+1})&\overset{?}{\leq}4\Gamma\sqrt{n}\eps_0\iff\\
10L\tau_0\eta(1+8\eta\tau_0\beta^2)\Theta&\overset{?}{\leq}4\iff\\
L&\overset{?}{\leq}\frac{2}{5\tau_0\eta(1+8\eta\tau_0\beta^2)\Theta}.
\end{align}
Concluding the induction since $L$ satisfies the final line. Consequently, for all $0\leq t\leq \tau_0$, we have that
\[
p_t\leq 8t\eta\Gamma\sqrt{n}\eps_0\Theta\beta=c_0t\eta\eps_0 \Gamma^3n^{3/2}\sqrt{\log K}.
\]
Next, note that, condition on $L$ is implied by
\begin{align}
k&\geq 1000\Gamma^2n(\tau_0\eta\beta)^4\Theta^2\\
&=\order{\Gamma^4n{\frac{K^4}{ n^4\la(\Cb)^4}}\log(\frac{\Gamma\sqrt{n\log K}}{\rho})^4(\|\Cb\|\Gamma\sqrt{n/K})^4(\Gamma \sqrt{n\log K})^2}\\
&=\order{\Gamma^{10}{\frac{K^2\|\Cb\|^4}{\la(\Cb)^4}}\log(\frac{\Gamma\sqrt{n\log K}}{\rho})^4\log^2(K)}
\end{align}
which is implied by $k\geq \order{\Gamma^{10}{\frac{K^2\|\Cb\|^4}{\la(\Cb)^4}}\log(\frac{\Gamma\sqrt{n\log K}}{\rho})^6}$.

Finally, following \eqref{induct induct}, distance satisfies
\[
d_t\leq 20t\eta^2\tau_0 \Gamma\sqrt{n}\eps_0\Theta\beta^2\leq \order{t\eta\eps_0 \frac{\Gamma^4Kn}{\la(\Cb)}\log(\frac{\Gamma\sqrt{n\log K}}{\rho})^2}.
\]
\end{proof}
\subsubsection{Completing the Proof of Theorem \ref{main thm robust2}}
Theorem \ref{main thm robust2} is obtained by the theorem below when we ignore the log terms, and treating $\Gamma$, $\la(\Cb)$ as constants. We also plug in $\eta=\frac{K}{2c_{up}n\Gamma^2\opnorm{\Cb}^2}$.
\begin{theorem} [Training neural nets with corrupted labels] \label{main thm robust22}Let $\{(\x_i,y_i)\}_{i=1}^n$ be an $(s,\eps_0,\delta)$ clusterable noisy dataset as described in Definition \ref{noisy model}. Let $\{\tilde{y}_i\}_{i=1}^n$ be the corresponding noiseless labels. Suppose $|\phi(0)|,|\phi'|,|\phi''|\leq \Gamma$ for some $\Gamma\geq 1$, input noise and the number of hidden nodes satisfy 
\[
\eps_0\leq \order{\frac{\la(\Cb)}{\Gamma^2K\log(\frac{\Gamma\sqrt{n\log K}}{\rho})}}\quad\text{and}\quad k\geq \order{\Gamma^{10}{\frac{K^2\|\Cb\|^4}{\la(\Cb)^4}}\log(\frac{\Gamma\sqrt{n\log K}}{\rho})^6}.
\]
where $\Cb\in\R^{K\times d}$ is the matrix of cluster centers. Set learning rate $\eta\leq \frac{K}{2c_{up}n\Gamma^2\opnorm{\Cb}^2}$ and randomly initialize $\W_0\distas\Nn(0,1)$. With probability $1-3/K^{100}-K\exp(-100d)$, after $\tau= \order{\frac{K}{\eta n\la(\Cb)}}\log(\frac{\Gamma\sqrt{n\log K}}{\rho})$ iterations, for all $1\leq i\leq n$, we have that
\begin{itemize}
%\item the entrywise prediction to {\em{true labels}} $\{\tilde{y}_i\}_{i=1}^n$ satisfy 
%\[
%|f(\W_{\tau},\x_i)-\tilde{y}_i|\leq 4\rho+c\frac{\eps_0\Gamma^3K\sqrt{n\log K}}{\la(\Cb)}\log(\frac{\Gamma\sqrt{n\log K}}{\rho}).%\order{\frac{Ks}{n}},
%\]

\item The per sample normalized $\ell_2$ norm bound satisfies
\[
\frac{\tn{f(\W_{\tau},\X)-\tilde{\y}}}{\sqrt{n}}\leq 4\rho+c\frac{\eps_0\Gamma^3K\sqrt{\log K}}{\la(\Cb)}\log(\frac{\Gamma\sqrt{n\log K}}{\rho}).
\]
\item Suppose $\rho\leq\delta/8$. Denote the total number of prediction errors with respect to true labels (i.e.~not satisfying \eqref{pls satisfy this eq}) by $\text{err}(\W)$. With same probability, $\text{err}(\W_\tau)$ obeys
\[
\frac{\text{err}(\W_\tau)}{n}\leq c\frac{\eps_0K}{\delta}\frac{\Gamma^3\sqrt{\log K}}{\la(\Cb)}\log(\frac{\Gamma\sqrt{n\log K}}{\rho}).
\]
\item Suppose $\rho\leq\delta/8$ and $\eps_0\leq c'\delta\min(\frac{ \la(\Cb)^2}{{\Gamma^5K^2}\log(\frac{\Gamma\sqrt{n\log K}}{\rho})^3},\frac{1}{\Gamma\sqrt{d}})$, then, $\W_\tau$ assigns all inputs in the $\eps_0$ neighborhood of cluster centers to the correct labels i.e.~for any cluster center $\cb_\ell$ and $\x$ obeying $\tn{\x-\cb_\ell}\leq \eps_0$, $\x$ receives the ground truth label of $\cb_\ell$.
\item Finally, for any iteration count $0\leq t\leq \tau$ the total distance to initialization is bounded as
\begin{align}
\tf{\W_t-{\W}_0} \leq  \order{\Gamma\sqrt{\frac{K\log K}{\la(\Cb)}}+t\eta\eps_0 \frac{\Gamma^4Kn}{\la(\Cb)}\log(\frac{\Gamma\sqrt{n\log K}}{\rho})^2}.\label{dst W bound}
\end{align}
\end{itemize}
\end{theorem}
\begin{proof} Note that proposed number of iterations $\tau$ is set so that it is large enough for Theorem \ref{main thm robust} to achieve small error in the clean input model ($\eps_0=0$) and it is small enough so that Theorem \ref{robust path} is applicable. In light of Theorems \ref{robust path} and \ref{main thm robust} consider two gradient descent iterations starting from $\W_0$ where one uses clean dataset (as if input vectors are perfectly cluster centers) $\tilde{\X}$ and other uses the original dataset $\X$. Denote the prediction residual vectors of the noiseless and original problems at time $\tau$ with respect true ground truth labels $\tilde{\y}$ by $\tilde{\rb}_\tau=f(\tilde{\W}_{\tau},\tilde{\X})-\tilde{\y}$ and $\rb_\tau=f(\W_{\tau},\X)-\tilde{\y}$ respectively. Applying Theorems \ref{robust path} and \ref{main thm robust}, under the stated conditions, we have that
\begin{align}
\tin{\tilde{\rb}_\tau}&\leq 4\rho\quad\text{and}\quad\\
\tn{\rb_\tau-\tilde{\rb}_\tau}&\leq c\eps_0\frac{K}{ n\la(\Cb)}\log(\frac{\Gamma\sqrt{n\log K}}{\rho})\Gamma^3n^{3/2}\sqrt{\log K}\\
&=c\frac{\eps_0\Gamma^3K\sqrt{n\log K}}{\la(\Cb)}\log(\frac{\Gamma\sqrt{n\log K}}{\rho})
\end{align}
\noindent {\bf{First statement:}} The latter two results imply the $\ell_2$ error bounds on $\rb_\tau=f(\W_{\tau},\X)-\tilde{\y}$. 

\noindent {\bf{Second statement:}} To assess the classification rate we count the number of entries of $\rb_\tau=f(\W_{\tau},\X)-\tilde{\y}$ that is larger than the class margin $\delta/2$ in absolute value. Suppose $\rho\leq \delta/8$. Let $\Ic$ be the set of entries obeying this. For $i\in\Ic$ using $\tin{\tilde{\rb}_\tau}\leq 4\rho\leq \delta/4$, we have
\[
|r_{\tau,i}|\geq \delta/2\implies |r_{\tau,i}|+|r_{\tau,i}-\bar{r}_{\tau,i}|\geq \delta/2\implies |r_{\tau,i}-\bar{r}_{\tau,i}|\geq \delta/4.
\]
Consequently, we find that%Denoting this number by $\|\rb_\tau\|_{0,\delta}$, we have that
\[
\tone{\rb_\tau-\rbb_\tau}\geq |\Ic|\delta/4.
\]
Converting $\ell_2$ upper bound on the left hand side to $\ell_1$, we obtain
\[
c\sqrt{n}\frac{\eps_0\Gamma^3K\sqrt{n\log K}}{\la(\Cb)}\log(\frac{\Gamma\sqrt{n\log K}}{\rho})\geq |\Ic|\delta/4.
\]
Hence, the total number of errors is at most
\[
|\Ic|\leq c'\frac{\eps_0nK}{\delta}\frac{\Gamma^3\sqrt{\log K}}{\la(\Cb)}\log(\frac{\Gamma\sqrt{n\log K}}{\rho})
\]
\noindent{\bf{Third statement -- Showing zero error:}} Pick an input $\x$ within $\eps_0$ neighborhood of one of the cluster centers $\cb\in (\cb_\ell)_{\ell=1}^K$. We will argue that $f(\W_\tau,\x)-f(\tilde{\W}_\tau,\cb)$ is smaller than $\delta/4$ when $\eps_0$ is small enough. We again write
\[
|f(\W_\tau,\x)-f(\tilde{\W}_\tau,\cb)|\leq |f(\W_\tau,\x)-f(\tilde{\W}_\tau,{\x})|+|f(\tilde{\W}_\tau,{\x})-f(\tilde{\W}_\tau,\cb)|
\]
The first term can be bounded via
\begin{align}
|f(\W_\tau,\x)-f(\tilde{\W}_\tau,{\x})|&=|\vb^T\phi(\W_\tau\x)-\vb^T\phi(\tilde{\W}_\tau\x)|\leq \tn{\vb}\tn{\phi({\W}_\tau\x)-\phi(\tilde{\W}_\tau\x)}\\
&\leq \Gamma \tf{\W_{\tau}-\tilde{\W}_{\tau}}\\
&\leq \order{\eps_0 \frac{\Gamma^5K^2}{\la(\Cb)^2}\log(\frac{\Gamma\sqrt{n\log K}}{\rho})^3}
\end{align}
Next, we need to bound
\begin{align}
|f(\tilde{\W}_\tau,{\x})-f(\tilde{\W}_\tau,\cb)|&\leq |\vb^T\phi(\tilde{\W}_\tau\x)-\vb^T\phi(\tilde{\W}_\tau\cb)|
\end{align}
where $\tf{\tilde{\W}_\tau-\W_0}\leq \order{\Gamma\sqrt{\frac{K\log K}{\la(\Cb)}}}$, $\tn{\x-\cb}\leq \eps_0$ and $\W_0\distas\Nn(0,\Iden)$. Consequently, using by assumption we have
\[
k\geq \order{\tf{\tilde{\W}-\W_0}^2}=\order{\Gamma^2\frac{K\log K}{\la(\Cb)}},
\]
and applying Theorem \ref{double pert2} (which is a variation of Theorem \ref{double pert}), with probability at $1-K\exp(-100d)$, for all inputs $\x$ lying $\eps_0$ neighborhood of cluster centers, we find that
\begin{align}
|f(\tilde{\W}_\tau,{\x})-f(\tilde{\W}_\tau,\cb)|&\leq C'\Gamma \eps_0(\tf{\tilde{\W}_\tau-\W_0}+\sqrt{d})\\
&C\Gamma \eps_0(\Gamma\sqrt{\frac{K \log K}{\la(\Cb)}}+\sqrt{d}).
%&\leq \Gamma \tn{\W_0(\x-\cb)}+\Gamma\eps_0\|\tilde{\W}_\tau-\W_0\|\\
%&\leq \Gamma \tn{\W_0(\x-\cb)}+\Gamma\eps_0\order{\sqrt{\frac{K\log K}{\la(\Cb)}}}\\
%&\leq \Gamma \eps_0 ???
\end{align}
Combining the two bounds above we get
\begin{align}
|f(\W_\tau,\x)-f(\tilde{\W}_\tau,\cb)|&\leq \eps_0\order{ \frac{\Gamma^5K^2}{\la(\Cb)^2}\log(\frac{\Gamma\sqrt{n\log K}}{\rho})^3+\Gamma (\Gamma\sqrt{\frac{K \log K}{\la(\Cb)}}+\sqrt{d})}\\
&\leq \eps_0\order{ \frac{\Gamma^5K^2}{\la(\Cb)^2}\log(\frac{\Gamma\sqrt{n\log K}}{\rho})^3}.
\end{align}
Hence, if $\eps_0\leq c'\delta\min(\frac{ \la(\Cb)^2}{{\Gamma^5K^2}\log(\frac{\Gamma\sqrt{n\log K}}{\rho})^3},\frac{1}{\Gamma\sqrt{d}})$, we obtain that, for all $\x$, the associated cluster $\cb$ and true label assigned to cluster $\tilde{y}=\tilde{y}(\cb)$, we have that
\[
|f(\W_\tau,\x)-\tilde{y}|< |f(\tilde{\W}_\tau,\cb)-f(\W_\tau,\x)|+|f(\tilde{\W}_\tau,\cb)-\tilde{y}|\leq 4\rho + \frac{\delta}{4}.
\]
If $\rho\leq \delta/8$, we obtain
\[
|f(\W_\tau,\x)-\tilde{y}|< \delta/2
\]
hence, $\W_\tau$ outputs the correct decision for all samples.

\noindent{\bf{Fourth statement -- Distance:}} This follows from the triangle inequality
\[
\tf{\W_\tau-\W_0}\leq \tf{\W_\tau-\tilde{\W}_\tau}+\tf{\tilde{\W}_\tau-\W_0}
\]
We have that right hand side terms are at most $ \order{\Gamma\sqrt{\frac{K\log K}{\la(\Cb)}}}$ and $\order{t\eta\eps_0 \frac{\Gamma^4Kn}{\la(\Cb)}\log(\frac{\Gamma\sqrt{n\log K}}{\rho})^2}$ from Theorems \ref{robust path} and \ref{main thm robust} respectively. This implies \eqref{dst W bound}.
\end{proof}

\section{Proof of Lemma \ref{simple pert}}
Create two matrices $\X\in\R^{s\times d}$ and $\tilde{\X}\in\R^{s\times d}$ by concatenating the input samples. Note that the matrix $\mtx{X}-\tilde{\mtx{X}}$ has i.i.d.~$\Nn(0,2\eps_0^2/d)$ entries. Thus, using standard results regarding the concentration of the spectral norm with probability at least $1-e^{-d/2}$, we have
\[
\|\X-\tilde{\X}\|\leq \sqrt{2}\left(\sqrt{\frac{s}{d}}+2\right)\eps_0\leq 5\eps_0.
\]
Define the vectors $\y,\tilde{\y}\in\R^s$ with entries given by $y_i$ and $\tilde{y}_i$, respectively. Suppose $\W$ fits these labels perfectly. Using the fact that $\tn{\vb}=1$, we can conclude that
\begin{align*}
\sqrt{s}\delta&\leq \tn{\y-\tilde{\y}}=\tn{f(\W,\X)-f(\W,\tilde{\X})},\\
&=\tn{\vb^T(\phi(\W\X)-\phi(\W\tilde{\X}))},\\
&\leq\Gamma\tn{\vb}\tf{\W(\X-\tilde{\X})},\\
&\leq \Gamma\|\X-\tilde{\X}\|\tf{\W}\leq 5\Gamma\eps_0 \tf{\W}.
\end{align*}
This implies the desired lower bound on $\tf{\W}$.
\section{Uniform guarantee for minimum distance}\label{singlabpert}
%Let $f(\mtx{W},\vct{x})=\vct{v}^T\phi\left(\mtx{W}\vct{x}\right)$ where $\vb$ is fixed, $\mtx{W}\in\R^{k\times d}$, and $k\ge c d$ with $c>0$ a fixed constant. 
\begin{theorem}\label{double pert2}
Assume $\abs{\phi'},\abs{\phi''}\le \Gamma$ and $k\gtrsim d$. Suppose $\W_0\distas\Nn(0,1)$. Let $\cb_1,\dots,\cb_K$ be cluster centers. Then, with probability at least $1-2e^{-(k+d)}-Ke^{-100d}$ over $\W_0$, any matrix $\W$ satisfying $\fronorm{\mtx{W}-\mtx{W}_0}\lesssim\sqrt{k}$ satisfies the following. For all $1\leq i\leq K$,
\begin{align*}
\sup_{\tn{\x-\cb_i},\tn{\tilde{\x}-\cb_i}\leq \eps_0}|f(\mtx{W},\vct{x})- f(\mtx{W},\tilde{\vct{x}})|\leq C\Gamma\eps_0(\opnorm{\mtx{W}-\mtx{W}_0}+\sqrt{d}).
\end{align*}
\end{theorem}

\begin{proof}
Note that
\begin{align*}
\abs{f(\W,\vct{x})-f(\W,\widetilde{\vct{x}})}=&\abs{\vct{v}^T\left(\phi\left(\W\vct{x}\right)-\phi\left(\W\widetilde{\vct{x}}\right)\right)}\\
\le&\abs{\vct{v}^T\left(\phi\left(\W\vct{x}\right)-\phi\left(\W\widetilde{\vct{x}}\right)\right)-\vct{v}^T\left(\phi\left(\W_0\vct{x}\right)-\phi\left(\W_0\widetilde{\vct{x}}\right)\right)}+\abs{\vct{v}^T\left(\phi\left(\W_0\vct{x}\right)-\phi\left(\W_0\widetilde{\vct{x}}\right)\right)}
\end{align*}
To continue note that by the general mean value theorem there exists a point $(\overline{\W},\overline{\vct{x}})$ in the square $(\mtx{W}_0,\vct{x}), (\mtx{W}_0,\widetilde{\vct{x}}),(\mtx{W},\vct{x}),$ and $(\mtx{W}, \widetilde{\vct{x}})$ such that
\begin{align*}
&\left(f(\mtx{W},\vct{x})-f(\mtx{W}_0,\vct{x})\right)-\left(f(\mtx{W},\widetilde{\vct{x}})-f(\mtx{W}_0,\widetilde{\vct{x}})\right)\\
&\quad\quad=(\vct{x}-\widetilde{\vct{x}})^T\overline{\mtx{W}}^T\text{diag}\left(\vct{v}\right)\text{diag}\left(\phi''(\overline{\mtx{W}}\overline{\vct{x}})\right)(\mtx{W}-\mtx{W}_0)\overline{\vct{x}}+\vct{v}^T\text{diag}\left(\phi'\left(\overline{\mtx{W}}\overline{\vct{x}}\right)\right)(\mtx{W}-\mtx{W}_0)(\vct{x}-\widetilde{\vct{x}})
\end{align*}
Using the above  we have that
\begin{align}
\label{meanvalineq}
\Big|\left(f(\mtx{W},\vct{x})-f(\mtx{W}_0,\vct{x})\right)-&\left(f(\mtx{W},\widetilde{\vct{x}})-f(\mtx{W}_0,\widetilde{\vct{x}})\right)\Big|\\
\overset{(a)}{\le}&\abs{(\vct{x}-\widetilde{\vct{x}})^T\overline{\mtx{W}}^T\text{diag}\left(\vct{v}\right)\text{diag}\left(\phi''(\overline{\mtx{W}}\overline{\vct{x}})\right)(\mtx{W}-\mtx{W}_0)\overline{\vct{x}}}\nonumber\\
&+\abs{\vct{v}^T\text{diag}\left(\phi'\left(\overline{\mtx{W}}\overline{\vct{x}}\right)\right)(\mtx{W}-\mtx{W}_0)(\vct{x}-\widetilde{\vct{x}})}\nonumber\\
\overset{(b)}{\le}&\left(\infnorm{\vct{v}}\twonorm{\overline{\vct{x}}}\opnorm{\overline{\W}}+\twonorm{\vct{v}}\right)\Gamma\twonorm{\vct{x}-\widetilde{\vct{x}}}\opnorm{\mtx{W}-\mtx{W}_0}\nonumber\\
\overset{(c)}{\le} &\left(\frac{1}{\sqrt{k}}\twonorm{\overline{\vct{x}}}\opnorm{\overline{\W}}+1\right)\Gamma\twonorm{\vct{x}-\widetilde{\vct{x}}}\opnorm{\mtx{W}-\mtx{W}_0}\nonumber\\
\overset{(d)}{\le}&\left(\frac{1}{\sqrt{k}}\opnorm{\overline{\W}}+1\right)\Gamma\twonorm{\vct{x}-\widetilde{\vct{x}}}\opnorm{\mtx{W}-\mtx{W}_0}\nonumber\\
\overset{(e)}{\le} &\left(\frac{1}{\sqrt{k}}\opnorm{\mtx{W}_0}+\frac{1}{\sqrt{k}}\opnorm{\overline{\W}-\mtx{W}_0}+1\right)\Gamma\twonorm{\vct{x}-\widetilde{\vct{x}}}\opnorm{\mtx{W}-\mtx{W}_0}\nonumber\\
\overset{(f)}{\le} &\left(\frac{1}{\sqrt{k}}\opnorm{\mtx{W}_0}+\frac{1}{\sqrt{k}}\fronorm{\overline{\W}-\mtx{W}_0}+1\right)\Gamma\twonorm{\vct{x}-\widetilde{\vct{x}}}\opnorm{\mtx{W}-\mtx{W}_0}\nonumber\\
\overset{(g)}{\le} &\left(\frac{1}{\sqrt{k}}\fronorm{\overline{\W}-\mtx{W}_0}+3+2\sqrt{\frac{d}{k}}\right)\Gamma\twonorm{\vct{x}-\widetilde{\vct{x}}}\opnorm{\mtx{W}-\mtx{W}_0}\nonumber\\
\overset{(h)}{\le} &C\Gamma\twonorm{\vct{x}-\widetilde{\vct{x}}}\opnorm{\mtx{W}-\mtx{W}_0}\label{final dist bound}
\end{align}
Here, (a) follows from the triangle inequality, (b) from simple algebraic manipulations along with the fact that $\abs{\phi'(z)}\le \Gamma$ and $\abs{\phi''(z)}\le \Gamma$, (c) from the fact that $\vct{v}_\ell=\pm \frac{1}{\sqrt{k}}$, (d) from $\twonorm{\vct{x}}=\twonorm{\widetilde{\vct{x}}}=1$ which implies $\twonorm{\overline{\vct{x}}}\le 1$, (e) from triangular inequality, (f) from the fact that Frobenius norm dominates the spectral norm, (g) from the fact that with probability at least $1-2e^{-(d+k)}$, $\opnorm{\mtx{W}_0}\le 2(\sqrt{k}+\sqrt{d})$, and (h) from the fact that $\opnorm{\overline{\W}-\W_0}\le \fronorm{\W-\W_0}\le \widetilde{c}\sqrt{k}$ and $k\ge cd$.

Next we note that for a Gaussian random vector $\vct{g}\sim\mathcal{N}(\vct{0},\mtx{I}_d)$ we have
\begin{align}
\label{subgauss}
\|\phi(\vct{g}^T\vct{x})-\phi(\vct{g}^T\widetilde{\vct{x}})\|_{\psi_2}=&\|\phi(\vct{g}^T\vct{x})-\phi(\vct{g}^T\widetilde{\vct{x}})\|_{\psi_2}\nonumber\\
=&\|\phi'\left(t\vct{g}^T\vct{x}+(1-t)\vct{g}^T\widetilde{\vct{x}}\right)\vct{g}^T(\vct{x}-\widetilde{\vct{x}})\|_{\psi_2}\nonumber\\
\le&\Gamma\|\vct{g}^T(\vct{x}-\widetilde{\vct{x}})\|_{\psi_2}\nonumber\\
\le& c\Gamma\twonorm{\vct{x}-\widetilde{\vct{x}}}.
\end{align}
Also note that
\begin{align*}
f(\mtx{W}_0,\vct{x})-f(\mtx{W}_0,\widetilde{\vct{x}})=&\vct{v}^T\left(\phi\left(\mtx{W}_0\vct{x}\right)-\phi\left(\mtx{W}_0\widetilde{\vct{x}}\right)\right)\\
\sim&\sum_{\ell=1}^k \vct{v}_\ell \left(\phi(\vct{g}_\ell^T\vct{x})-\phi(\vct{g}_\ell^T\widetilde{\vct{x}})\right)
\end{align*}
where $\vct{g}_1,\vct{g}_2,\ldots,\vct{g}_k$ are i.i.d.~vectors with $\Nn(0,\Iden_d)$ distribution. Also for $\vct{v}$ obeying $\vct{1}^T\vct{v}=0$ this random variable has mean zero. Hence, using the fact that weighted sum of subGaussian random variables are subgaussian combined with \eqref{subgauss} we conclude that $f(\mtx{W}_0,\vct{x})-f(\mtx{W}_0,\widetilde{\vct{x}})$ is also subGaussian with Orlicz norm obeying $\|f(\mtx{W}_0,\vct{x})-f(\mtx{W}_0,\widetilde{\vct{x}})\|_{\psi_2}\le c\Gamma\twonorm{\vct{v}}\twonorm{\vct{x}-\widetilde{\vct{x}}} $. Now, suppose $\x,\tilde{\x}$ be within $\eps_0$ neighborhood of a cluster center $\cb$. We write
\[
|f(\mtx{W}_0,\widetilde{\vct{x}})-f(\mtx{W}_0,{\vct{x}})|\leq |f(\mtx{W}_0,\cb)-f(\mtx{W}_0,{\vct{x}})|+|f(\mtx{W}_0,\widetilde{\vct{x}})-f(\mtx{W}_0,\cb)|
\]
To proceed, since $X_{\x}=f(\mtx{W}_0,\x)$ is a Gaussian process, applying standard chaining bounds \cite{talagrand2006generic}, we find
\begin{align}
\sup_{\tn{\x-\cb}\leq \eps_0} |f(\mtx{W}_0,\cb)-f(\mtx{W}_0,{\vct{x}})|\leq c'\Gamma \eps_0\sqrt{d}\label{clust bound}
\end{align}
with probability $1-\exp(-100d)$. Here $\eps_0\sqrt{d}$ comes from the $\gamma_2$ functional of the scaled ball around the cluster. Applying a union bound over all clusters $\cb_1$ to $\cb_K$, we find that, with $1-\exp(-d)$ probability, \eqref{clust bound} holds uniformly which implies that for all $\x,\tilde{\x}$ pairs of interest
\[
\sup_{\x,\tilde{\x}~\text{within cluster}}|f(\mtx{W}_0,\widetilde{\vct{x}})-f(\mtx{W}_0,{\vct{x}})|\leq 2c'\Gamma \eps_0\sqrt{d}.
\]
Combining this with \eqref{final dist bound}, we conclude with the advertised bound.
\end{proof}

\end{document}